\documentclass[11pt,letterpaper]{article}

\usepackage{microtype}
\usepackage{setspace}
\usepackage{times}
\usepackage{fullpage}
\usepackage[small,compact]{titlesec}
\usepackage{abstract}

\usepackage[hang,flushmargin]{footmisc}
\usepackage{amsmath}
\usepackage{amssymb}
\usepackage{amsthm}
\usepackage{mathabx}
\usepackage{bm}
\usepackage{accents}
\usepackage[mathscr]{euscript}
\DeclareMathOperator*{\argmin}{arg\,min}

\usepackage{graphicx}
\usepackage{caption}
\usepackage{subcaption}
\usepackage{booktabs}
\usepackage{multirow}
\usepackage{multicol}
\usepackage{threeparttable}
\usepackage{array}

\usepackage{ifthen}
\usepackage{url}
\usepackage{enumerate}
\usepackage[normalem]{ulem}
\usepackage[dvipsnames]{xcolor}
\usepackage{color-edits}
\usepackage{comment}

\usepackage[sectionbib,round]{natbib}
\usepackage{bibunits}
\defaultbibliographystyle{plainnat}
\defaultbibliography{mybib}
\usepackage[toc,page]{appendix}

\usepackage{hyperref}
\definecolor{MyDarkBlue}{RGB}{158,0,0}
\hypersetup{colorlinks=true, linkcolor=MyDarkBlue, citecolor=MyDarkBlue, filecolor=MyDarkBlue, urlcolor=MyDarkBlue}

\theoremstyle{plain}
\newtheorem{theorem}{Theorem}
\newtheorem{proposition}{Proposition}
\newtheorem{lemma}{Lemma}

\theoremstyle{definition}
\newtheorem{definition}{Definition}
\newtheorem{example}{Example}

\newcommand{\Halmos}{\ensuremath{\square}}

\newcommand{\mathbbm}[1]{\text{\usefont{U}{bbm}{m}{n}#1}}

\newcommand{\eps}{\varepsilon}
\newcommand{\bI}{\mathbbm{1}}
\newcommand{\cP}{\mathcal{P}}
\newcommand{\cX}{\mathcal{X}}
\newcommand{\cY}{\mathcal{Y}}

\newcommand{\Var}{\mathrm{Var}}
\newcommand{\Cov}{\mathrm{Cov}}
\newcommand{\bE}{\mathrm{E}}

\usepackage{tikz}
\usetikzlibrary{positioning, arrows.meta, calc, shapes.geometric}
\usepackage{changepage}
\usepackage{framed}

\newcommand{\squishlist}{
   \begin{list}{$\bullet$}
    { \setlength{\itemsep}{0pt} \setlength{\parsep}{1pt}
      \setlength{\topsep}{1pt} \setlength{\partopsep}{1pt}
      \setlength{\leftmargin}{1.5em} \setlength{\labelwidth}{1em}
      \setlength{\labelsep}{0.5em} } }

\newcommand{\squishlisttwo}{
   \begin{list}{$\bullet$}
    { \setlength{\itemsep}{0pt} \setlength{\parsep}{0pt}
      \setlength{\topsep}{0pt} \setlength{\partopsep}{0pt}
      \setlength{\leftmargin}{1em} \setlength{\labelwidth}{1.5em}
      \setlength{\labelsep}{0.5em} } }

\newcommand{\squishend}{
    \end{list}  }

\newcommand{\hy}[1]{\noindent{\color{blue}\{{\bf Hema:}  #1\}}}

\title{Rectification Difficulty and Optimal Sample Allocation in LLM-Augmented Surveys}

\author{Zikun Ye\thanks{We thank Olivier Toubia and the participants of the UW--UBC Marketing Conference for their helpful comments and feedback. Please address all correspondence to: zikunye@uw.edu and hemay@uw.edu.}\\ \textit{University of Washington} \and Hema Yoganarasimhan \\ \textit{University of Washington}}

\date{This version: July 9, 2026}

\begin{document}

\maketitle

\begin{abstract}
\begin{singlespace}
\noindent Large Language Models can generate synthetic survey responses at low cost, but their accuracy varies unpredictably across questions.
We study the design problem of allocating a fixed budget of human respondents across estimation tasks when cheap LLM predictions are available for every task.
Our framework combines three components.
First, building on Prediction-Powered Inference, we characterize a question-specific \emph{rectification difficulty} that governs how quickly the estimator's variance decreases with human sample size.
Second, we derive a closed-form optimal allocation rule that directs more human labels to tasks where the LLM is least reliable.
Third, since rectification difficulty depends on unobserved human responses for new surveys, we propose a meta-learning approach, trained on historical data, that predicts it for entirely new tasks without pilot data.
The framework extends to general M-estimation, covering regression coefficients and multinomial logit partworths for conjoint analysis.
We validate the framework on two datasets spanning different domains, question types, and LLMs, showing that our approach captures 61--79\% of the theoretically attainable efficiency gains, achieving 11.4\% and 10.5\% MSE reductions without requiring any pilot human data for the target survey.
\end{singlespace}
\end{abstract}

\noindent \textbf{Keywords:} Market Research, Survey Design, Large Language Models, Prediction-Powered Inference.

\thispagestyle{empty}
\newpage

\begin{bibunit}
\setcounter{page}{1}

\section{Introduction}

\subsection{The promise and challenge of LLMs as synthetic respondents}

Surveys play an important role in market research, from brand tracking and conjoint analysis to willingness-to-pay estimation. However, fielding them is expensive. This cost pressure has led both researchers and firms to explore whether Large Language Models (LLMs) can serve as low-cost synthetic respondents. 
Industry interest is growing rapidly: synthetic-research startups have attracted major funding, market-research firms are developing digital-twin panels, and many firms are investing in tools that simulate consumer behavior, forecast customer reactions, and substitute for conventional surveys \citep{indexventures2026simile,ipsos2025syntheticrespondents,ipsos2025stanfordsynthetic,ipsosdigital2026innotest,accenture2025aaru,maier2025purchase}. Academic research has developed in parallel, showing that LLMs can generate plausible survey responses at negligible marginal cost and, in some settings, replicate aggregate human patterns across marketing and social-science tasks \citep[e.g.,][]{brand2023using, argyle2023out, wang2024market}.

At the same time, LLM responses exhibit systematic biases, insufficient within-population heterogeneity, and sensitivity to topics, populations, prompts, and question formats \citep{bisbee2024synthetic, anthis2025position, brucks2025prompt}. In some settings, LLMs approximate aggregate human responses well; in others, they produce biased or near-constant answers that fail to capture respondent-level variation \citep{motoki2024more, peng2025mega}. The central empirical fact is therefore heterogeneity: LLM predictions are useful for some questions and nearly useless for others, and this variation is difficult to anticipate ex ante \citep{toubia2025database}. LLMs should therefore be viewed not as wholesale substitutes for human respondents, but as unevenly informative auxiliary signals. Thus, the relevant managerial question is how to use those signals to determine where scarce human responses are most valuable.

In a multi-question survey, this becomes a pre-fielding budget-allocation problem. A firm can cheaply generate LLM predictions for every question/item, but must divide a fixed human-response budget before observing any target human outcomes. Consider a survey asking respondents ``Would you pay \$4 more for an oat-milk latte?'' and ``Should gas stoves be banned?'' The LLM may capture respondent-level variation well for the former but produce homogeneous or misleading responses for the latter. An efficient design would therefore allocate fewer human respondents to the first question and more to the second. Yet before fielding the survey, the firm observes only the questions, response options, and LLM predictions---not the joint behavior of human and LLM responses needed to assess their relative informativeness. 



This example highlights three challenges. First, what question-level quantity should guide the allocation? Raw LLM accuracy measures how closely predictions match human responses, but does not necessarily capture how valuable an additional human response is for precise estimation. Second, even if the appropriate quantity were known, the firm faces a cold-start problem: it depends on the joint distribution of human and LLM responses, which is unobserved before the new survey is fielded. Third, once human responses have been collected, how should biased but informative LLM predictions be combined with those responses while preserving valid inference?

\subsection{An efficient design framework for LLM-augmented surveys}

We propose a three-part survey-design framework that jointly addresses the challenges outlined above: (i) we define \emph{rectification difficulty} as the question-specific residual uncertainty that remains after the LLM's signal is used optimally, and show how it governs the allocation of human respondents; (ii) we derive a closed-form optimal allocation rule that directs more human responses to questions with greater rectification difficulty; and (iii) we develop a meta-learning procedure that predicts rectification difficulty for new questions from historical survey data. 


\smallskip\noindent\textit{Rectification and rectification difficulty.}\quad
To address the estimation challenge, we build on Prediction-Powered Inference \citep[PPI;][]{angelopoulos2023prediction} and its \textsc{PPI++}\ extension \citep{angelopoulos2023ppi}. \textsc{PPI++} uses human responses to correct the LLM's average prediction error while adaptively controlling how much weight is placed on the LLM. 
When synthetic LLM responses are abundant, the asymptotic variance of the estimator for question $q$ takes the form \(A_q/n_q\), where $n_q$ is the human sample size and $A_q$, the rectification difficulty, is the residual variance that remains after the LLM prediction is used optimally.
Our contribution is to isolate \(A_q\) as a question-level design object and show that it is the correct variance primitive for allocating human respondents across questions. Neither raw LLM accuracy nor the human-response variance 
alone determines the value of an additional human response. Thus, \textsc{PPI++} addresses how to combine human and LLM responses after data collection, while rectification difficulty provides the bridge from estimation to survey design.



\smallskip\noindent\textit{Optimal allocation rule.}\quad
Next, we show that if the rectification difficulties $\{A_q\}$ were known, the resulting allocation problem has a closed-form solution. Under a weighted mean-squared-error objective, $n_q^\star \propto \sqrt{\frac{w_q A_q}{c_q}}$, 
where $w_q$ is the importance weight for question $q$
and $c_q$ is its per-response cost (Theorem~\ref{thm:optimal_allocation}). The rule allocates more human responses to questions for which the LLM leaves greater residual uncertainty, and fewer to questions for which it provides a strong respondent-level signal. 
It is a natural extension of classical Neyman square-root allocation \citep{neyman1934two}, with rectification difficulty \(A_q\), rather than the raw human-response variance governing the allocation. The square-root structure itself is classical; the key contribution is identifying \(A_q\) as the appropriate quantity to insert into that rule in an LLM-augmented survey. It also embodies the principle of managerially efficient design \citep{toubia2007research}. 

Our framework also extends to other variants of the LLM-augmented survey design problem. We show that there is an equivalent formulation in which the firm can specify a desired level of precision and use the framework to determine the minimum human-response budget needed to attain it. The same square-root rule governs how that budget should be distributed across questions; only the overall number of human responses changes. We also provide solutions for other variations, e.g., general $M$ estimands, alternative loss functions, power analysis, and block-level allocation for matrix-sampling designs.

\smallskip\noindent\textit{Meta-learning for zero-shot allocation.}\quad
The third component makes the allocation rule operational in the cold-start setting. Because $A_q$ depends on the joint distribution of human and LLM responses, it cannot be computed directly for a new question before the survey is fielded. We therefore propose using the firm's historical survey archive.
For each previously fielded question, paired human--LLM responses provide an estimate of rectification difficulty. 
A predictive model then learns how difficulty varies with observable question characteristics---represented by LLM-embeddings of the question wording and response options---and transfers this relationship to unseen questions.

Returning to the firm in our example, suppose its historical paired human--LLM data indicate low rectification difficulty for a question such as ``Would you pay \$5 more for a plant-based burger?'', but higher difficulty for ``Should gas-powered cars be banned by 2035?''. When the firm's new survey asks ``Would you pay \$4 more for an oat-milk latte?'' and ``Should gas stoves be banned?'', the meta-learner uses the relationship learned between question characteristics and rectification difficulty to predict a lower $\tilde{A}_q$ for the former and a higher $\tilde{A}_q$ for the latter. Because the square-root rule allocates human responses in proportion to $\sqrt{\tilde{A}_q}$, the firm assigns fewer respondents to the oat-milk-latte question and more to the gas-stove question. 



An important feature of our framework is that errors in the predicted difficulties affect only allocation efficiency, not inferential validity. Under standard regularity conditions, prediction errors in the meta-learner may reduce precision but do not compromise asymptotic coverage. We further show that the resulting efficiency loss is bounded and depends on relative prediction errors across questions.

Together, the three components provide a principled and simple framework for LLM-augmented survey design: \textsc{PPI++} provides a valid way to combine human and LLM responses after the survey is fielded; rectification difficulty and the square-root rule translate estimator precision into an allocation decision; and meta-learning makes that decision operational before the new survey is fielded. 


\subsection{Empirical validation}

We evaluate the framework in two complementary settings spanning different domains, question formats, respondent populations, and response-generating LLMs. The first is the Twin-2K-500 digital-twin dataset \citep{toubia2025database}, which contains paired human and GPT-4o responses from more than 2{,}000 respondents across a broad battery of behavioral-economics questions. The second is the 2024 Cooperative Election Study \citep[CCES;][]{cces2024}, a large political survey covering economic perceptions, vote choice, and foreign policy, for which we generate surrogate responses using Gemini~2.5~Flash. In both settings, we treat the full observed respondent pool as a ground-truth population for controlled Monte Carlo evaluation.

Our approach delivers economically meaningful efficiency gains across both datasets. We benchmark against a conventional human-only design that allocates a fixed human-response budget uniformly and uses the human sample mean for each question. Our method, which combines \textsc{PPI++} estimation with allocation based on meta-learned rectification difficulties, reduces MSE by 11.4\% in Twin-2K-500 and 10.5\% in the CCES. Because MSE scales approximately as the inverse of the human-response budget, these reductions are equivalent to achieving the benchmark's precision with roughly 11\% and 10\% fewer human responses, respectively. 
We also compare with an infeasible oracle that allocates using the true rectification difficulties, which are unavailable at the design stage.
Our proposed method captures 78.6\% of the oracle improvement in Twin-2K-500 and 61.3\% in the CCES. The oracle gains themselves---14.5\% and 17.1\%, respectively---show that heterogeneity in LLM reliability creates substantial scope for more efficient survey design.

The mechanism behind these gains is also informative. Most of the improvement comes from reallocating human responses rather than from changing the estimator. Under uniform allocation, \textsc{PPI++} alone reduces MSE by only 3.6\% in Twin-2K-500 and 1.5\% in the CCES; reallocating the budget using predicted rectification difficulty adds a further 7.8 and 9.0 percentage points, respectively. Thus, allocation accounts for about 68\% of the total gain in Twin-2K-500 and 86\% in the CCES. The larger allocation effect in the CCES reflects its greater heterogeneity: relatively easy factual items coexist with polarized questions on vote choice, Ukraine, and Gaza, for which the LLM often produces nearly homogeneous responses and provides little respondent-level signal.

The analysis further shows why rectification difficulty, rather than raw LLM accuracy, is the appropriate design index. In Twin-2K-500, the LLM achieves relatively high average agreement with humans, yet standard PPI increases or fails to reduce variance for most questions, and the variance-minimizing \textsc{PPI++} weight is zero for 26 of the 68 questions. In the CCES, the optimal weight is zero for 40 of the 133 questions.
At the same time, rectification difficulty is sufficiently systematic to be predicted from question characteristics: the out-of-sample Spearman correlation between predicted and realized log difficulty is 0.73 in Twin-2K-500 and 0.48 in the more heterogeneous CCES. Thus, our results confirm that the meta-learner need not predict $A_q$ perfectly to recover a useful ranking of where human responses are most valuable.

Finally, a broad set of robustness exercises supports the underlying mechanism. Rectification difficulty is highly stable across survey waves; a meta-learner trained on one dataset transfers meaningfully to a different survey generated by a different LLM; aggregation to task- or module-level allocation preserves/strengthens the gains; and the value of optimal allocation increases as the dispersion of question difficulty grows.

\subsection{Contribution}

Our work makes both methodological and managerial contributions to the market-research literature. Methodologically, we formulate the pre-fielding, cross-question sample-allocation problem for LLM-augmented surveys. We identify rectification difficulty \(A_q\)---rather than raw LLM accuracy or the human-response variance---as the variance index that governs the marginal value of an additional human response. We then derive the corresponding Neyman-type oracle allocation and develop a meta-learning procedure that uses question embeddings and historical paired human--LLM data to predict rectification difficulty for unseen questions. Together, these components turn \textsc{PPI++} from a post-collection estimation method into an operational survey-design framework that determines where human responses should be collected. Because the meta-learner affects only the allocation, prediction errors may reduce efficiency but do not compromise the asymptotic validity of the final estimates.

From a managerial perspective, our framework provides a practical answer to a common question faced by practitioners and applied marketers: how can we effectively leverage LLMs to reduce the cost and improve the efficiency of survey design? Alternately, given a fixed human-response budget and cheap LLM predictions, how can we optimally allocate human responses across questions/tasks? Two empirical applications demonstrate that our approach can deliver meaningful efficiency gains across distinct survey domains and LLMs. Furthermore, because its required inputs (a historical archive of paired human--LLM responses and a fixed budget) are common in repeated survey operations, the framework can be implemented with existing survey archives and standard fielding workflows.

\paragraph{Roadmap.}
Section~\ref{sec:related} reviews the related literature. Section~\ref{sec:setup} formalizes the problem setup. Section~\ref{sec:ppi_pp} introduces the \textsc{PPI++} estimator and defines the rectification difficulty $A_q$. Section~\ref{sec:framework} develops the survey design framework: optimal allocation, meta-learning for $A_q$, and end-to-end inference guarantees. Sections~\ref{sec:empirics} and~\ref{sec:cces} validate the framework on the Twin-2K-500 and CCES datasets. Section~\ref{sec:conclusion} concludes with a discussion of implications and future directions.

\section{Related Literature}
\label{sec:related}

Our work draws on and contributes to two broad streams of research: efficient experimental and survey design in marketing and Hybrid human–LLM designs and Prediction-Powered Inference for combining human and machine-generated data.

\paragraph{Efficient sampling and survey design.}

The idea of allocating observations to maximize inferential precision has a long foundation in statistical sampling. A classical starting point is Neyman allocation in stratified sampling \citep{neyman1934two, cochran1977sampling, sarndal1992model}. Neyman's key insight is that an efficient sample need not be a miniature version of the population: when some parts of the problem contribute more sampling uncertainty than others, precision is improved by allocating more observations to those parts rather than sampling strictly in proportion to their size. A classical Neyman rule applied across questions would allocate more respondents to questions with larger human-response standard deviations, i.e., \(n_q \propto \sqrt{\Var(Y_q)}\), where \(\Var(Y_q)\) is the variance of human responses. Our framework departs from this classical benchmark in two ways. First, although our allocation rule retains the same square-root structure, we show that the relevant variance primitive in the hybrid human--LLM setting in our setting is the residual variance of the rectified estimator, \(A_q\). Thus, the allocation directs more human responses not simply to questions with more variable human outcomes, but to questions where the LLM leaves greater residual uncertainty after \textsc{PPI++} correction. Second, \(A_q\) is unknown for a new question because it depends on paired human and LLM responses that have not yet been observed. Whereas the classical optimal-allocation rule can be implemented once its variance inputs are known or estimated, we address this cold-start problem by using a meta-learner trained on historical paired data to predict \(A_q\) from question characteristics, enabling allocation for new surveys.

Marketing research has developed a related tradition of efficient experimental and survey design, especially in conjoint analysis and preference measurement. A foundational line of work develops adaptive methods that select which product profiles or choice sets to present to each respondent, tailoring question content in real time to maximize information about consumer preferences \citep{toubia2003fast, toubia2004polyhedral}. \citet{dzyabura2011active} extend this idea using machine learning to adaptively select the most informative questions for identifying consideration heuristics. A parallel literature optimizes the static design of choice experiments, i.e., which profiles to include and how to combine them, using prior information about respondent preferences to improve efficiency \citep{huber1996importance, sandor2001designing}. \citet{toubia2007research} unify these ideas under the principle of ``managerially efficient design'': experimental resources should be directed toward improving the estimates that matter most for the decision at hand. A complementary strand improves conjoint efficiency not by adapting the questions asked, but by pooling information across related studies: \citet{mccoy2022twoforone} develop a Bayesian cross-category learning approach that shares partworth information across categories with overlapping attributes, such as jointly estimating preferences for hiking jackets and sleeping bags. This cross-study pooling principle is conceptually analogous to our meta-learning step, which pools information across historical questions to predict rectification difficulty for new questions.

Our framework shares the design philosophy of this literature: allocate scarce respondent effort where it yields the greatest improvement in statistical efficiency. The specific resource allocation we study (how many respondents to recruit per question, given a fixed total budget) is a natural complement to the question-content and module-assignment decisions above. What makes non-uniform allocation over questions both feasible and valuable in our setting is the availability of cheap LLM predictions: they provide a task-specific signal about where human data are most needed, creating the heterogeneous difficulty structure that our allocation rule exploits. Thus, the main innovation in our study, compared to the earlier literature, is that we show how a designer can incorporate the LLM predictions to make the allocation over questions/tasks more efficient. 


\paragraph{Hybrid human--LLM designs and Prediction-Powered Inference.}
Hybrid human--LLM estimation builds on the logic of classical model-assisted survey sampling \citep{isaki1982survey, deville1992calibration, sarndal1992model}, which combines a design-based human sample with model-based auxiliary predictions to improve precision while preserving valid inference. A growing set of papers in marketing and survey research studies related ways of combining human data with LLM-generated predictions. In marketing, \citet{ye2025lola} provide an early example of this broader human--LLM hybrid logic: their \textsc{LOLA} framework uses LLM predictions to initialize adaptive content experiments and then updates decisions using observed human click behavior. Whereas \textsc{LOLA} focuses on sequential reward maximization within an experiment, our objective is pre-fielding sample allocation and valid inference across multiple survey estimands. In survey and choice-model settings, \citet{yin2026synthetic} derive a closed-form heterogeneous mixing rule for how much synthetic data to incorporate in a logit choice model, while \citet{huang2025many} develop an uncertainty-quantification framework for selecting the synthetic sample size needed for reliable inference under human--LLM misalignment. These papers focus on how much synthetic data to use for a given question, model, or estimand, typically taking the human-data component as given or separately specified. By contrast, our framework asks where to allocate a fixed human-label budget across multiple questions/tasks before any target human responses have been collected.

The recent Prediction-Powered Inference literature \citep[PPI;][]{angelopoulos2023prediction} and its \textsc{PPI++}\ extension \citep{angelopoulos2023ppi} provide a principled way to combine cheap machine predictions with gold-standard human labels, using a labeled sample to correct prediction bias while leveraging a large prediction pool for variance reduction. PPI-style methods are increasingly applied to business and social-science problems: \citet{ji2025predictions} revisit the surrogate-outcome problem through the lens of modern AI predictions, \citet{vafa2025estimating} apply PPI to estimate wage disparities, and \citet{wang2025finetune} combine domain-specific fine-tuning with rectification to improve inference when human labels are limited. Several related papers extend this logic to design and allocation problems. \citet{broska2025mixed} propose a PPI-based mixed-subjects framework and use power analysis to determine the optimal mix of human subjects and LLM predictions for a single estimand. \citet{fisch2024stratified} partition the input space for a single estimand into strata and allocate human labels across them, using autorater confidence as a proxy for unknown stratum-level quantities. \citet{angelopoulos2025costoptimal} determine when to query a costly strong rater rather than a cheap weak rater within a single evaluation task, operationalizing their policy through either transfer from a related labeled dataset or a labeled burn-in sample. Our framework differs by studying a pre-fielding, cross-question allocation problem: each survey question corresponds to a distinct estimand, and the manager must decide how many human respondents to assign to each question before observing any target human outcomes. We identify the question-level \textsc{PPI++} residual variance \(A_q\) as the relevant allocation index and predict it from historical paired human--LLM data, enabling managers to design an entirely new survey without collecting pilot human data for the target questions. 
Subsequent work by \citet{ye2026allocating} builds on our rectification-difficulty framework and studies a complementary online setting,
where rectification difficulty is not predicted ex-ante but learned adaptively from target-survey responses
using a UCB allocation policy.

\section{Problem Setup and Overview of Solution Concept}
\label{sec:setup}

We study a hybrid survey design problem in which LLM predictions are cheap to generate, whereas human responses are costly. The researcher has two sources of information: (i) a historical corpus $\mathcal{H}$ of previously fielded questions for which paired human and LLM responses are available, and (ii) a new survey $\mathcal{T}$ whose questions have not yet been fielded to humans. The design problem is to decide how many human responses to collect for each question in $\mathcal{T}$, subject to a budget constraint, while using LLM predictions as an auxiliary source of information.

\subsection{LLM surrogates, historical corpus, and new survey}

When question $q$ is administered to respondent $i$, it generates an observation $(\bm{X}_{qi},Y_{qi})$, where $\bm{X}_{qi}\in\cX$ denotes the full prompt information for that respondent-question pair (e.g., question wording and options, product profile, and possibly respondent characteristics), and $Y_{qi}\in\cY$ is the human response (e.g., a rating or a choice indicator). For each fixed question $q$, we assume the respondent-level observations are i.i.d.:
\[
(\bm{X}_{qi},Y_{qi}) \stackrel{\mathrm{i.i.d.}}{\sim} \cP_q,
\qquad i=1,2,\ldots,
\]
and we write $(\bm{X}_q,Y_q)\sim \cP_q$ for a generic draw from the target population for question $q$.

We use an LLM as a surrogate response generator through a prediction function $f:\cX\to\cY$. For any prompt $\bm{X}$, the LLM produces a surrogate response $f(\bm{X})$. In particular, for a labeled respondent prompt $\bm{X}_{qi}$, we denote the corresponding paired LLM prediction by
\[
Y^{\mathrm{LLM}}_{qi}:=f(\bm{X}_{qi}).
\]
Throughout, we treat the LLM prediction rule $f(\cdot)$ as fixed; it may be deterministic or stochastic.

\paragraph{Historical corpus $\mathcal{H}$.}
For each question $q\in\mathcal{H}$, we observe previously collected human responses $\{(\bm{X}_{qi},Y_{qi})\}_{i=1}^{n_q^{\mathcal{H}}}$, and can compute LLM predictions $f(\bm{X}_{qi})$ for the same respondent prompts. Thus, for historical questions, paired human and LLM responses are available.

\paragraph{New survey $\mathcal{T}$.}
Let $\mathcal{T}$ denote the set of new survey questions. At the design stage, the researcher observes each target question $q\in\mathcal{T}$ and can generate LLM predictions from the corresponding prompts, but has not yet collected any human outcomes. The design decision is how many human responses $n_q$ to collect for each $q\in\mathcal{T}$.

\paragraph{Synthetic pool and the synthetic-data-rich (SDR) regime.}
For each target question $q\in\mathcal{T}$, at the design stage the researcher may draw a large unlabeled synthetic pool of prompts
\[
\{\widetilde{\bm{X}}_{qi}\}_{i=1}^{m_q} \stackrel{\mathrm{i.i.d.}}{\sim} \cP_q(\bm{X}),
\]
and compute corresponding LLM-only predictions
\[
\tilde{Y}^{\mathrm{LLM}}_{qi}:=f(\widetilde{\bm{X}}_{qi}).
\]
These synthetic data are available before any human responses for question $q$ have been collected.

Because LLM predictions are comparatively cheap, we focus on the \emph{synthetic-data-rich} (SDR) regime, in which $m_q$ is large; for the main-text theory, we take $m_q\to\infty$. The binding resource constraint is therefore the budget for human responses.

Note that the historical corpus $\mathcal{H}$ and the target survey $\mathcal{T}$ play different roles. The target survey is the object of inference and design: for these questions, the researcher must decide the human sample sizes $\{n_q\}_{q\in\mathcal{T}}$. The historical corpus contains previously observed paired human and LLM responses that can be used to inform that design decision.

\subsection{Target estimands and design goal}

In the main exposition, for each $q\in\mathcal{T}$, the estimand is the population mean response
\begin{equation}
\theta_q^\star := \bE_{(\bm{X}_q,Y_q)\sim \cP_q}[Y_q].
\end{equation}
This scalar setup covers many market-research quantities, such as average ratings and response shares, because each can be written as the mean of a per-respondent outcome. Section~\ref{subsec:extensions} extends the framework to more general $M$-estimation targets.

Suppose that collecting one human response to question $q$ costs $c_q>0$ (e.g., recruitment, incentives, and platform fees). We allow these costs to vary across questions. In practice, some questions take longer to answer, are more cognitively demanding, or require higher respondent compensation, which induces question-specific collection costs. Given a total human-label budget $B$, the researcher chooses sample sizes $\{n_q\}_{q\in\mathcal{T}}$ subject to $\sum_{q\in\mathcal{T}} c_q\,n_q \le B$, with pre-specified importance weights $\{w_q\}_{q\in\mathcal{T}}$ reflecting how much precision is desired for each question.

Our design goal is to choose $\{n_q\}_{q\in\mathcal{T}}$ to minimize the weighted mean-squared estimation error across the target questions:
\begin{equation}
\min_{\{n_q\}_{q\in\mathcal{T}}}
\sum_{q\in\mathcal{T}} w_q\,\bE\!\left[(\widehat{\theta}_q-\theta_q^\star)^2\right]
\quad
\text{s.t.}
\quad
\sum_{q\in\mathcal{T}} c_q\,n_q \le B.
\end{equation}
Equivalently, one may fix a target precision level and minimize the total human-data cost; we discuss this dual formulation in Section~\ref{subsec:extensions} and Web Appendix~\ref{rmk:dual_and_single_q}.

\subsection{Overview of solution concept}
\label{subsec:solution_overview}

A useful starting benchmark is the case in which LLM predictions are perfectly aligned with human responses. In that case, the researcher could simply use the LLM predictions in place of collecting new human data, and the survey-design problem would essentially disappear. In practice, however, a growing literature shows that while LLM predictions are often informative, they are not fully accurate: they can be biased, and their reliability varies substantially across questions \citep{peng2025mega}. Thus, the relevant problem is not whether to use LLMs or humans, but how to combine cheap LLM predictions with scarce and costly human responses in a way that improves efficiency while preserving valid inference.

\paragraph{Challenges.}
Note that the manager/researcher must decide how to allocate the human-label budget across the questions in the new survey $\mathcal{T}$ before any new human responses have been collected. This design decision raises three related challenges. First, what question-level quantity should guide the allocation? A natural candidate from prior work is the LLM's predictive accuracy on each question, but accuracy captures how close LLM predictions are to human responses on average, not how useful the LLM is for precise estimation once its errors are corrected with human data; it is not obvious what the right index is, nor how the allocation should depend on it. Second, even given the right index, the researcher faces a cold-start problem: the index depends on the joint distribution of human and LLM responses, which has not been observed for the new survey. Third, once human responses have been collected, how should biased but informative LLM predictions be combined with those human responses for estimation and valid inference?

\paragraph{Key idea.}
Our solution rests on three components that together address these challenges. First, we propose a question-level index, the \emph{rectification difficulty} $A_q$, which measures how hard it is to correct the LLM's errors with human data; we show that this index, rather than LLM accuracy, is exactly the quantity that governs the variance of the rectified estimator (Section~\ref{sec:ppi_pp}). Second, given the index, we derive the optimal allocation rule: human labels should be allocated in proportion to $\sqrt{A_q}$, a closed-form square-root rule that extends classical Neyman allocation to the hybrid human--LLM setting (Section~\ref{subsec:oracle_allocation}). Third, to resolve the cold-start problem, we exploit the fact that LLM reliability is systematic rather than arbitrary: patterns in historical paired human--LLM data reveal which types of questions are hard to rectify, so a meta-learning model trained on historical surveys can predict $A_q$ for new questions before any human data on those questions are collected (Section~\ref{subsec:meta_learning}). Once the survey is fielded, the \textsc{PPI++}\ estimator combines the collected human responses with LLM predictions for valid inference.


Figure~\ref{fig:conceptual_overview} organizes these components into an operational two-phase workflow. Phase A (meta-learning) trains a regression model $\widehat{\phi}$ on the historical corpus $\mathcal{H}$ and outputs the predicted difficulties $\{\tilde{A}_q\}_{q \in \mathcal{T}}$. Phase B (allocation and inference) plugs these predictions into the square-root allocation rule (Theorem~\ref{thm:optimal_allocation}) and forms the \textsc{PPI++}\ estimator from the collected data (Propositions~\ref{thm:end_to_end} and~\ref{prop:robust_allocation}).

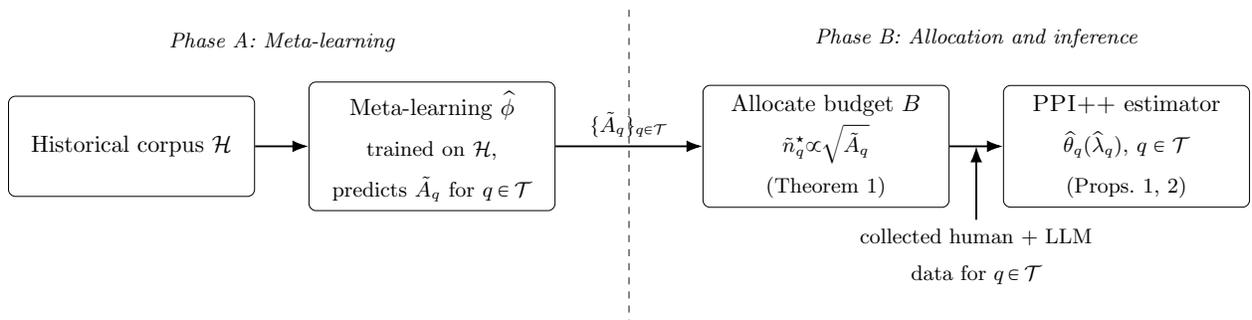
\begin{figure}[ht!]
\centering
\resizebox{\linewidth}{!}{%
\begin{tikzpicture}[
  font=\small,
  every node/.style={align=center},
  box/.style={draw, rounded corners=3pt, minimum height=15mm, text width=34mm,
              inner sep=4pt, font=\small},
  arr/.style={-{Latex[length=2.4mm]}, thick},
  phase/.style={font=\footnotesize\itshape}
]

\node[box] (H) at (0, 0) {Historical corpus $\mathcal{H}$};
\node[box, right=8mm of H] (meta) {Meta-learning $\widehat{\phi}$ \\[2pt] {\footnotesize trained on $\mathcal{H}$, \\ predicts $\tilde{A}_q$ for $q \in \mathcal{T}$}};

\node[box, right=22mm of meta] (alloc) {Allocate budget $B$ \\[2pt] {\footnotesize $\tilde{n}_q^\star \propto \sqrt{\tilde{A}_q}$ \\ (Theorem~\ref{thm:optimal_allocation})}};
\node[box, right=8mm of alloc] (estim) {\textsc{PPI++}\ estimator \\[2pt] {\footnotesize $\widehat{\theta}_q(\widehat{\lambda}_q)$, $q \in \mathcal{T}$ \\ (Props.~\ref{thm:end_to_end},~\ref{prop:robust_allocation})}};

\draw[arr] (H) -- (meta);
\draw[arr] (meta) -- node[above, font=\footnotesize] {$\{\tilde{A}_q\}_{q \in \mathcal{T}}$} (alloc);
\draw[arr] (alloc) -- (estim);

\coordinate (allocEstimMid) at ($(alloc.east)!0.5!(estim.west)$);
\node[font=\footnotesize, align=center, below=11mm of allocEstimMid, anchor=north]
  (collect) {collected human + LLM\\data for $q \in \mathcal{T}$};
\draw[arr] (collect.north) -- (allocEstimMid);

\coordinate (sepTop) at ($(meta.north east)!0.5!(alloc.north west) + (0, 1.1)$);
\coordinate (sepBot) at ($(meta.south east)!0.5!(alloc.south west) - (0, 1.7)$);
\draw[dashed, thick, black!50] (sepTop) -- (sepBot);

\node[phase] at ($(H.north)!0.5!(meta.north) + (0, 0.7)$) {Phase A: Meta-learning};
\node[phase] at ($(alloc.north)!0.5!(estim.north) + (0, 0.7)$) {Phase B: Allocation and inference};

\end{tikzpicture}}
\caption{Overview of the framework. Phase A trains a meta-learner $\widehat{\phi}$ on the historical corpus $\mathcal{H}$ and predicts the rectification difficulty $\tilde{A}_q$ of each new question; Phase B allocates the human-label budget across questions via the square-root rule and, once the data are collected, forms the \textsc{PPI++}\ estimator with valid confidence intervals.}
\label{fig:conceptual_overview}
\end{figure}

An important feature of the framework is worth emphasizing: it separates design from inference. Historical paired data are used to guide the allocation decision, while the final estimator for the target survey is constructed only after the new human responses are observed. This decoupling yields two guarantees that are worth keeping distinct. First, validity: conditional on any realized allocation, and hence on any predicted $\{\tilde{A}_q\}$, the post-collection estimator is asymptotically normal and its confidence intervals attain nominal coverage (Proposition~\ref{thm:end_to_end}). Second, efficiency: prediction errors in $\tilde{A}_q$ distort only the sample sizes $\tilde{n}_q^\star$, and the resulting efficiency loss is bounded (Proposition~\ref{prop:robust_allocation}). Errors in $\tilde{A}_q$ thus affect how precise the survey is, not whether its conclusions are valid.




\section{PPI++, Variance, and Rectification Difficulty}
\label{sec:ppi_pp}

This section introduces two ingredients the paper builds on: the \textsc{PPI++}\ estimator, the final estimator used in Phase B, and the \emph{rectification difficulty} $A_q$, a question-level measure and the prediction target of Phase A meta-learning. We first review benchmark estimators and motivate \textsc{PPI++}, then characterize its variance, and use that characterization to define the rectification difficulty. Throughout this section, we fix an arbitrary target question $q \in \mathcal{T}$ and suppress the subscript.

\subsection{Benchmark estimators}

Our estimand of interest is the true population mean
\begin{equation}
\theta^\star = \bE[Y].
\end{equation}
We consider three benchmark estimators.

\textbf{Human-only estimator.} A natural benchmark is the human-only sample mean, $\frac{1}{n}\sum_{i=1}^{n} Y_i$. This estimator uses the \textit{gold-standard} human preference data and, as such, is unbiased but can have high variance when $n$ is small.

\textbf{LLM-only estimator.} At the other extreme, one might instead use the mean of the LLM predictions from the large synthetic pool, $\frac{1}{m}\sum_{i=1}^{m} \tilde{Y}^{\mathrm{LLM}}_i$, which is cheap to compute but is generally biased because $\bE[Y^{\mathrm{LLM}}] \neq \bE[Y]$. 

\textbf{Prediction-Powered Inference estimator.} The PPI estimator \citep{angelopoulos2023prediction} combines the two data sources (human and LLM) by using the human sample to correct the average bias of the LLM while leveraging cheap LLM predictions to reduce variance:
\begin{equation}
\widehat{\theta}_{\mathrm{PPI}}
=
\frac{1}{n}\sum_{i=1}^{n}Y_i
+
\left(
\frac{1}{m}\sum_{i=1}^{m}\tilde{Y}^{\mathrm{LLM}}_i
-
\frac{1}{n}\sum_{i=1}^{n}Y_i^{\mathrm{LLM}}
\right).
\end{equation}
The first term is the human-only estimator, while the second term is an LLM-based correction term that adjusts the human estimate using the difference between the unlabeled and labeled LLM averages. Here the two LLM symbols play different roles: $Y_i^{\mathrm{LLM}} = f(\bm{X}_i)$ is the LLM prediction for labeled respondent $i$, paired with the human response $Y_i$, whereas $\tilde{Y}_i^{\mathrm{LLM}} = f(\widetilde{\bm{X}}_i)$ is the LLM response to an additional prompt drawn from the same population, for which no human response is ever collected. The PPI framework is now widely used for integrating LLM predictions with human labels in a principled way \citep[see also][]{broska2025mixed,ji2025predictions,wang2025finetune}.

\paragraph{Intuition behind PPI.} An equivalent regrouping of the same estimator makes the rectification logic clear.
Let $\tilde{\theta}_{\mathrm{LLM}} := \frac{1}{m}\sum_{i=1}^{m}\tilde{Y}_i^{\mathrm{LLM}}$ denote the LLM-only plug-in estimate, and let $\widehat{\Delta} := \frac{1}{n}\sum_{i=1}^{n}\bigl(Y_i^{\mathrm{LLM}}-Y_i\bigr)$ denote the labeled-sample estimate of the LLM's average prediction error. Then the PPI estimator can be written as
\[
\widehat{\theta}_{\mathrm{PPI}}
=
\underbrace{\tilde{\theta}_{\mathrm{LLM}}}_{\text{plug-in prediction term}}
-
\underbrace{\widehat{\Delta}}_{\text{bias-correction term / rectifier}}.
\]
The plug-in term $\tilde{\theta}_{\mathrm{LLM}}$ inherits the LLM's bias noted above, while $\widehat{\Delta}$ estimates exactly that bias, since $\bE[\widehat{\Delta}] = \bE[Y^{\mathrm{LLM}}-Y] = \bE[Y^{\mathrm{LLM}}]-\bE[Y]$. Subtracting it therefore ``rectifies'' the plug-in estimate and restores unbiasedness: $\bE[\widehat{\theta}_{\mathrm{PPI}}] = \bE[\tilde{\theta}_{\mathrm{LLM}}] - \bE[\widehat{\Delta}] = \bE[Y] = \theta^\star$.

In the synthetic-data-rich (SDR) regime ($m \to \infty$), the variance of the PPI estimator is:
\[
\Var\!\left(\widehat{\theta}_{\mathrm{PPI}}\right)
=
\frac{1}{n}\Var(Y - Y^{\mathrm{LLM}}).
\]
When the LLM closely tracks human responses, so that $\Cov(Y, Y^{\mathrm{LLM}})$ is large, the variance of the PPI estimator can be substantially smaller than that of the human-only sample mean, $\Var(Y)/n$. However, when the LLM signal is weak, the bias-correction term in the PPI estimator adds noise rather than removing it. A short calculation gives
\[
\Var(Y - Y^{\mathrm{LLM}})
\;=\;
\Var(Y)+\Var(Y^{\mathrm{LLM}})-2\,\Cov(Y, Y^{\mathrm{LLM}}),
\]
so $\Var(Y - Y^{\mathrm{LLM}}) > \Var(Y)$ precisely when $\Cov(Y, Y^{\mathrm{LLM}}) < \tfrac{1}{2}\Var(Y^{\mathrm{LLM}})$. Whenever this threshold fails, the PPI variance coefficient $\Var(Y - Y^{\mathrm{LLM}})$ exceeds the human-only coefficient $\Var(Y)$, making PPI strictly worse than the sample mean. This motivates a more flexible estimator that adapts the degree of reliance on the LLM.

\subsection{PPI++ estimator and its variance decomposition}

The \textsc{PPI++} estimator \citep{angelopoulos2023ppi} addresses the limitations of PPI by introducing a tuning parameter $\lambda \in [0,1]$ that controls the weight placed on the LLM correction as follows:
\begin{equation}
\label{eq:ppi++}
\widehat{\theta}(\lambda)
=
\frac{1}{n}\sum_{i=1}^{n}Y_i
+
\lambda\left(
\frac{1}{m}\sum_{i=1}^{m}\tilde{Y}^{\mathrm{LLM}}_i
-
\frac{1}{n}\sum_{i=1}^{n}Y_i^{\mathrm{LLM}}
\right).
\end{equation}
When $\lambda = 0$, the estimator ignores LLM predictions and reduces to the sample mean or human-only $\widehat{\theta}(0)=\frac{1}{n}\sum_{i=1}^{n}Y_i$. When $\lambda = 1$, it reduces to standard PPI, $\widehat{\theta}(1) = \widehat{\theta}_{\mathrm{PPI}}$. Further, like PPI, \textsc{PPI++} is unbiased for $\theta^\star$ for any fixed $\lambda$.

\begin{lemma}[MSE and Variance Decomposition for PPI++]
\label{lem:mse_var}
For any fixed $\lambda$, the \textsc{PPI++} estimator is unbiased for $\theta^\star$, so
\[
\bE\!\left[(\widehat{\theta}(\lambda)-\theta^\star)^2\right]
=
\Var\!\left(\widehat{\theta}(\lambda)\right).
\]
Moreover,
\[
\Var\!\left(\widehat{\theta}(\lambda)\right)
=
\frac{1}{n}\Var\!\left(Y-\lambda Y^{\mathrm{LLM}}\right)
+
\frac{\lambda^2}{m}\Var\!\left(\tilde{Y}^{\mathrm{LLM}}\right).
\]
In the SDR regime ($m\to\infty$), this reduces to $\Var\!\left(\widehat{\theta}(\lambda)\right) =
\Var\!\left(Y-\lambda Y^{\mathrm{LLM}}\right)/n$.
\end{lemma}

The proof is straightforward and is deferred to Web Appendix~\ref{app:proof_mse_var}. 

Lemma~\ref{lem:mse_var} provides the key intuition for the rest of the framework. The first term of the variance, $\Var(Y-\lambda Y^{\mathrm{LLM}})/n$, is the sampling variation from the labeled human sample after partial rectification of the LLM prediction. The second term, $\lambda^2 \Var(\tilde{Y}^{\mathrm{LLM}})/m$, comes from sampling variation in the synthetic LLM pool. Because LLM responses can be generated very cheaply, the synthetic-pool size $m$ is effectively a free design choice in our setting: the researcher can draw as many synthetic prompts as desired. In the SDR idealization of Section~\ref{sec:setup} ($m\to\infty$), the second variance term vanishes. In practice $m$ is finite but can be made large; the standard error then simply adds back the easily computable term $\lambda^2 \widehat{\Var}(\tilde{Y}^{\mathrm{LLM}})/m$, so the SDR idealization neither hides cost nor affects the validity of inference. As a result, the performance of \textsc{PPI++} is governed by the single quantity $\Var(Y-\lambda Y^{\mathrm{LLM}})$, which measures how much residual uncertainty remains after using human data to correct the LLM. This quantity will be the basis for our definition of rectification difficulty in the next subsection.

\subsection{Rectification difficulty}
\label{ssec:rect_diff}
Lemma~\ref{lem:mse_var} reduces the question-level performance of \textsc{PPI++}\ to a single scalar, $\Var(Y-\lambda Y^{\mathrm{LLM}})$: it plays exactly the role for \textsc{PPI++}\ that the population variance $\Var(Y)$ plays for the classical sample mean. We formally define this scalar as \emph{rectification difficulty}, which will serve as the central design primitive of our framework.

\begin{definition}[Rectification difficulty at tuning parameter $\lambda$]
\label{def:rect_diff_lambda}
We define the \emph{rectification difficulty} at tuning parameter $\lambda$ as
\begin{equation}
A(\lambda) := \Var(Y-\lambda Y^{\mathrm{LLM}}).
\end{equation}
\end{definition}
Intuitively, this quantity measures how hard it is to use human data to rectify the LLM for the question at hand. Smaller values of $A(\lambda)$ mean that the LLM already tracks the human outcome well after correction; larger values mean that more human data are needed to achieve a given level of precision. More precisely, because $A(\lambda)$ is the variance of the rectified outcome $Y - \lambda Y^{\mathrm{LLM}}$, \textsc{PPI++} precision depends on how strongly $Y^{\mathrm{LLM}}$ co-varies with $Y$ across respondents, not simply on how close their values are on average. An LLM that predicts the same value near $\bE[Y]$ for every respondent gives a good LLM-only estimator, yet offers essentially no rectification power.

Because $A(\lambda)$ is quadratic in $\lambda$, its unconstrained minimizer has the closed form
\[
\beta^\star
=
\frac{\Cov(Y,Y^{\mathrm{LLM}})}{\Var(Y^{\mathrm{LLM}})},
\]
the population regression (OLS) coefficient of $Y$ on $Y^{\mathrm{LLM}}$.
Following \citet{angelopoulos2023ppi}, the \textsc{PPI++} tuning clips this to $[0,1]$, $\lambda^\star=\Pi_{[0,1]}(\beta^\star)$ (with $\lambda^\star=0$ when $\Var(Y^{\mathrm{LLM}})=0$), so that $\lambda^\star$ minimizes $A(\lambda)$ over $\lambda\in[0,1]$. The \textsc{PPI++} estimator uses this variance-minimizing tuning, yielding variance
\begin{equation}
\Var\!\left(\widehat{\theta}(\lambda^\star)\right)
=
\frac{A(\lambda^\star)}{n}
\end{equation}
in the SDR regime. We can then write $A := A(\lambda^\star)$ for the minimized rectification difficulty and refer to it simply as the \emph{rectification difficulty} of the question. 

The rectification difficulty $A$ admits several useful interpretations. First, when the unconstrained coefficient lies in the unit interval ($\beta^\star\in[0,1]$, so $\lambda^\star=\beta^\star$), $A$ is the residual variance after optimally regressing $Y$ on $Y^{\mathrm{LLM}}$, and $A = \Var(Y)(1-\rho^2)$, where $\rho = \mathrm{Corr}(Y,Y^{\mathrm{LLM}})$.\footnote{This identity holds exactly because $\beta^\star$ is the population OLS coefficient of $Y$ on $Y^{\mathrm{LLM}}$ and $A$ is the resulting population residual variance; it is an algebraic identity that does not require joint normality or linearity of $\bE[Y\mid Y^{\mathrm{LLM}}]$.} More generally, when clipping binds, $A = A(\lambda^\star)$ is the minimum residual variance over the constrained class $\lambda\in[0,1]$. Second, the human-only sample mean corresponds to $\lambda=0$, for which $
A(0)=\Var(Y)$. Because $\lambda^\star$ minimizes $A(\lambda)$, we have
\[
A = A(\lambda^\star) \le A(0)=\Var(Y).
\]
Hence, \textsc{PPI++} weakly reduces the variance coefficient relative to the human-only sample mean: it uses the LLM when doing so is helpful and ignores it when it is not.  Third, restoring the question index $q$, let $\lambda_q^\star$ be the variance-minimizing tuning parameter and define
\begin{equation}
A_q := \Var\!\left(Y_q-\lambda_q^\star Y_q^{\mathrm{LLM}}\right).
\end{equation}
Under SDR ($m_q\to\infty$),
\begin{equation}
\Var\!\left(\widehat{\theta}_q(\lambda_q^\star)\right)=\frac{A_q}{n_q}.
\end{equation}
This variance-scaling relation is the bridge from estimation to design: it shows that, once rectification difficulty is known, the sample-allocation problem becomes a question of how to distribute the human budget across questions with different values of $A_q$. Section~\ref{sec:framework} develops this design problem. 

Although $\Var(Y_q - \lambda_q^\star Y_q^{\mathrm{LLM}})$ appears implicitly in \textsc{PPI++}\ variance derivations, prior work does not isolate it as a question-level object or use it for design. Our contribution is to identify $A_q$ as the correct variance index for hybrid human--LLM survey allocation. Two separate dominance properties then follow. First, at any fixed allocation, $A_q \le \Var(Y_q)$ ensures that \textsc{PPI++}\ weakly dominates the sample-mean estimator question by question. Second, under the \textsc{PPI++}\ design objective $J(\bm{A},\bm{n}) := \sum_q w_q A_q / n_q$ (developed in Section~\ref{sec:framework}), the $A_q$-based allocation $\bm{n}^\star(\bm{A})$ is by construction the optimizer and therefore weakly dominates any feasible alternative---including the classical Neyman rule applied to $\Var(Y_q)$---under that objective.

\paragraph{Rectification difficulty vs. Accuracy.} It is useful to note that rectification difficulty is distinct from raw LLM accuracy. Accuracy evaluates the LLM response as a substitute for the human response, penalizing discrepancies between $Y^{\mathrm{LLM}}$ and $Y$. Rectification difficulty instead evaluates the LLM as an auxiliary variable in a corrected estimator. Since \textsc{PPI++} uses human labels to estimate and remove systematic LLM errors, precision is governed by the residual respondent-level variation in $Y-\lambda^\star Y^{\mathrm{LLM}}$, not by average closeness in levels. Thus, an LLM can be inaccurate but easy to rectify if its errors are systematic: if $Y^{\mathrm{LLM}}=Y+c$, raw accuracy can be poor when $c$ is large, but $A=0$ because the correction removes the constant shift. Conversely, an LLM can be accurate in aggregate but provide no variance reduction: if $Y^{\mathrm{LLM}}=\bE[Y]$ for every respondent, the LLM-only mean is correct, but $\Var(Y^{\mathrm{LLM}})=0$, $\lambda^\star=0$, and $A=\Var(Y)$. Section~\ref{subsec:ppi_diagnostics} documents this dissociation empirically.

\paragraph{Rectification difficulty and LLM heterogeneity.}
A direct implication is that rectification difficulty is governed by whether the LLM reproduces \emph{individual-level} heterogeneity, not by whether it matches the population mean. When the LLM returns nearly the same prediction for every respondent, $\Var(Y^{\mathrm{LLM}})$ is close to zero, the optimal tuning $\lambda^\star$ collapses to zero, and $A_q$ reverts to the human-only variance $\Var(Y_q)$, so the LLM contributes no rectification power even when its average is accurate. This links rectification difficulty to a well-documented limitation of LLM respondents, namely that they tend to understate within-population heterogeneity and collapse toward modal responses \citep{peng2025mega, toubia2025database}. Questions on which this collapse is most severe are exactly those with the highest $A_q$, and the allocation rule responds by directing human respondents toward them. Seen this way, $A_q$ also serves as a common yardstick for efforts to improve synthetic respondents: any intervention that injects genuine respondent-level variation into the LLM, whether through richer persona prompting, conditioning on individual covariates, or domain-specific fine-tuning \citep{wang2025finetune}, lowers $A_q$ and thereby reduces the human budget needed to reach a target precision. Our framework is therefore complementary to that line of work; it does not require the LLM to capture heterogeneity, but it quantifies, through $A_q$, how much any such improvement is worth for survey design.

\section{Survey Design Framework}
\label{sec:framework}

In Section~\ref{sec:ppi_pp}, we showed that, under the variance-minimizing \textsc{PPI++} tuning, the estimator for question $q$ has variance $\Var(\widehat{\theta}_q(\lambda_q^\star)) = A_q/n_q$ in the SDR regime. In this section, we use that variance characterization to derive an optimal allocation framework that addresses our central design question of how to choose the number of human responses, for each question $q$, $\{n_q\}_{q\in\mathcal{T}}$, for the new survey questions in $\mathcal{T}$.

This section is organized as follows. First, in Section~\ref{subsec:oracle_allocation}, we solve the allocation problem under the benchmark case in which the rectification difficulties $\{A_q\}_{q\in\mathcal{T}}$ are known. This yields an oracle allocation rule and clarifies what the efficient design would look like if the design-relevant quantities were observed. Second, because the rectification diffculties are unknown for new questions, Section~\ref{subsec:allocation_unknown} introduces a two-phase procedure that operationalizes the oracle rule with predicted difficulties: Phase~A (Section~\ref{subsec:meta_learning}) trains a meta-learner on historical paired human--LLM data to predict $\tilde{A}_q$ for each target question, and Phase~B (Section~\ref{subsec:end_to_end}) combines these predictions with the \textsc{PPI++}\ estimator and establishes end-to-end inference guarantees. Finally, in Section~\ref{subsec:extensions} we discuss several extensions of the framework.

\subsection{Optimal allocation with known rectification difficulty}
\label{subsec:oracle_allocation}

We begin with the oracle benchmark in which the rectification difficulties $\{A_q\}_{q\in\mathcal{T}}$ are known. By Lemma~\ref{lem:mse_var}, under the variance-minimizing tuning $\lambda=\lambda_q^\star$, the mean squared error equals the variance:
\[
\bE\!\left[\bigl(\widehat{\theta}_q(\lambda_q^\star)-\theta_q^\star\bigr)^2\right]
=
\Var\!\left(\widehat{\theta}_q(\lambda_q^\star)\right)
=
\frac{A_q}{n_q},
\]
in the SDR regime where $m_q\to\infty$. Thus, if $\{A_q\}$ were known, choosing the human sample sizes $\{n_q\}$ reduces to a deterministic variance-minimization problem.

For each question $q$, take as given: (i) rectification difficulty $A_q>0$, (ii) per-response human cost $c_q>0$, and (iii) an importance weight $w_q\ge 0$ reflecting how much the researcher values precision on that question. The total human-data budget is $B$. The oracle allocation problem is
\begin{align}
\label{eq:opt_problem}
\min_{\{n_q\}_{q\in\mathcal{T}}}
\quad & \sum_{q\in\mathcal{T}} w_q \frac{A_q}{n_q} \quad
\text{s.t.} \quad \sum_{q\in\mathcal{T}} c_q n_q \le B, \quad  n_q \ge 0 \quad \forall \; q \in \mathcal{T}.
\end{align}

\begin{theorem}[Oracle Neyman-type allocation under known rectification difficulty]
\label{thm:optimal_allocation}
Suppose $A_q>0$ and $c_q>0$ for all $q\in\mathcal{T}$, and assume there exists at least one $q$ with $w_q>0$. Then the unique optimal solution to Equation~\eqref{eq:opt_problem} satisfies $n_q^\star=0$ for all $q$ with $w_q=0$, and for $q$ with $w_q>0$,
\begin{equation}
\label{eq:closed_form_sol}
n_q^\star
=
\frac{B}{\sum_{j\in \mathcal{T}: w_j>0} \sqrt{w_j A_j c_j}}
\sqrt{\frac{w_q A_q}{c_q}}.
\end{equation}
\end{theorem}
The full proof is given in Web Appendix~\ref{ssec:proof_thm_opt_all}. The objective is strictly convex and the constraint linear, so a unique minimum exists; the closed-form solution follows from the first-order conditions.

The square-root form in Equation~\eqref{eq:closed_form_sol} is a direct analogue of Neyman allocation in stratified sampling \citep{neyman1934two}, with questions playing the role of strata and rectification difficulty $A_q$ replacing the stratum variance. The allocation rule itself is intuitive: more important questions (larger $w_q$) receive more human labels; more expensive labels (larger $c_q$) receive fewer allocations; and questions on which the LLM is harder to rectify (larger $A_q$) receive more human labels. When the LLM closely tracks humans, $A_q$ is small and only a small human sample is needed; when the LLM is unreliable, $A_q$ is large and the allocation shifts toward that question.

\subsection{Allocation with unknown rectification difficulty}
\label{subsec:allocation_unknown}

The oracle rule in Theorem~\ref{thm:optimal_allocation} assumes that the rectification difficulties for the set of questions in the survey $\mathcal{T}$ ($\{A_q\}_{q\in\mathcal{T}}$) are known. For the setting that motivates this paper, i.e., a \emph{new} survey, rectification difficulties are not known. Estimating $A_q$ from the target survey itself would require collecting exactly the human data whose allocation we are trying to optimize. This is the core cold-start problem of LLM-augmented survey design. We resolve this with a two-phase \emph{survey design framework}, summarized below. Phase A (meta-learning) uses the historical corpus $\mathcal{H}$, where paired human responses and LLM predictions are available, to learn how rectification difficulty depends on question features, and then uses that learned mapping to predict $\tilde{A}_q$ for each target question $q \in \mathcal{T}$. Phase B (allocation and inference) plugs the predictions into the oracle allocation rule, collects human responses accordingly, and employs the \textsc{PPI++}\ estimator for each target question.

\begin{framed}
\noindent\textbf{Survey Design Framework} (full formulas in Sections~\ref{subsec:meta_learning}--\ref{subsec:end_to_end}).

\medskip
\textbf{Phase A: Meta-learning} (train on $\mathcal{H}$, predict difficulty for $\mathcal{T}$).
\begin{adjustwidth}{1.5em}{0pt}
\textbf{Step 1.} Estimate the rectification difficulty $\widehat{A}_q^{\,\mathcal{H}}$ on each historical question $q \in \mathcal{H}$.

\textbf{Step 2.} Fit a meta-learner $\widehat{\phi}$ from question features $\bm{z}_q$ to $\log A_q$.

\textbf{Step 3.} Predict $\tilde{A}_q = \exp(\widehat{\phi}(\bm{z}_q))$ for each target question $q \in \mathcal{T}$.
\end{adjustwidth}

\bigskip
\textbf{Phase B: Allocation and inference} (on $\mathcal{T}$).
\begin{adjustwidth}{1.5em}{0pt}
\textbf{Step 4.} Allocate $\tilde{n}_q^\star \propto \sqrt{w_q \tilde{A}_q / c_q}$ subject to $\sum_q c_q \tilde{n}_q^\star = B$ (Theorem~\ref{thm:optimal_allocation}).

\textbf{Step 5.} Field the survey: collect $\tilde{n}_q^\star$ human responses per question and their paired LLM predictions.

\textbf{Step 6.} Recompute $\widehat{\lambda}_q$ on the collected data and form the \textsc{PPI++}\ estimator $\widehat{\theta}_q(\widehat{\lambda}_q)$, with valid confidence intervals.
\end{adjustwidth}
\end{framed}

\subsubsection{Phase A: Meta-learning}
\label{subsec:meta_learning}

To make the oracle rule operational, we use the historical corpus $\mathcal{H}$, where paired human responses and LLM predictions are already available, to learn how rectification difficulty varies across questions. The idea is that questions that are hard for LLM-assisted estimation in past surveys often share characteristics with questions that will be hard in new surveys. We exploit this regularity to predict design-relevant difficulty for the target questions before any new human responses are collected.

The key innovation is that we do not try to predict the target survey responses themselves. Instead, we predict a question-level \emph{difficulty of rectification}: how hard it will be to use human data to correct the LLM for that question. This is exactly the quantity needed for design. Phase A accomplishes this in three steps.

\textbf{Step~1 (Estimate rectification difficulty on $\mathcal{H}$).}
For each historical question $q\in\mathcal{H}$, we compute the plug-in tuning parameter $\widehat{\lambda}_q^{\,\mathcal{H}} = \widehat{\Cov}(Y_q,Y_q^{\mathrm{LLM}})/\widehat{\Var}(Y_q^{\mathrm{LLM}})$
clipped to $[0,1]$ (and set to $0$ when $\widehat{\Var}(Y_q^{\mathrm{LLM}})=0$), together with the plug-in rectification difficulty
\[
\widehat{A}_q^{\,\mathcal{H}}
:=
\widehat{\Var}\!\left(Y_q-\widehat{\lambda}_q^{\,\mathcal{H}} Y_q^{\mathrm{LLM}}\right),
\]
computed from the paired sample $\{(Y_{qi},Y_{qi}^{\mathrm{LLM}})\}_{i=1}^{n_q^{\mathcal{H}}}$.
Under standard regularity conditions, $\widehat{A}_q^{\,\mathcal{H}} \to A_q$ in probability \citep{angelopoulos2023ppi}, where $A_q = A_q(\lambda_q^\star)$ is the rectification difficulty defined in Section~\ref{ssec:rect_diff}. The superscript $\mathcal{H}$ distinguishes this historical-sample estimate from the target-sample plug-in $\widehat{\lambda}_q$ used in Phase B below.

\textbf{Step~2 (Train predictive model).}
We represent each historical question by a feature vector $\bm{z}_q\in\mathbb{R}^K$---for example, LLM embeddings of the question wording and response options, the question type, and the number of feasible choice options---yielding training samples $\{(\bm{z}_q, \log \widehat{A}_q^{\,\mathcal{H}})\}_{q\in\mathcal H}$. We then fit a predictive model $\bm{z}_q\mapsto \phi(\bm{z}_q)$ by minimizing squared error in log difficulty:
\[
\widehat{\phi}
\in
\arg\min_{\phi\in\Phi}
\sum_{q\in\mathcal H}
\bigl(\log \widehat{A}_q^{\,\mathcal{H}}-\phi(\bm{z}_q)\bigr)^2,
\]
possibly with regularization. We model $\log A_q$ rather than $A_q$ directly because rectification difficulty is positive and can vary by orders of magnitude, so the log scale is more stable. The function class $\Phi$ is flexible and may consist of linear models, trees, or neural networks.

\textbf{Step~3 (Predict for $\mathcal{T}$).}
For each new question $q\in\mathcal{T}$, we predict
\[
\tilde{A}_q := \exp\!\bigl(\widehat{\phi}(\bm{z}_q)\bigr).
\]
We use the tilde notation here, rather than $\widehat{A}_q^{\,\mathcal{H}}$, to emphasize that $\tilde{A}_q$ is a design-stage prediction for a new survey question for which human responses have not yet been collected.

In sum, Phase A converts historical paired human--LLM data into design-stage predictions of rectification difficulty for new questions. This is what bridges the gap between the oracle allocation problem and a practical survey-design procedure: it enables zero-shot allocation of human effort before the target survey has been fielded.

\subsubsection{Phase B: Allocation and inference}
\label{subsec:end_to_end}

With the predicted difficulties $\widetilde{\bm{A}} = (\tilde{A}_q)_{q\in\mathcal T}$ in hand, the researcher can proceed from design to estimation. Phase B consists of three steps.

\textbf{Step~4 (Allocate human labels).}
Plug $\widetilde{\bm{A}}$ into the oracle rule in Theorem~\ref{thm:optimal_allocation} to obtain the design-stage allocation $\tilde{n}_q^\star$ for each $q\in\mathcal{T}$, rounding the continuous allocation to integers via largest-remainder rounding. In practice one would also impose a minimum allocation $\tilde{n}_q^\star \ge n_{\min}$ so that every fielded question has enough responses to estimate the plug-in tuning $\widehat{\lambda}_q$ and form a valid Wald interval; the simulations below use the unconstrained allocation and therefore include very small per-question samples at the lowest budgets, which we flag when relevant in our empirical analysis.

\textbf{Step~5 (Collect human responses).}
For each $q\in\mathcal{T}$, collect $\tilde{n}_q^\star$ human responses and record their paired LLM predictions.

\textbf{Step~6 (Estimate and conduct inference).}
For each $q\in\mathcal{T}$, compute the target-sample plug-in tuning parameter $\widehat{\lambda}_q = \widehat{\Cov}(Y_q,Y_q^{\mathrm{LLM}}) / \widehat{\Var}(Y_q^{\mathrm{LLM}})$ on the collected $\tilde{n}_q^\star$ labeled pairs (clipped to $[0,1]$), and form the \textsc{PPI++}\ estimator $\widehat{\theta}_q(\widehat{\lambda}_q)$ via Equation~\eqref{eq:ppi++}.

We now formalize the inference guarantees for the resulting estimator.

\begin{proposition}[Asymptotic normality of \textsc{PPI++}]
\label{thm:end_to_end}
Fix a target question $q\in\mathcal{T}$, and assume: (i) $\{(\bm{X}_{qi},Y_{qi})\}_{i=1}^{\tilde n_q^\star}$ are i.i.d.\ draws from $\cP_q$ with $\bE[Y_q^4]<\infty$ and $\bE[(Y_q^{\mathrm{LLM}})^4]<\infty$; (ii) the labeled sample and the synthetic pool $\{\widetilde{\bm{X}}_{qj}\}_{j=1}^{m_q}$ are drawn independently; (iii) $\tilde n_q^\star\to\infty$, $m_q\to\infty$, and $m_q/\tilde n_q^\star\to\infty$ (synthetic-data-rich regime); and (iv) $A_q>0$ (the degenerate case $A_q=0$, in which $Y_q-\lambda_q^\star Y_q^{\mathrm{LLM}}$ is almost surely constant, is excluded). Adopt the clipping convention of Section~\ref{sec:ppi_pp}: $\lambda_q^\star$ is the ratio $\Cov(Y_q, Y_q^{\mathrm{LLM}})/\Var(Y_q^{\mathrm{LLM}})$ projected onto $[0,1]$, with $\lambda_q^\star := 0$ when $\Var(Y_q^{\mathrm{LLM}})=0$, and analogously for the sample plug-in $\widehat{\lambda}_q$. Then, conditional on the historical corpus $\mathcal{H}$ and the resulting design-stage allocation $\tilde n_q^\star$,
\[
\frac{\widehat{\theta}_q(\widehat{\lambda}_q)-\theta_q^\star}{\sigma_q}
\xrightarrow{d}
\mathcal N(0,1),
\qquad
\sigma_q^2 \;:=\; \frac{A_q}{\tilde{n}_q^\star} \;+\; \frac{(\lambda_q^\star)^2\,\Var(Y_q^{\mathrm{LLM}})}{m_q},
\]
where $A_q = \Var\!\left(Y_q-\lambda_q^\star Y_q^{\mathrm{LLM}}\right)$.
\end{proposition}

This is a direct consequence of Theorem~1 in \citet{angelopoulos2023ppi}; we provide a self-contained proof in Web Appendix~\ref{app:end_to_end_proof}. The estimation error in $\widehat{\lambda}_q$ is $O_p((\tilde{n}_q^\star)^{-1/2})$, while its effect on $\widehat{\theta}_q(\widehat{\lambda}_q)$ is only $O_p((\tilde{n}_q^\star)^{-1})$, which is asymptotically negligible relative to the $O_p((\tilde{n}_q^\star)^{-1/2})$ sampling error. Thus, the plug-in estimator inherits the CLT of the oracle-$\lambda_q^\star$ case. 


An important implication is that prediction errors in $\tilde A_q$ affect the efficiency of the allocation but not its asymptotic validity. Concretely, the standard Wald interval
\[
\widehat{\theta}_q(\widehat{\lambda}_q) \;\pm\; z_{1-\alpha/2}\,\widehat{\sigma}_q,
\qquad
\widehat{\sigma}_q^2 \;:=\; \frac{\widehat{A}_q}{\tilde n_q^\star} \;+\; \frac{\widehat{\lambda}_q^{\,2}\,\widehat{\Var}(Y_q^{\mathrm{LLM}})}{m_q},
\]
constructed entirely from target-sample quantities (labeled pairs plus the synthetic pool), attains nominal $(1-\alpha)$ coverage in the SDR limit for \emph{any} positive $\tilde A_q$: Phase-A prediction error rescales the allocated sample size $\tilde n_q^\star$, and hence the width of the interval, but does not enter the coverage probability.

The next proposition quantifies the efficiency loss from misspecified predicted difficulties.

\begin{proposition}[Sensitivity to misspecified rectification difficulty]
\label{prop:robust_allocation}
Let $\bm{A} = (A_1, \dots, A_Q)$ and $\widetilde{\bm{A}}=(\tilde{A}_1,\dots,\tilde{A}_Q)$ be positive vectors, and let $\bm{n}^\star(\bm{A})$ and $\widetilde{\bm{n}}^\star = \bm{n}^\star(\widetilde{\bm{A}})$ denote the corresponding optimal continuous allocations. Suppose $|\log \tilde{A}_q-\log A_q|\le \varepsilon$ for all $q\in\mathcal{T}$. Then:
\begin{enumerate}[(i)]
\item $J(\bm{A},\widetilde{\bm{n}}^\star)\ge J^\star(\bm{A})$, with equality if and only if $\tilde{A}_q$ is proportional to $A_q$ across the questions with $w_q>0$.
\item The relative efficiency loss satisfies
\[
1 \le \frac{J(\bm{A},\widetilde{\bm{n}}^\star)}{J^\star(\bm{A})} \le \cosh^2(\varepsilon/2) \le e^\varepsilon ,
\]
where $\cosh^2(\varepsilon/2) = 1+\tfrac{\varepsilon^2}{4}+O(\varepsilon^4)$ as $\varepsilon\to0$.

\end{enumerate}
\end{proposition}

The proof is given in Web Appendix~\ref{subsec:robustness_A}. The practical reading of Proposition~\ref{prop:robust_allocation} is that only \emph{relative} errors in $\tilde{A}_q$ matter. By the equality condition in part~(i), multiplying all $\tilde{A}_q$ by a common constant leaves the allocation unchanged: a meta-learner that over- or under-predicts every difficulty by the same factor incurs no efficiency loss. Indeed, the main source of efficiency loss is misranking questions relative to one another. For the same reason, the worst-case bound in part~(ii) ($\cosh^2(\varepsilon/2)\approx 1+\varepsilon^2/4$, second order in $\varepsilon$) is conservative, because a common multiplicative error inflates $\varepsilon$ yet causes no loss. As a sharper typical-case benchmark, Web Appendix~\ref{app:avg_case_robustness} shows that when the log-scale prediction errors $\log\tilde{A}_q - \log A_q$ are approximately Gaussian with variance $\sigma_\eta^2$, the expected efficiency loss is approximately $1+\sigma_\eta^2/4$; we treat this as descriptive rather than a formal guarantee. Empirically, the meta-learned allocation recovers 61--79\% of the oracle gains out of sample within each survey (Sections~\ref{sec:empirics} and~\ref{sec:cces}), and 68\% of the oracle gains in a harder cross-dataset test in which the meta-learner is trained on one survey and applied to a different survey generated by a different LLM (Section~\ref{subsec:robustness_eval}), indicating robustness to imperfect predictions and to distribution shift between $\mathcal{H}$ and $\mathcal{T}$.

In sum, the two-phase framework turns the oracle benchmark of Section~\ref{subsec:oracle_allocation} into an operational design procedure. Its inputs are modest: a historical corpus of paired human--LLM responses (which any organization that has previously fielded surveys already possesses), text embeddings of the old and new questions, and a budget. Its outputs are: (i) the sample sizes $\{\tilde n_q^\star\}$ and (ii) after fielding the new survey, \textsc{PPI++}\ estimates with valid confidence intervals. The division of labor between the two phases is what makes the procedure reliable in practice: meta-learning affects only where human labels go, and Proposition~\ref{prop:robust_allocation} bounds the efficiency cost of its errors, while statistical validity rests solely on the post-collection estimator (Proposition~\ref{thm:end_to_end}). To our knowledge, no existing method provides this combination of allocation rule across heterogeneous and new survey questions and end-to-end inference guarantees.

\subsection{Extensions}
\label{subsec:extensions}

The framework above is presented for scalar population means under a weighted-MSE objective and a per-question allocation cost. However, the machinery is quite general and can be extended in two broad directions. We discuss each below:

\subsubsection{Direction I: Alternative design problems.}
The first set of extensions retains the \textsc{PPI++}\ mean estimator and varies only the design problem (objective, constraint, or allocation unit). Because Proposition~\ref{thm:end_to_end} holds for any fixed allocation, the inference guarantees carry over unchanged to every variant in this direction: once the allocation is set and data are collected, $\widehat{\theta}_q(\widehat{\lambda}_q)$ delivers valid inference regardless of how the allocation was chosen.

\squishlist
\item \textbf{Cost-minimizing dual formulation.} The budget-constrained problem in Equation~\eqref{eq:opt_problem} has a natural counterpart: given a target weighted MSE $J_{\mathrm{target}}$, minimize the total human-data cost $\sum_q c_q n_q$ subject to $\sum_q w_q A_q / n_q \le J_{\mathrm{target}}$. Both problems share the same first-order conditions up to rescaling of the Lagrange multiplier, so the optimal allocation proportions are identical: $n_q^\star \propto \sqrt{w_q A_q / c_q}$. Only the overall scale changes. For a single question with unit cost, this reduces to the standard \textsc{PPI++}\ sample-size determination $n^\star = A_q / J_{\mathrm{target}}$, with $A_q$ replacing the classical variance $\Var(Y_q)$. See Web Appendix~\ref{rmk:dual_and_single_q}.

\item \textbf{Alternative loss functions.} The framework also accommodates loss functions other than MSE. Under a weighted mean-absolute-error (MAE) objective $\sum_q w_q \, \bE\bigl[|\widehat{\theta}_q(\lambda_q^\star) - \theta_q^\star|\bigr]$, asymptotic normality (Proposition~\ref{thm:end_to_end}) gives $\bE[|\widehat{\theta}_q - \theta_q^\star|] \propto \sqrt{A_q / n_q}$, and substituting into the budget-constrained problem yields the cube-root rule $n_q^\star \propto (w_q / c_q)^{2/3} A_q^{1/3}$. The cube-root dependence on $A_q$ is less sensitive to heterogeneity in $\{A_q\}$ than the square-root rule under MSE, which can be desirable when difficulty predictions are noisy. Other loss functions are handled analogously, with the allocation exponent determined by how the expected loss scales with $n_q$: absolute-error-type objectives (e.g., quantile losses) behave like MAE and yield cube-root-type rules, while locally quadratic losses (e.g., Huber) retain the $1/n_q$ scaling and the square-root rule. See Web Appendix~\ref{rmk:alternative_loss}.

\item \textbf{Power analysis for hypothesis testing.} The variance scaling $\Var(\widehat{\theta}_q) = A_q / n_q$ also supports \textsc{PPI++}-based power analysis. To detect an effect of size $\delta_q$ at significance level $\alpha$ with power $1 - \beta$, the standard calculation gives the required sample size $n_q = A_q \, (z_{1-\alpha/2} + z_{1-\beta})^2 / \delta_q^2$, where $z_p$ is the standard-normal $p$-quantile. Because $A_q \le \Var(Y_q)$, the \textsc{PPI++}-based design always requires weakly fewer human labels than the classical human-only power calculation. The framework also accommodates simultaneous multiple-hypothesis testing through standard global error-rate corrections (e.g., Bonferroni FWER, \v{S}id\'ak under independence). See Web Appendix~\ref{rmk:power_analysis}.

\item \textbf{Wave-level (block-level) allocation.} Our main framework treats per-question allocation, which is appropriate when questions can be fielded independently. In many surveys, however, respondents answer multiple questions in a single session (matrix sampling), so responses are correlated within a respondent. A natural extension treats each survey \emph{wave}---a block of questions fielded jointly to a shared respondent pool---as the allocation unit. Stacking the within-wave estimating equations, the wave-level sandwich covariance accounts for the within-respondent correlations, and its scalarization $A_w$ feeds into the same square-root rule, $n_w^\star \propto \sqrt{w_w A_w / c_w}$. Section~\ref{subsec:robustness_eval} validates this extension empirically by treating each Twin-2K-500 task as a wave. See Web Appendix~\ref{rmk:matrix_sampling}.
\squishend

\subsubsection{Direction II: Beyond scalar mean estimands.}
The second direction generalizes the estimand itself, which requires its own asymptotic theory rather than a direct appeal to Proposition~\ref{thm:end_to_end}.

\squishlist
\item \textbf{General $M$-estimation.} The variance scaling $A_q / n_q$ extends well beyond scalar population means to a broad class of $M$-estimators that are central in marketing applications: linear regression coefficients (used in demand and pricing models), categorical response proportions (e.g., choice shares, brand preference), and multinomial-logit partworths (used in conjoint analysis). For these estimands, the asymptotic variance of \textsc{PPI++}\ takes a sandwich form $\bm{\Sigma}_q(\lambda) = \bm{H}_q^{-1} \bm{V}_q^{\Delta(\lambda)} \bm{H}_q^{-\top}$ that retains the $1/n_q$ scaling, so the allocation framework applies after scalarizing this matrix into a scalar difficulty index $A_q$ (e.g., trace for A-optimality, determinant for D-optimality). Web Appendix~\ref{appsec:M-estimation} presents the sandwich derivation, scalarization criteria, and analysis for a broad set of estimands (e.g., OLS regression, multinomial logit, and categorical response), together with the corresponding asymptotic-normality result.
\squishend

\section{Empirical Analysis on Twin-2K-500}
\label{sec:empirics}

We now turn to the empirical evaluation of our framework on the Twin-2K-500 digital-twin dataset \citep{toubia2025database}. 
This dataset is especially well-suited to our analysis for four reasons. First, it provides paired human and LLM responses at the respondent--question level, which is exactly the structure needed to compute question-specific rectification difficulty. Second, it spans a broad range of behavioral-economics tasks and response formats, creating the heterogeneity in LLM reliability that our allocation framework is designed to exploit. Third, the large number of respondents per question allows us to treat the full observed sample as a ground-truth population for controlled Monte Carlo evaluation. Fourth, the multi-wave design, including the Wave~4 retest battery, enables robustness checks on the stability of rectification difficulty across waves.

In Section~\ref{subsec:data_setup}, we describe the data and our setup. In Section~\ref{subsec:ppi_diagnostics}, we examine \textsc{PPI++} and the rectification difficulty $A_q$ of questions in this dataset. In Section~\ref{subsec:meta_learning_results}, we evaluate the meta-learning model's predictive accuracy. In Section~\ref{subsec:performance_eval}, we compare our proposed method with other benchmarks. Finally, in Section~\ref{subsec:robustness_eval}, we present a series of robustness checks.

\subsection{Dataset and setup}
\label{subsec:data_setup}

\paragraph{Data source.}
We use the Twin-2K-500 dataset, a large-scale digital twin benchmark in which over 2{,}000 US respondents each answer 500+ questions spanning demographics, psychological scales, cognitive tasks, economic preferences, and behavioral economics experiments. Each human respondent is paired with a ``digital twin,'' a comprehensive text profile of the respondent, including demographic information. This profile is combined with survey questions to prompt the LLM, and the LLM's responses to each question are recorded.

The dataset contains four survey waves: Waves~1--3 (test) administered the full question battery, and Wave~4 (retest) re-administered a subset of questions to the same respondent panel. Following \citet{toubia2025database}, we focus on the Wave~4 retest battery, which enables a test-retest robustness check, and restrict attention to multiple-choice questions with varying response formats (binary, Likert, and other ordered scales). We exclude product-pricing and open-ended questions. The final sample comprises 68 questions spanning 14 task types. A \emph{task} is a templated question family sharing the same experimental structure but varying the specific item; for example, the ``Nonseparability of Risks and Benefits'' task contains multiple questions that each ask respondents to rate a different technology (bicycles, nuclear power, pesticides, etc.) on the same benefit scale. In total, we have 86{,}448 respondent--question observations, with mean $n_q = 1{,}271$ and range 651--2{,}058.

For cross-question comparability, all responses are rescaled to $[0,1]$ via min--max normalization using each question's scale endpoints, and the population mean $\theta_q^\star = \bE[Y_q]$ on this scale is our estimand.\footnote{Although the questions are multiple-choice, the response options are ordinal (e.g., Likert scales representing degrees of agreement), so rescaling to $[0,1]$ and computing population means is meaningful. This follows the same normalization used in \citet{toubia2025database}.}

\paragraph{LLM predictions and historical/target split.}
To ensure replicability, we use the dataset's pre-computed LLM-imputed responses rather than making new LLM calls; we refer readers to \citet{toubia2025database} for the imputation and prompting details. Following \citet{toubia2025database}, we report LLM accuracy for question $q$ as the mean normalized agreement between human and LLM responses,
\[
\mathrm{acc}_q \;:=\; 1 - \frac{1}{n_q}\sum_{i=1}^{n_q} \bigl|Y_{qi} - Y_{qi}^{\mathrm{LLM}}\bigr|,
\]
on the $[0,1]$-rescaled responses. We use this metric only as a descriptive summary; as discussed in Sections~\ref{ssec:rect_diff} and \ref{subsec:ppi_diagnostics}, it is the rectification difficulty $A_q$, not $\mathrm{acc}_q$, that governs \textsc{PPI++}\ design.

Because the dataset contains a fixed pool of questions, we emulate the historical/target split using task-held-out folds: questions in held-out tasks play the role of $\mathcal{T}$, while the remaining tasks serve as $\mathcal{H}$. We hold out at the task level rather than the question level because questions within the same task are highly similar, so question-level holdout would leak template-specific information and overstate out-of-sample predictive accuracy. Task-level holdout is therefore a more conservative and more appropriate benchmark.



\subsection{PPI++ diagnostics and rectification difficulty}
\label{subsec:ppi_diagnostics}

\paragraph{Full-sample diagnostic quantities.}
In this subsection, we use the full observed population for each question to characterize how useful the LLM is in this dataset. For each question, we form the full-sample plug-ins $\widehat{\lambda}^{\,\mathrm{full}}_q$ and $\widehat{A}^{\,\mathrm{full}}_q$ using the Step~1 formulas of the proposed framework in Section~\ref{subsec:allocation_unknown}, with $\widehat{\lambda}^{\,\mathrm{full}}_q$ clipped to $[0,1]$. We treat them as the best available proxies for $\lambda_q^\star$ and $A_q$ throughout the empirical analysis.

\paragraph{LLM accuracy is not an ideal design index for \textsc{PPI++}.}
LLM predictions are reasonably accurate on average in this dataset---$\overline{\mathrm{acc}}_q = 75.7\%$ across the 68 questions, not far below humans' own test--retest agreement of $81.7\%$. Table~\ref{tab:Aq-by-task} aggregates the question-level diagnostics to the 14 behavioral-economics tasks, sorts rows by $\overline{\widehat{A}^{\,\mathrm{full}}_q}$ (easiest to hardest to rectify), and reports within-task averages together with each task's rank by $\overline{\mathrm{acc}}_q$ (1 = highest accuracy). If accuracy tracked rectification difficulty, the accuracy ranks would run $1, 2, \ldots, 14$ down the rows; instead they are scrambled, and the disagreement is sharpest on the cases that matter for design. Anchoring is the clearest example: it sits at row 10 (quite hard to rectify, $\overline{\widehat{\lambda}^{\,\mathrm{full}}_q} = 0.01$) yet has accuracy rank 4---the LLM gets the population mean about right, so aggregate accuracy looks favorable, but its per-respondent predictions carry almost no information about $Y_q$ after partial rectification, and \textsc{PPI++} provides essentially no variance reduction. Because \textsc{PPI++} precision depends on $\Cov(Y_q, Y_q^{\mathrm{LLM}})$ rather than on $\mathrm{acc}_q$, the right quantity to predict at the design stage is $\widehat{A}^{\,\mathrm{full}}_q$ itself.\footnote{A task-level breakdown of LLM accuracy and rectification difficulty, together with human test--retest reliability comparisons and a simple illustrative example, is provided in Web Appendix~\ref{appsec:data_summary}.}

\begin{table}[ht!]
\centering
\caption{Task-level summary, Twin-2K-500.}
\label{tab:Aq-by-task}
\footnotesize
\begin{tabular}{lcccccc}
\toprule
Task & \#Qs & $\bar{N}$ & $\overline{\widehat{A}^{\,\mathrm{full}}_q}$ & $\overline{\widehat{\lambda}^{\,\mathrm{full}}_q}$ & $\overline{\mathrm{acc}}_q$ & Acc.\ rank \\
\midrule
Outcome Bias & 2 & 1029 & 0.041 & 0.36 & 85.9\% & 1 \\
Conjunction (Linda) & 6 & 1029 & 0.048 & 0.07 & 78.0\% & 5 \\
Nonseparability of Risks and Benefits & 8 & 2058 & 0.051 & 0.21 & 80.5\% & 2 \\
WTA/WTP – Thaler Problem & 3 & 686 & 0.054 & 0.31 & 79.2\% & 3 \\
Framing Problem & 2 & 1029 & 0.059 & 0.01 & 71.9\% & 11 \\
Myside Bias & 2 & 1029 & 0.062 & 0.28 & 77.6\% & 6 \\
False Consensus & 10 & 2058 & 0.074 & 0.45 & 74.2\% & 9 \\
Less is More & 9 & 686 & 0.087 & 0.38 & 75.7\% & 8 \\
Omission Bias & 1 & 2058 & 0.105 & 0.40 & 73.1\% & 10 \\
Anchoring & 4 & 1029 & 0.145 & 0.01 & 78.7\% & 4 \\
Probability Matching vs Maximizing & 16 & 1030 & 0.170 & 0.00 & 76.6\% & 7 \\
Abs vs Rel Savings & 2 & 1029 & 0.203 & 0.03 & 66.4\% & 12 \\
Allais Problem & 2 & 1029 & 0.234 & 0.00 & 50.6\% & 14 \\
Denominator Neglect & 1 & 2058 & 0.235 & 0.00 & 56.8\% & 13 \\
\midrule
Overall & 68 & 1271 & 0.106 & 0.19 & 75.7\% & --- \\
\bottomrule
\end{tabular}

\medskip
\begin{minipage}{\textwidth}
\scriptsize\noindent\textit{Notes.} Within-task averages of full-sample plug-ins $\widehat{A}^{\,\mathrm{full}}_q$, $\widehat{\lambda}^{\,\mathrm{full}}_q$, and the LLM accuracy $\mathrm{acc}_q = 1 - \tfrac{1}{n_q}\sum_i |Y_{qi} - Y_{qi}^{\mathrm{LLM}}|$. Rows are sorted by $\overline{\widehat{A}^{\,\mathrm{full}}_q}$ ascending (easiest to rectify at the top). The final column gives each task's rank by $\overline{\mathrm{acc}}_q$ (1 = highest).
\end{minipage}
\end{table}

\paragraph{Why \textsc{PPI++} helps.}
Standard PPI ($\lambda = 1$) reduces variance relative to the sample mean if and only if $\Cov(Y_q, Y_q^{\mathrm{LLM}}) > \tfrac{1}{2}\Var(Y_q^{\mathrm{LLM}})$ (from the calculation in Section~\ref{sec:ppi_pp}). In our data, this threshold is violated for 58 of 68 questions ($85.3\%$): 37 questions show strictly higher standard-PPI variance than the sample-mean variance and 21 are essentially unchanged, because LLM predictions often exhibit low respondent-level variance and even smaller covariance with human responses. Averaging across the 68 questions, the sample-mean variance is $\tfrac{1}{68}\sum_q \widehat{\Var}(Y_q) = 0.110$, while the standard-PPI variance is $\tfrac{1}{68}\sum_q \widehat{\Var}(Y_q - Y_q^{\mathrm{LLM}}) = 0.123$: naively incorporating LLM predictions \emph{inflates} variance on average. \textsc{PPI++} fixes this by adapting the tuning parameter to each question. Under $\widehat{\lambda}^{\,\mathrm{full}}_q$ the resulting rectification difficulty satisfies $\widehat{A}^{\,\mathrm{full}}_q \le \widehat{\Var}(Y_q)$: setting $\lambda = 0$ already recovers the sample-mean variance, and the variance-minimizing $\widehat{\lambda}^{\,\mathrm{full}}_q$ can only improve on that benchmark. \textsc{PPI++} therefore weakly dominates the sample mean question by question.

\paragraph{Distribution of tuning parameters.}
We first examine the distribution of the adaptive tuning parameter. Across the 68 questions, $\widehat{\lambda}^{\,\mathrm{full}}_q$ averages 0.187 (median: 0.068, range: 0.000--0.716); see Figure~\ref{fig:Aq-distribution}(A). The distribution has a sizable mass at zero: 26 questions have $\widehat{\lambda}^{\,\mathrm{full}}_q = 0$, meaning that the LLM provides no useful respondent-level signal and \textsc{PPI++} reduces to the sample mean, while the remaining 42 questions have $\widehat{\lambda}^{\,\mathrm{full}}_q > 0$. The average variance reduction relative to the sample mean is only 4.1\%, with substantial heterogeneity across tasks (Web Appendix~\ref{appsec:data_summary}). This modest ceiling is driven by the weak respondent-level LLM signal: \textsc{PPI++} alone, under a fixed allocation rule, cannot dramatically outperform the sample mean.

\paragraph{Distribution of rectification difficulty.}
We next examine the distribution of the resulting rectification difficulty. The full-sample \textsc{PPI++} difficulty $\widehat{A}^{\,\mathrm{full}}_q$ ranges from 0.024 (WTA/WTP--Thaler) to 0.239 (Probability Matching), with mean 0.106 and median 0.076; see Figure~\ref{fig:Aq-distribution}(B). This heterogeneity is the fundamental source of gains from optimal allocation: the Neyman-style rule $n_q^\star \propto \sqrt{A_q}$ directs human labels toward high-difficulty questions, whereas uniform allocation spreads the budget equally across both easy and hard questions. Questions that are harder for the LLM also tend to have lower human test--retest reliability, suggesting a shared source of difficulty rooted in question ambiguity (Web Appendix~\ref{appsec:data_summary}, Figure~\ref{fig:test-retest}).

\begin{figure}[ht!]
    \centering
    \includegraphics[width=0.85\linewidth]{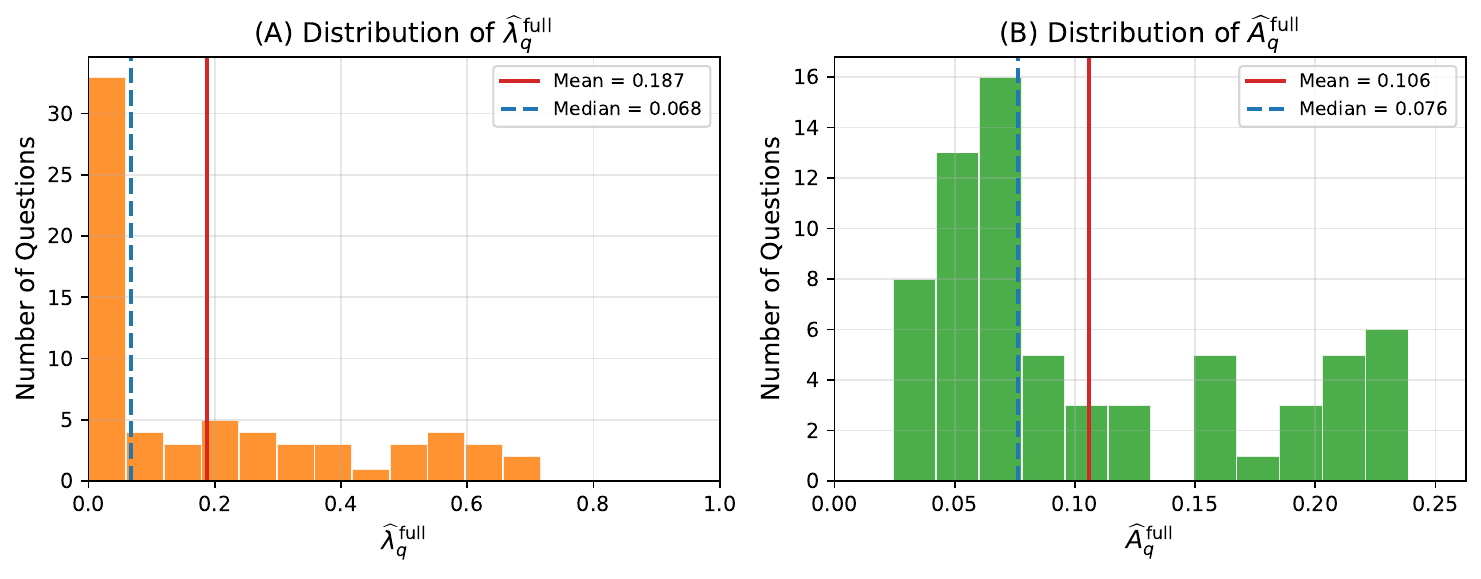}
    \caption{Distribution of full-sample (A)~tuning parameter $\widehat{\lambda}^{\,\mathrm{full}}_q$ and (B)~rectification difficulty $\widehat{A}^{\,\mathrm{full}}_q$ across the 68 Twin-2K-500 questions.}
    \label{fig:Aq-distribution}
\end{figure}

Overall, these diagnostics show substantial heterogeneity in question-level difficulty. The rectification difficulty $\widehat{A}^{\,\mathrm{full}}_q$ identifies which questions need more human data in a way that is not recoverable from accuracy alone. Naively allocating equal human effort across all questions is unlikely to be efficient, which motivates the meta-learning exercise in the next subsection.

\subsection{Meta-learning performance}
\label{subsec:meta_learning_results}

\paragraph{Evaluation protocol.}
We evaluate the meta-learning model out of sample via task-level 5-fold \emph{GroupKFold} cross-validation: the 14 tasks are partitioned into 5 groups (all questions of a task stay in the same group to avoid template leakage), and the model is refit 5 times, each time training on 4 groups and predicting $\tilde{A}_q$ for the held-out group. Pooling the five held-out predictions yields a single out-of-sample prediction vector $\{\tilde{A}_q\}_{q=1}^{68}$ covering every question in the dataset. We repeat this procedure for 20 random task-level splits and report means with $2.5$th/$97.5$th-percentile 95\% CIs across the 20 replications.


\paragraph{Model specification.}
Following Section~\ref{subsec:meta_learning}, we represent each question by embedding the concatenation of the question stem and response options using OpenAI's \texttt{text-embedding-3-large} model (3{,}072 dimensions), reduce dimensionality via PCA, and append simple question-level covariates such as the number of response options and the scale range. The feature set is constructed from the question text alone; respondent-level persona information enters only through the twin LLM responses used to compute the training targets $\widehat{A}^{\,\mathrm{full}}_q$, never through the meta-learner's features. We then fit a regularized linear regression to predict $\log(\widehat{A}^{\,\mathrm{full}}_q)$, selecting the pipeline configuration (PCA variance threshold and covariate set) via cross-validation.

\paragraph{Results.}
Table~\ref{tab:meta-learning} summarizes the cross-validated prediction quality. The text embeddings explain about 56\% of the variance in log-scale rectification difficulty ($R^2 = 0.556$ [0.524, 0.572]), with Spearman $\rho = 0.734$ [0.708, 0.754] and Pearson $r = 0.748$ [0.728, 0.758]. Thus, although the model does not perfectly recover $\widehat{A}^{\,\mathrm{full}}_q$, it captures a substantial share of the cross-question variation in rectification difficulty.

\paragraph{Implication for design.}
These results suggest that rectification difficulty is meaningfully predictable from question characteristics alone, even without any pilot human data for the target survey. This is the key requirement for our design approach: the meta-learning model need not predict $\widehat{A}^{\,\mathrm{full}}_q$ perfectly, but it must recover enough of the cross-question heterogeneity to support informative, non-uniform allocation. The next subsection evaluates whether these predicted difficulties translate into improvements in survey design relative to standard benchmarks.

\begin{table}[ht!]
\centering
\footnotesize
\caption{Out-of-sample correlation between meta-learned predictions $\log \tilde{A}_q$ and the ground-truth full-sample difficulty $\log \widehat{A}^{\,\mathrm{full}}_q$ on Twin-2K-500.}
\label{tab:meta-learning}
\begin{tabular}{lcc}
\toprule
Metric & Mean & 95\% CI \\
\midrule
Spearman $\rho$ & 0.734 & [0.708, 0.754] \\
Pearson $r$ & 0.748 & [0.728, 0.758] \\
RMSE & 0.417 & [0.409, 0.431] \\
$R^2$ & 0.556 & [0.524, 0.572] \\
\bottomrule
\end{tabular}

\medskip
\begin{minipage}{\textwidth}
\scriptsize\noindent\textit{Notes.} Means and 95\% CIs across 20 replications. Within each replication, 5-fold task-level GroupKFold yields out-of-sample $\tilde{A}_q$ for every $q$; correlations and $R^2$ are computed between the pooled $\log \tilde{A}_q$ and $\log \widehat{A}^{\,\mathrm{full}}_q$ across the 68 questions, then averaged across the 20 replications.
\end{minipage}
\end{table}

\subsection{Survey-design performance}
\label{subsec:performance_eval}

We now evaluate whether the predicted rectification difficulties $\tilde{A}_q$ translate into better estimation efficiency in our proposed survey design. Our main empirical finding is that the proposed design-based approach delivers substantially larger gains than \textsc{PPI++} alone under uniform allocation, and recovers most of the improvement attainable under the oracle allocation benchmark.

\paragraph{Benchmark methods and evaluation protocol.}
We compare four methods, each defined by an estimator and an allocation rule; they differ only in how the human-label budget is allocated and whether estimation uses the sample mean or \textsc{PPI++}. In all method labels, tables, and figures below, ``\textsc{PPI}'' denotes the \textsc{PPI++}\ estimator (Section~\ref{sec:ppi_pp}); standard PPI ($\lambda=1$) is never used in the empirical comparison. The four methods are:

\squishlist
\item \textbf{\textsc{SM + Uniform}}: the standard benchmark in survey research, which allocates the budget equally across questions (uniform allocation) and uses the Sample Mean (SM) as the estimator for each question.

\item \textbf{\textsc{PPI + Uniform}}: applies the \textsc{PPI++} estimator to data collected under uniform allocation. Because LLM responses are available at negligible marginal cost, this method improves precision over the sample mean without altering the data-collection protocol. It thus isolates the gain from improved estimation while keeping the data-collection protocol unchanged.

\item \textbf{\textsc{PPI + Opt.\ (Pred.)}}: applies \textsc{PPI++} with allocation based on the meta-learned predicted difficulties $\tilde{A}_q$. This is our proposed method.

\item \textbf{\textsc{PPI + Opt.\ (Oracle)}}: applies \textsc{PPI++} with allocation based on the population-proxy difficulties $\widehat{A}^{\,\mathrm{full}}_q$ computed from the full observed pool. This is infeasible in practice and serves as an upper bound on the attainable gain.
\squishend

For each of the 20 meta-learning replications of Section~\ref{subsec:meta_learning_results}, we take the pooled out-of-sample predictions $\{\tilde{A}_q\}_{q=1}^{68}$ and form each method's per-question allocation across all 68 questions for a given budget $B$, mapping the continuous allocation to integers via largest-remainder rounding. We then draw the allocated number of human responses per question from the observed pool and form the method's estimator. The Monte Carlo draws are repeated 200 times per replication, yielding $20 \times 200 = 4{,}000$ simulated datasets per budget, and the pipeline is run over the grid $B \in \{100, 200, \ldots, 1000, 1200, 1500, 2000\}$. Reported 95\% CIs are $2.5$th/$97.5$th percentiles across the 20 replication-level averages; they quantify pipeline variability, not the statistical uncertainty of a single deployed survey.

We evaluate each method using MSE, RMSE, and MAE over the target questions in $\mathcal{T}$, where
\[
\mathrm{MSE} = \frac{1}{|\mathcal{T}|}\sum_{q \in \mathcal{T}} (\widehat{\theta}_q-\theta_q^\star)^2,
\qquad
\mathrm{RMSE} = \sqrt{\mathrm{MSE}},
\qquad
\mathrm{MAE} = \frac{1}{|\mathcal{T}|}\sum_{q \in \mathcal{T}} \left|\widehat{\theta}_q-\theta_q^\star\right|.
\]

For each metric-method combination, we report the percentage reduction relative to the baseline \textsc{SM + Uniform}. For example, the reduction for a given method for the MSE metric is
\[
r = 1 - \frac{\mathrm{MSE}_{\mathrm{method}}}{\mathrm{MSE}_{\mathrm{SM+Uniform}}}.
\]
Because MSE scales as $1/B$ under continuous allocation, an $r$\% MSE reduction is equivalent to requiring $r$\% fewer human labels to match the precision of the baseline. We also report \emph{gain coverage}, defined as $r/r_{\textsc{Oracle}} \times 100\%$, i.e., the fraction of the oracle improvement recovered by the predicted-allocation method.

Two further columns of Table~\ref{tab:compact-reduction} translate the reductions into budget terms. \emph{Cost savings} is computed from the analytical objective: we find the budget $B_{\mathrm{equiv}}$ at which \textsc{SM + Uniform} would attain the method's weighted MSE at budget $B$, and report $1 - B/B_{\mathrm{equiv}}$, the fraction of the baseline budget that the method renders unnecessary. Because the analytical objective scales exactly as $1/B$, this quantity coincides with the analytical MSE reduction; it can differ modestly from the Monte Carlo reduction, which additionally reflects the integer rounding of allocations at small budgets. \emph{Effective size gain} expresses the Monte Carlo MSE reduction $r$ as an equivalent enlargement of the human sample: a method with reduction $r$ achieves the same MSE as the baseline would with a human-label budget larger by $r/(1-r)$ (e.g., an 11.4\% reduction corresponds to a 12.9\% larger budget).

A useful feature of these percentage reductions is that they are approximately invariant to the total budget $B$. Under any allocation rule, the analytical objective takes the form $J(\bm{A},\bm{n}) = \sum_q w_q A_q / n_q$, and because each $n_q$ is proportional to $B$, we have $J \propto 1/B$ for all methods. As a result, the ratio $J_{\mathrm{method}}/J_{\mathrm{baseline}}$, and hence the percentage reduction, is approximately constant across the budget grid. Absolute MSE still declines with budget (Web Appendix~\ref{appsec:absolute_mse}, Table~\ref{tab:absolute-mse}), but the relative gains are stable.

\paragraph{Results.}
Table~\ref{tab:compact-reduction} summarizes the results. For each method-metric combination, the reported percentage reduction is the average across Monte Carlo simulation runs, and the bracketed values are budget-averaged 95\% confidence intervals computed across replications. Two findings stand out. First, \textsc{PPI++} by itself delivers only modest gains when the budget remains uniformly allocated. \textsc{PPI + Uniform} achieves a 3.6\% reduction in MSE relative to \textsc{SM + Uniform}. This is consistent with the diagnostics in Section~\ref{subsec:ppi_diagnostics}: because the respondent-level LLM signal is weak for many questions, improved estimation alone cannot yield large gains. Second, reallocating human effort using predicted rectification difficulty produces much larger improvements. Our proposed method, \textsc{PPI + Opt.\ (Pred.)}, achieves an 11.4\% reduction in MSE, more than tripling the gain from \textsc{PPI++} under uniform allocation. Its gain coverage is 78.6\%, implying that the meta-learned allocation recovers nearly four-fifths of the improvement attainable under the oracle benchmark.

\begin{table}[ht!]
\centering
\caption{Twin-2K-500: percentage reduction in error metrics vs.\ \textsc{SM + Uniform}.}
\label{tab:compact-reduction}
{\scriptsize
\setlength{\tabcolsep}{2.5pt}
\renewcommand{\arraystretch}{1.0}
\begin{tabular}{lcccccc}
\toprule
Method & MSE Red.\ (\%) & RMSE Red.\ (\%) & MAE Red.\ (\%) & Cost Savings (\%) & \shortstack{Eff.\ Size\\ Gain (\%)} & \shortstack{Gain\\ Coverage (\%)} \\
\midrule
\textsc{PPI + Uniform} & \phantom{0}3.6\% \scriptsize{[\phantom{0}2.7, \phantom{0}4.3]} & \phantom{0}1.8\% \scriptsize{[\phantom{0}1.4, \phantom{0}2.2]} & \phantom{0}2.2\% \scriptsize{[\phantom{0}1.8, \phantom{0}2.6]} & \phantom{0}3.6\% \scriptsize{[\phantom{0}3.6, \phantom{0}3.6]} & \phantom{0}3.7\% & 24.5\% \\
\textsc{PPI + Opt.\ (Pred.)} & 11.4\% \scriptsize{[10.2, 12.2]} & \phantom{0}6.0\% \scriptsize{[\phantom{0}5.3, \phantom{0}6.4]} & \phantom{0}4.4\% \scriptsize{[\phantom{0}3.9, \phantom{0}4.8]} & \phantom{0}9.0\% \scriptsize{[\phantom{0}8.7, \phantom{0}9.1]} & 12.9\% & 78.6\% \\
\textsc{PPI + Opt.\ (Oracle)} & 14.5\% \scriptsize{[13.9, 15.1]} & \phantom{0}7.6\% \scriptsize{[\phantom{0}7.3, \phantom{0}7.9]} & \phantom{0}5.1\% \scriptsize{[\phantom{0}4.7, \phantom{0}5.4]} & 12.3\% \scriptsize{[12.3, 12.3]} & 17.0\% & 100.0\% \\
\bottomrule
\end{tabular}
}

\medskip
\begin{minipage}{\textwidth}
\scriptsize\noindent\textit{Notes.} Point estimates are Monte Carlo means. Bracketed intervals are 95\% split-replication ranges (2.5th/97.5th percentiles across the 20 meta-learning replications, i.e., pipeline variability rather than single-survey sampling uncertainty), averaged across the 13 budgets $B \in \{100, 200, \ldots, 2{,}000\}$. Gain coverage is $r / r_{\textsc{Oracle}} \times 100\%$. Because MSE scales as $1/B$ under continuous allocation, an $r\%$ MSE reduction is equivalent to $r\%$ fewer human labels to match baseline precision. Cost savings is the analytical budget reduction at which \textsc{SM + Uniform} matches the method's MSE; effective size gain, $r/(1-r)$, is the equivalent increase in human-label budget implied by an MSE reduction of $r$ (Section~\ref{subsec:performance_eval}).
\end{minipage}
\end{table}

\paragraph{Where do the gains come from?}
The total 11.4\% MSE reduction of \textsc{PPI + Opt.\ (Pred.)} relative to \textsc{SM + Uniform} can be decomposed into two components. The \emph{estimator effect}---switching from the sample mean to \textsc{PPI++} while holding allocation fixed---yields a 3.6\% MSE reduction. This gain is essentially free, because it uses the same human labels together with LLM responses that are already available at negligible marginal cost. The \emph{allocation effect}---optimizing allocation using the predicted difficulties while holding the estimator fixed---adds a further 7.8 percentage points of MSE reduction, increasing the total gain from 3.6\% to 11.4\%. Thus, the dominant source of improvement is not the \textsc{PPI++} estimator by itself, but the reallocation of human effort toward high-difficulty questions. We also benchmark against the classical Neyman design that allocates by the human-response variance $\Var(Y_q)$ without using any LLM information; our framework outperforms this benchmark as well, with the full comparison reported in Web Appendix~\ref{appsec:sm_opt_decomposition}.

\paragraph{How does the allocation change?}
Figure~\ref{fig:allocation-scatter} visualizes the predicted allocation at $B=2000$. The figure reports the \emph{continuous} allocation implied by the square-root rule before integer rounding, averaged over the 20 meta-learning replications, and plots allocated sample sizes $\tilde{n}_q^\star$ against the full-sample rectification difficulties $\widehat{A}^{\,\mathrm{full}}_q$, colored by task. Under uniform allocation, each question receives $B/|\mathcal{T}| \approx 29.4$ labels. The meta-learned allocation instead ranges from $22.2$ for the easiest questions to $35.3$ for the hardest, showing that the predicted difficulties induce a meaningful reallocation of human effort toward harder questions. The compressed range relative to the oracle (which spans roughly $14.8$ to $46.3$ labels) reflects the regularization of the meta-learner, but the Spearman rank correlation with the oracle allocation is $\rho = 0.714$, and by Proposition~\ref{prop:robust_allocation} the efficiency loss is controlled by log-ratio misspecification, consistent with the $78.6\%$ gain coverage reported in Table~\ref{tab:compact-reduction}.

\begin{figure}[ht!]
    \centering
    \includegraphics[width=\textwidth]{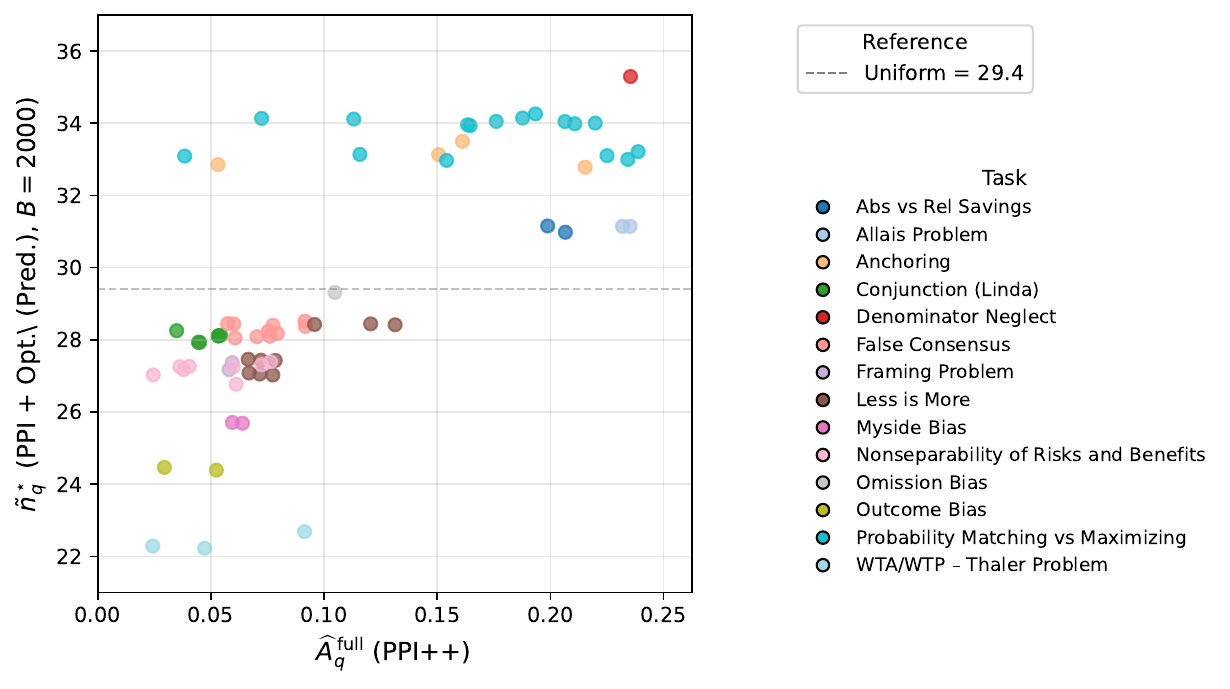}
    \caption{Per-question predicted allocation at $B = 2000$ vs.\ full-sample rectification difficulty $\widehat{A}^{\,\mathrm{full}}_q$, color-coded by task. Dashed line is the uniform allocation $B/|\mathcal{T}| \approx 29.4$.}
    \label{fig:allocation-scatter}
\end{figure}

\subsection{Robustness checks}
\label{subsec:robustness_eval}

We assess robustness along five dimensions: the cross-wave stability of rectification difficulty within Twin-2K-500, cross-dataset transfer of the meta-learner to a genuinely new survey environment, the value of zero-shot meta-learning relative to a target-survey pilot, aggregation from question-level to task-level allocation, and the sensitivity of gains to greater heterogeneity in question difficulty. The first three stress the inputs of the framework (where the difficulty estimates come from); the last two stress the design environment in which the allocation operates.

\squishlist

\item \textbf{Cross-wave stability of $A_q$.}
We compute the full-sample rectification difficulties from Waves~1--3 and use them for allocation while evaluating performance on Wave~4. The cross-wave and within-wave difficulties correlate at $\rho = 0.99$, and the resulting allocation yields a 14.1\% MSE reduction, virtually identical to the 14.5\% of the within-wave oracle: rectification difficulty is a stable property of questions, so practitioners can estimate it reliably from prior waves or other historical data.

\item \textbf{Cross-dataset transfer.}
To assess whether the meta-learning generalizes to a completely different dataset, we train the meta-learner on a political survey dataset CCES (see Section~\ref{sec:cces}), and predict rectification difficulty $A_q$ on Twin-2K-500. The cross-dataset prediction achieves Spearman $\rho = 0.52$ and yields a 9.6\% MSE reduction, 68\% of the oracle, indicating that which questions are hard to rectify is driven largely by question characteristics that generalize across survey contexts. Details can be found in Web Appendix~\ref{appsec:cross_dataset_transfer}.

\item \textbf{Zero-shot meta-learning vs.\ a pilot baseline.}
A natural alternative to estimating $A_q$ is to run a pilot on the target survey: spend part of the budget collecting a few responses per question, estimate $A_q$ directly, and allocate the remainder by the square-root rule (the pilot responses are reused in the final estimate). In the same Monte Carlo pipeline, the pilot's benefit grows with its size but stays below meta-learning. At $B=2000$, a $10$-response pilot (a third of the budget) raises MSE by $4.6\%$ rather than reducing it (it does worse than \textsc{SM + Uniform} because of noisy estimation of $A_q$), while $15$- and $20$-response pilots (half and two-thirds of the budget) reduce MSE by $6.6\%$ and $8.7\%$, all short of the $11.4\%$ reduction from zero-shot meta-learning. A useful pilot must therefore spend most of the budget. A pilot is also infeasible at smaller budgets (it requires $Q\,n_{\text{pilot}} < B$, e.g.\ $B>680$ for $10$ responses per question). Zero-shot meta-learning instead spends none of the budget on a pilot, applies at any budget, and recovers $78.6\%$ of the oracle gain.

\item \textbf{Task-level allocation.}
Many surveys allocate respondents at the module or task level, with each respondent completing all questions within an assigned block (matrix sampling; \citealt{raghunathan1995split, adiguzel2008split}; cf.\ the wave-level extension in Section~\ref{subsec:extensions}). We aggregate the 68 questions into $T=14$ tasks. Because a respondent assigned to task $t$ answers all $|t|$ of its questions, the difficulty driving the allocation is the weighted \emph{sum} $\sum_{q\in t} w_q A_q$ (equivalently $w_t A_t$, where $A_t$ is the within-task average and $w_t \propto |t|$), and we set $n_t \propto \sqrt{w_t A_t}$. Aggregation strengthens prediction: question-level errors partially cancel within tasks, so the Spearman correlation between predicted and true difficulty rises from $0.73$ to $0.89$, and the meta-learned allocation recovers $95.4\%$ of the task-level oracle gain (a consequence of greater between-task than within-task heterogeneity in difficulty). Under a respondent budget (matrix sampling, $B$ respondents each answering their whole task), \textsc{PPI + Opt.\ (Pred.)} attains a $24.6\%$ MSE reduction at $B=2000$; because each respondent then contributes $|t|\approx 5$ responses, this is not directly comparable to the $11.4\%$ obtained at the question level under a per-response budget (Table~\ref{tab:compact-reduction}), and at a fixed total number of responses the two allocation gains are comparable. The framework thus applies, with no loss of efficiency, at the operationally natural module granularity at which surveys are typically fielded.

\item \textbf{Sensitivity to heterogeneity in question difficulty.}
We amplify the dispersion of $A_q$ via a mean-preserving log-scale transform, $\log A_q^{(\alpha)} = \bar{\mu} + \alpha(\log A_q - \bar{\mu})$, where $\alpha \ge 1$ controls the spread. At $B=2000$, the MC MSE reduction for \textsc{PPI + Opt.\ (Pred.)} grows from 10.6\% at $\alpha = 1$ (baseline) to 20.9\% at $\alpha = 2$ and 34.1\% at $\alpha = 5$, while capturing roughly 72--81\% of the oracle gains even though the predictions $\tilde{A}_q$ are held fixed at their $\alpha = 1$ values (see Web Appendix~\ref{appsec:robustness_table}, Table~\ref{tab:robustness} and Figure~\ref{fig:robustness}). Surveys with greater heterogeneity in LLM reliability, such as those mixing factual items with subjective judgments, therefore benefit substantially more from optimal allocation.

\squishend

\paragraph{Discussion and Summary.}


Taken together, the Twin-2K-500 results yield three main conclusions. First, rectification difficulty varies substantially across questions and is not captured by raw LLM accuracy. Second, this variation is sufficiently predictable from question features to support sample allocation: our proposed method reduces MSE by 11.4\% relative to \textsc{SM + Uniform} and recovers 78.6\% of the oracle gain, with most of the improvement coming from reallocating human labels rather than from \textsc{PPI++} alone. Third, the gains are robust to using difficulty estimates from prior waves, transferring the meta-learner across datasets, and allocating at the task level, and they increase as heterogeneity in question difficulty grows. Section~\ref{sec:cces} next examines whether these findings generalize to a different survey domain and a different response-generating LLM.

\section{Empirical Analysis on CCES 2024}
\label{sec:cces}

In this section, we assess whether our framework generalizes beyond the behavioral-economics setting of Twin-2K-500 by applying it to the 2024 Cooperative Election Study \citep[CCES;][]{cces2024}, a large-scale political survey. Relative to Twin-2K-500, the CCES setting differs in three important ways. First, the domain is political rather than behavioral-economic. Second, question difficulty is substantially more heterogeneous, ranging from routine factual items (e.g., voter registration status) to highly polarized policy preferences (e.g., Ukraine and Gaza). Third, the LLM used to generate the surrogate survey responses is Google's Gemini~2.5~Flash rather than GPT-4o, while the meta-learning step continues to use OpenAI embeddings to construct question features.

We preview two main findings from this analysis. First, question difficulty is more heterogeneous than in Twin-2K-500, so the headroom for optimal allocation is larger: the oracle gain is 17.1\% versus 14.5\%. Second, despite a weaker LLM signal and noisier difficulty predictions, our method still delivers a comparable realized gain (10.5\% vs.\ 11.4\%), and the gain flows almost entirely through the allocation channel rather than the estimator (Section~\ref{subsec:cces_performance_eval}): a low-signal environment is precisely where careful survey design creates the value. Because the cross-dataset robustness check in Section~\ref{subsec:robustness_eval} already uses CCES, we do not repeat a separate robustness subsection here.

\subsection{Dataset and setup}
\label{subsec:cces_data_setup}

\paragraph{Data source.}
We use the CCES 2024 \citep{cces2024}, a nationally representative survey administered to over 60{,}000 US adults. We focus on questions amenable to our framework: 133 subquestions nested within 40 parent questions, with approximately 5{,}000 respondents per question. The 40 parent questions span five question types: multiple-choice (17), binary (4), grid-agree (4), grid-support (9), and multi-item (6). Topics range from economic perceptions and party identification to foreign policy (Ukraine, Gaza), gun regulation, immigration, abortion, the environment, race and gender attitudes, and vote choice. As in Twin-2K-500, all responses are rescaled to $[0,1]$ via min--max normalization using each question's scale endpoints.

\paragraph{LLM predictions and ground-truth population.}
For each respondent--question pair, we generate LLM predictions using Google's Gemini~2.5~Flash with persona-conditioned prompting. Each respondent's demographic profile (age, gender, race, education, income, state, and party registration) is encoded as text, following the same prompt template as in Twin-2K-500. The full CCES sample serves as the ground-truth population, from which we compute the full-sample plug-ins $\widehat{\lambda}^{\,\mathrm{full}}_q$ and $\widehat{A}^{\,\mathrm{full}}_q$ exactly as in Section~\ref{subsec:ppi_diagnostics}.

\paragraph{Question weights.}
Unlike Twin-2K-500, the CCES contains parent questions with varying numbers of sub-items. To prevent parent questions with many subquestions from dominating the overall objective, we assign equal total importance to each parent question. Specifically, for subquestion $q$, we set
\[
w_q = \frac{1}{Q_{\mathrm{parents}} \times n_{\mathrm{subs},q}},
\]
where $Q_{\mathrm{parents}}$ is the number of parent questions and $n_{\mathrm{subs},q}$ is the number of subquestions under parent question $q$. Thus, a parent question with many sub-items receives the same aggregate weight as a parent question with only one sub-item. These weights enter both the design objective and the evaluation: all reported metrics (MSE, RMSE, MAE, and cost savings) aggregate per-question errors with the same weights, e.g., $\mathrm{MSE} = \sum_{q} w_q (\widehat{\theta}_q - \theta_q^\star)^2$, which reduces to the unweighted average of Section~\ref{subsec:performance_eval} when $w_q = 1/Q$.

\paragraph{Simulation protocol.}
The simulation protocol follows Section~\ref{subsec:performance_eval} with three changes: GroupKFold uses parent questions as groups, so that all subquestions from the same parent are held out together; results are reported at a single budget, $B = 5{,}000$ (about 38 responses per subquestion under uniform allocation), since percentage reductions are approximately budget-invariant; and 95\% CIs are computed across 50 replications of the pipeline, each with a fresh GroupKFold split and 2{,}000 Monte Carlo draws.

\subsection{PPI++ diagnostics and rectification difficulty}
\label{subsec:cces_ppi_diagnostics}

LLM predictions achieve a mean normalized agreement of 61.1\% across the 133 CCES questions, substantially below the 75.7\% observed in Twin-2K-500. The \textsc{PPI++}\ dominance property $\widehat{A}^{\,\mathrm{full}}_q \le \widehat{\Var}(Y_q)$ continues to hold question by question, so \textsc{PPI++}\ estimator weakly dominates the sample mean.

\paragraph{Illustrative examples.}
The resulting heterogeneity in LLM usefulness is substantial. On factual items such as ``Are you registered to vote?'', the LLM achieves 95.0\% accuracy and $\widehat{\lambda}^{\,\mathrm{full}}_q = 0.49$, reflecting genuine individual-level predictive power. In contrast, on politically sensitive topics the LLM tends to produce near-constant predictions across respondents. For instance, on Gaza conflict sub-items (``What should the U.S.\ do?'') and the 2024 presidential vote choice, the LLM outputs are nearly identical for every respondent regardless of demographics, yielding $\widehat{\Var}(Y_q^{\mathrm{LLM}}) \approx 0$, $\widehat{\lambda}^{\,\mathrm{full}}_q = 0$, and no variance reduction. \textsc{PPI++}\ automatically detects these cases and falls back to the sample mean.

\paragraph{Distribution of tuning parameters.}
The distribution of $\widehat{\lambda}^{\,\mathrm{full}}_q$ (Figure~\ref{fig:cces-Aq-distribution}A) averages 0.132 (vs.\ 0.187 in Twin-2K-500), with median 0.087. Forty of 133 questions (30\%) have $\widehat{\lambda}^{\,\mathrm{full}}_q = 0$---the LLM provides no useful individual-level signal on these items, and \textsc{PPI++}\ reduces to the sample mean. The zero-$\lambda$ questions concentrate on politically polarized topics: Ukraine policy items (5 of 8 sub-items), Gaza conflict items (7 of 9), and the 2024 presidential vote choice. One possible explanation is that safety-alignment practices in frontier LLMs may discourage the model from expressing partisan preferences \citep{bakker2022fine, ouyang2022instruct}. 

\paragraph{Distribution of rectification difficulty.}
The resulting $\widehat{A}^{\,\mathrm{full}}_q$ distribution (Figure~\ref{fig:cces-Aq-distribution}B) has a higher mean than in Twin-2K-500 ($0.146$ vs.\ $0.106$) and is more dispersed on the log scale that governs the multiplicative allocation rule: the standard deviation of $\log\widehat{A}^{\,\mathrm{full}}_q$ is $0.76$, versus $0.63$ in Twin-2K-500. (On the raw scale, the coefficient of variation is similar across the two datasets, $0.53$ vs.\ $0.63$; it is the log-scale spread, which the square-root rule and the meta-learner operate on, that drives the larger oracle gain.) Questions with $\widehat{\lambda}^{\,\mathrm{full}}_q > 0$ exhibit reduced $\widehat{A}^{\,\mathrm{full}}_q$ (the LLM absorbs some variance), while the 30\% of questions with $\widehat{\lambda}^{\,\mathrm{full}}_q = 0$ retain their full population variance $\widehat{\Var}(Y_q)$.

\begin{figure}[ht!]
    \centering
    \includegraphics[width=0.85\linewidth]{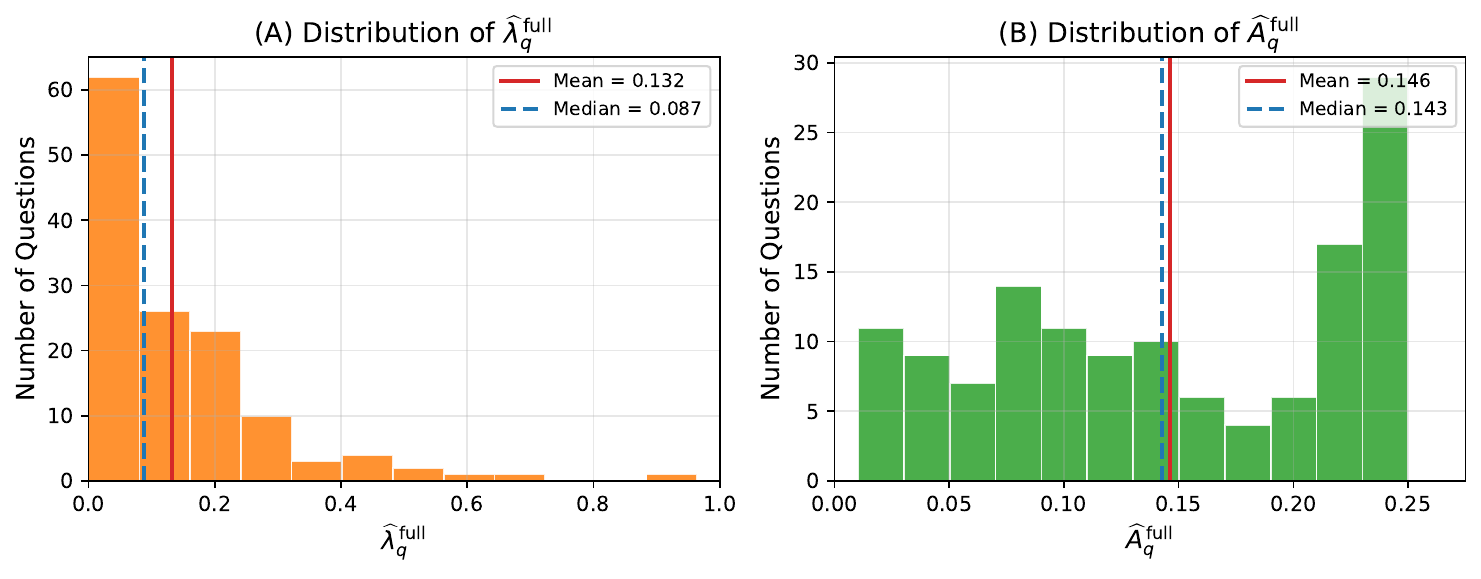}
    \caption{Distribution of full-sample (A)~tuning parameter $\widehat{\lambda}^{\,\mathrm{full}}_q$ and (B)~rectification difficulty $\widehat{A}^{\,\mathrm{full}}_q$ across the 133 CCES 2024 subquestions.}
    \label{fig:cces-Aq-distribution}
\end{figure}

\subsection{Meta-learning performance}
\label{subsec:cces_meta_learning}


We apply the same meta-learning pipeline as in Section~\ref{subsec:meta_learning_results}, using 5-fold \emph{GroupKFold} with parent questions as groups and repeating the entire cross-validation procedure 50 times with fresh random fold assignments, as described in Section~\ref{subsec:cces_data_setup}. Table~\ref{tab:cces-meta-learning} summarizes the cross-validated prediction quality. The text embeddings explain about 17\% of the variance in log-scale rectification difficulty ($R^2 = 0.165$ [0.113, 0.220]), with Spearman $\rho = 0.476$ [0.422, 0.542] and Pearson $r = 0.413$ [0.357, 0.470]. These are weaker than Twin-2K-500 ($\rho = 0.734$, $R^2 = 0.556$). The gap reflects the greater topical diversity of the CCES (40 heterogeneous parent questions versus 14 behavioral-economics tasks) and greater variation in response format (2 to 41 levels), which make it harder for text embeddings alone to capture difficulty structure. Nevertheless, the allocation framework requires only approximate rankings of $\sqrt{A_q}$, and Proposition~\ref{prop:robust_allocation} bounds the efficiency loss from moderate prediction errors.

\begin{table}[ht!]
\centering
\footnotesize
\caption{Out-of-sample correlation between meta-learned predictions $\log \tilde{A}_q$ and the ground-truth full-sample difficulty $\log \widehat{A}^{\,\mathrm{full}}_q$ on CCES 2024.}
\label{tab:cces-meta-learning}
\begin{tabular}{lcc}
\toprule
Metric & Mean & 95\% CI \\
\midrule
Spearman $\rho$ & 0.476 & [0.422, 0.542] \\
Pearson $r$ & 0.413 & [0.357, 0.470] \\
RMSE & 0.694 & [0.670, 0.715] \\
$R^2$ & 0.165 & [0.113, 0.220] \\
\bottomrule
\end{tabular}

\medskip
\begin{minipage}{\textwidth}
\scriptsize\noindent\textit{Notes.} Means and 95\% CIs across 50 replications. Within each replication, 5-fold parent-question-level GroupKFold yields out-of-sample $\tilde{A}_q$ for every $q$; correlations and $R^2$ are computed between the pooled $\log \tilde{A}_q$ and $\log \widehat{A}^{\,\mathrm{full}}_q$ across the 133 subquestions, then averaged across the 50 replications.
\end{minipage}
\end{table}

\subsection{Survey-design performance}
\label{subsec:cces_performance_eval}

\paragraph{Results.}
Table~\ref{tab:cces-compact-reduction} reports the same methods and metrics as Table~\ref{tab:compact-reduction}, computed with the parent-level weights, as described in Section~\ref{subsec:cces_data_setup}, at the single budget $B = 5{,}000$.
\textsc{PPI + Uniform} achieves a modest 1.5\% MSE reduction [0.7, 2.2], consistent with the weak LLM signal across most CCES questions.

In sharp contrast, \textsc{PPI + Opt.\ (Oracle)} achieves a 17.1\% MSE reduction [16.3, 17.9], \emph{larger} than the 14.5\% in Twin-2K-500. This larger oracle gain reflects the greater heterogeneity in $A_q$ across CCES questions: high-difficulty policy items coexist with easy factual items, and the optimal allocation concentrates labels where they matter most. This is exactly the regime predicted by the robustness analysis in Section~\ref{subsec:robustness_eval}, where amplifying $A_q$ dispersion ($\alpha > 1$) leads to larger allocation gains.

Our proposed method, \textsc{PPI + Opt.\ (Pred.)}, delivers a 10.5\% MSE reduction [9.5, 11.4], capturing 61.3\% of the oracle gains. This is comparable to the 11.4\% MSE reduction in Twin-2K-500, even though the meta-learning prediction is substantially weaker: the larger oracle headroom in the CCES compensates for the lower gain coverage (61.3\% vs.\ 78.6\%). Equivalently, our framework matches traditional survey precision with roughly 10\% fewer human labels.

\begin{table}[ht!]
\centering
\caption{CCES 2024: percentage reduction in error metrics vs.\ \textsc{SM + Uniform}.}
\label{tab:cces-compact-reduction}
{\scriptsize
\setlength{\tabcolsep}{2.5pt}
\renewcommand{\arraystretch}{1.0}
\begin{tabular}{lcccccc}
\toprule
Method & MSE Red.\ (\%) & RMSE Red.\ (\%) & MAE Red.\ (\%) & Cost Savings (\%) & \shortstack{Eff.\ Size\\ Gain (\%)} & \shortstack{Gain\\ Coverage (\%)} \\
\midrule
\textsc{PPI + Uniform} & \phantom{0}1.5\% \scriptsize{[\phantom{0}0.7, \phantom{0}2.2]} & \phantom{0}0.8\% \scriptsize{[\phantom{0}0.3, \phantom{0}1.1]} & \phantom{0}0.7\% \scriptsize{[\phantom{0}0.3, \phantom{0}1.1]} & \phantom{0}1.6\% \scriptsize{[\phantom{0}1.6, \phantom{0}1.6]} & \phantom{0}1.5\% & \phantom{0}8.8\% \\
\textsc{PPI + Opt.\ (Pred.)} & 10.5\% \scriptsize{[\phantom{0}9.5, 11.4]} & \phantom{0}5.4\% \scriptsize{[\phantom{0}4.8, \phantom{0}5.9]} & \phantom{0}7.7\% \scriptsize{[\phantom{0}7.2, \phantom{0}8.3]} & 11.5\% \scriptsize{[10.7, 12.2]} & 11.7\% & 61.3\% \\
\textsc{PPI + Opt.\ (Oracle)} & 17.1\% \scriptsize{[16.3, 17.9]} & \phantom{0}9.0\% \scriptsize{[\phantom{0}8.5, \phantom{0}9.4]} & \phantom{0}9.1\% \scriptsize{[\phantom{0}8.6, \phantom{0}9.5]} & 18.1\% \scriptsize{[18.1, 18.1]} & 20.7\% & 100.0\% \\
\bottomrule
\end{tabular}
}

\medskip
\begin{minipage}{\textwidth}
\scriptsize\noindent\textit{Notes.} Point estimates are Monte Carlo means. Bracketed intervals are 95\% split-replication ranges (2.5th/97.5th percentiles across 50 meta-learning replications, i.e., pipeline variability rather than single-survey sampling uncertainty) at $B = 5{,}000$ (with 2{,}000 MC draws each). Gain coverage is $r / r_{\textsc{Oracle}} \times 100\%$. Because MSE scales as $1/B$ under continuous allocation, an $r\%$ MSE reduction is equivalent to $r\%$ fewer human labels to match baseline precision. Cost savings is the analytical budget reduction at which \textsc{SM + Uniform} matches the method's MSE; effective size gain, $r/(1-r)$, is the equivalent increase in human-label budget implied by an MSE reduction of $r$.
\end{minipage}
\end{table}


\paragraph{Where do the gains come from?}
The total 10.5\% MSE reduction of \textsc{PPI + Opt.\ (Pred.)} relative to \textsc{SM + Uniform} can be decomposed into two channels. The \emph{estimator effect}---switching from the sample mean to \textsc{PPI++}\ while holding allocation fixed---yields a modest 1.5\% MSE reduction. The \emph{allocation effect}---optimizing allocation while holding the estimator fixed---adds 9.0 percentage points (from 1.5\% to 10.5\%), accounting for 86\% of the total gain.

This decomposition is particularly instructive when compared to Twin-2K-500, where the estimator effect was 3.6\% and the allocation effect was 7.8 percentage points. In the CCES, allocation accounts for an even larger share of the total gain because the estimator effect is weaker. This contrast between the two datasets reinforces the paper's central message: the value of our framework lies primarily in optimal allocation, not in \textsc{PPI++}\ alone.

\paragraph{How does the allocation change?}
To illustrate, consider the allocations at $B = 5{,}000$. Under uniform allocation, each subquestion receives $B/Q \approx 37.6$ labels. The oracle allocation assigns as few as 6.3 labels to the easiest sub-items and up to 105.4 to the hardest. The meta-learned allocation exhibits the same qualitative pattern, ranging from 21.8 to 80.0 labels, with a Spearman rank correlation of $\rho = 0.839$ against the oracle allocation, confirming that even imperfect meta-learning produces a reasonable ranking of subquestion difficulty. This allocation rank correlation ($0.839$) is higher than the difficulty-prediction correlation ($\rho = 0.476$; Table~\ref{tab:cces-meta-learning}) because the predicted and oracle allocations are both proportional to $\sqrt{w_q A_q}$ and therefore share the same parent-level weights $w_q$, whose dispersion dominates the allocation ranking; the allocation agreement thus overstates the accuracy of the underlying difficulty predictions. Figure~\ref{fig:cces-allocation-scatter} visualizes this allocation.

The allocation pattern is substantively intuitive: factual and near-consensus items receive the fewest labels because the LLM already provides a reliable signal there, while contested policy attitudes receive the most; the framework automatically directs human effort toward the questions where it is most needed.

Taken together, the CCES validation demonstrates that our framework generalizes across domains, datasets, and LLMs. The weak LLM signal on political surveys makes the case for optimal allocation \emph{stronger}: because \textsc{PPI++} alone delivers only modest gains, the only path to meaningful improvement is through smarter allocation of human effort.

\begin{figure}[ht!]
    \centering
    \includegraphics[width=0.8\textwidth]{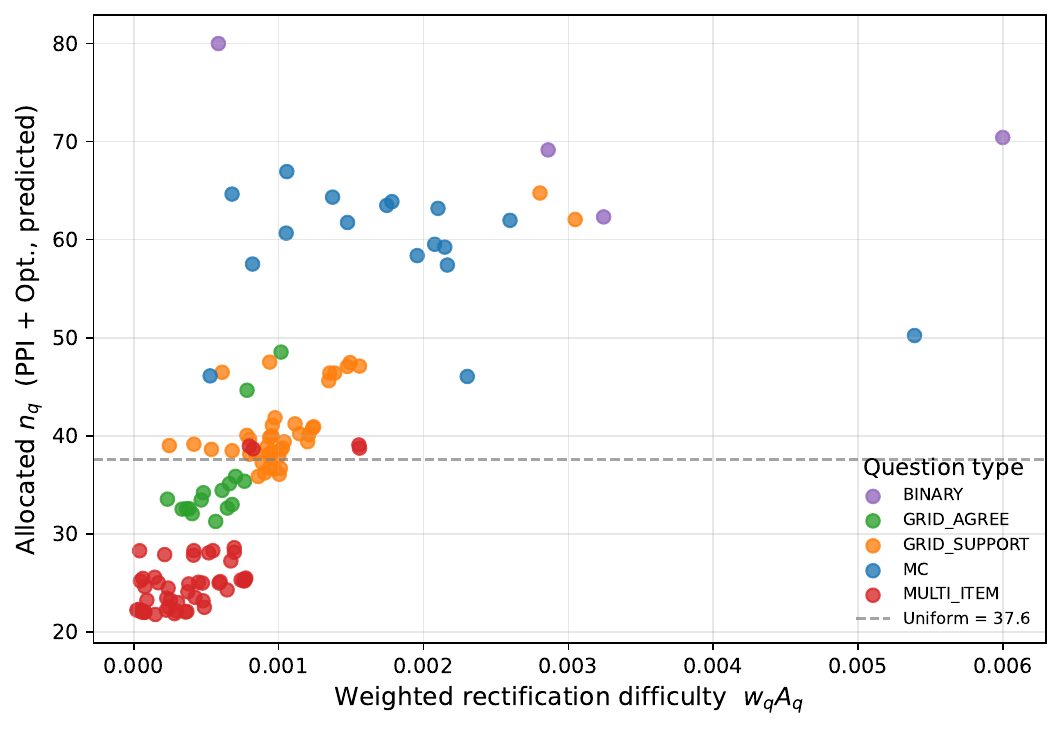}
    \caption{Per-subquestion predicted allocation at $B = 5{,}000$ vs.\ the \emph{weighted} rectification difficulty $w_q\widehat{A}^{\,\mathrm{full}}_q$ for CCES 2024, color-coded by question type; the dashed line is the uniform allocation $B/Q \approx 37.6$. Because CCES uses non-uniform parent-level importance weights $w_q$, the allocation tracks $w_q\widehat{A}^{\,\mathrm{full}}_q$ rather than $\widehat{A}^{\,\mathrm{full}}_q$ alone; for Twin-2K-500 (Figure~\ref{fig:allocation-scatter}) $w_q$ is constant, so this reduces to the difficulty axis there.}
    \label{fig:cces-allocation-scatter}
\end{figure}

\section{Conclusion}
\label{sec:conclusion}

We develop an end-to-end framework for designing LLM-augmented surveys when human responses are costly and LLM predictions are cheap but unevenly informative. The central design primitive is rectification difficulty \(A_q\), the residual uncertainty that remains after the LLM prediction is incorporated optimally through \textsc{PPI++}. This quantity links estimation to design: because the estimator's variance scales as \(A_q/n_q\), the optimal allocation follows a Neyman-type square-root rule, \(n_q^\star \propto \sqrt{w_qA_q/c_q}\). To make this rule operational before a new survey is fielded, we use historical paired human--LLM data and question embeddings to predict rectification difficulty for unseen questions, enabling zero-shot allocation without a target-survey pilot.

The framework is designed to preserve a clear separation between survey design and statistical inference. The meta-learner determines where human responses are collected, but the final \textsc{PPI++} estimates and standard errors are constructed from the newly collected target-survey data. Consequently, errors in predicted difficulty may reduce efficiency and widen confidence intervals, but they do not compromise asymptotic validity under the stated regularity conditions. We also show that the associated efficiency loss is bounded and depends on relative prediction errors across questions.

Our empirical results show that the primary value of this framework comes from allocating human responses more effectively, rather than from changing the estimator alone. Across the Twin-2K-500 behavioral-economics survey and the CCES 2024 political survey, \textsc{PPI++} under uniform allocation yields only modest MSE reductions of 1.5--3.6\%. Combining \textsc{PPI++} with the meta-learned allocation increases these reductions to 10.5--11.4\%, recovering 61--79\% of the gains attainable under infeasible oracle allocation. The gains are largest when rectification difficulty is heterogeneous across questions, precisely the settings in which uniform allocation is most wasteful. The analysis also shows why raw LLM accuracy is an inadequate design metric: an LLM may appear accurate in aggregate while contributing little respondent-level signal for reducing estimation variance.

For practitioners, the framework provides a concrete workflow: estimate rectification difficulty from historical paired data, predict it for the questions in a new survey, allocate the human-response budget using the square-root rule, and then combine the collected human responses with LLM predictions through \textsc{PPI++}. The same logic extends beyond question means to regression coefficients, categorical response probabilities, and conjoint partworths, and can be applied at the level of questions, modules, or survey waves.

\singlespacing
\putbib[mybib]
\end{bibunit}

\clearpage
\pagenumbering{roman}
\singlespacing

\begin{bibunit}
\begin{appendices}

\makeatletter
\renewcommand{\theHsection}{app.\thesection}
\renewcommand{\theHsubsection}{app.\thesubsection}
\setcounter{equation}{0}
\renewcommand{\theequation}{\thesection.\arabic{equation}}
\renewcommand{\theHequation}{app.\theequation}
\setcounter{figure}{0}
\renewcommand{\thefigure}{\thesection.\arabic{figure}}
\renewcommand{\theHfigure}{app.\thefigure}
\setcounter{table}{0}
\renewcommand{\thetable}{\thesection.\arabic{table}}
\renewcommand{\theHtable}{app.\thetable}
\setcounter{proposition}{0}
\renewcommand{\theproposition}{\thesection.\arabic{proposition}}
\renewcommand{\theHproposition}{app.\theproposition}
\setcounter{example}{0}
\renewcommand{\theexample}{\thesection.\arabic{example}}
\renewcommand{\theHexample}{app.\theexample}
\makeatother

\section{Design-Problem Extensions}
\label{app:extensions}

This appendix collects the four design-problem extensions from Direction~I of Section~\ref{subsec:extensions}: cost-minimizing dual, alternative loss functions, power analysis, and wave-level (matrix-sampling) allocation. All retain the \textsc{PPI++}\ mean estimator, so the inference guarantees of Proposition~\ref{thm:end_to_end} carry over unchanged. Throughout, $A_q$ denotes the population rectification difficulty defined in Section~\ref{ssec:rect_diff}, $\tilde{A}_q$ a design-stage prediction (e.g., from meta-learning), $w_q \ge 0$ an importance weight, $c_q > 0$ a per-response cost, $B$ the total human-data budget, and $\mathcal{T}$ the target survey.

\subsection{Cost-minimizing dual formulation and the single-question case}
\label{rmk:dual_and_single_q}

The main text considers the budget-constrained design problem~\eqref{eq:opt_problem}. We now derive the equivalent cost-minimizing dual and show that a single-question special case recovers a classical \textsc{PPI++}\ sample-size formula.

\paragraph{Dual cost-minimizing problem.}
Given a target weighted MSE $J_{\mathrm{target}} > 0$, the cost-minimizing problem is
\[
\min_{\{n_q\}_{q\in\mathcal{T}}} \quad \sum_{q\in\mathcal{T}} c_q n_q
\qquad \text{s.t.} \qquad \sum_{q\in\mathcal{T}} \frac{w_q A_q}{n_q} \le J_{\mathrm{target}},
\quad n_q \ge 0.
\]
At the optimum the constraint binds (otherwise we could uniformly shrink each $n_q$ and lower cost). Introducing a Lagrange multiplier $\nu > 0$ for the binding MSE constraint, the Lagrangian
\[
\mathcal{L}(\bm{n}, \nu) \;=\; \sum_{q} c_q n_q + \nu \Bigl( \sum_q \tfrac{w_q A_q}{n_q} - J_{\mathrm{target}} \Bigr)
\]
has stationarity condition $c_q - \nu \, w_q A_q / n_q^2 = 0$, giving
\begin{equation}\label{eq:cost_min_solution}
n_q^\star = \sqrt{\frac{\nu w_q A_q}{c_q}} \;\propto\; \sqrt{\frac{w_q A_q}{c_q}}.
\end{equation}
The proportions match those of the budget-constrained primal (Theorem~\ref{thm:optimal_allocation}); only the overall scale is set by $\nu$ rather than $B$. Substituting~\eqref{eq:cost_min_solution} back into the binding MSE constraint gives $\sum_q w_q A_q / n_q^\star = \nu^{-1/2} \sum_q \sqrt{w_q A_q c_q} = J_{\mathrm{target}}$, which solves to $\sqrt{\nu} = J_{\mathrm{target}}^{-1} \sum_q \sqrt{w_q A_q c_q}$, and therefore the minimum total cost is
\[
B_{\min} \;=\; \sum_q c_q n_q^\star \;=\; \sqrt{\nu} \sum_q \sqrt{w_q A_q c_q} \;=\; \frac{1}{J_{\mathrm{target}}} \Bigl( \sum_{q \in \mathcal{T}} \sqrt{w_q A_q c_q} \Bigr)^2.
\]
This is the dual analogue of the optimal-MSE expression $J^\star = \tfrac{1}{B} \bigl(\sum_q \sqrt{w_q A_q c_q}\bigr)^2$ derived in the proof of Theorem~\ref{thm:optimal_allocation}.

\paragraph{Single-question special case.}
When $|\mathcal{T}| = 1$ with unit cost $c = 1$ and unit weight $w = 1$, both formulations reduce to the standard \textsc{PPI++}\ sample-size determination. To achieve target MSE $J_{\mathrm{target}}$ on a single estimand, the required sample size is $n^\star = A / J_{\mathrm{target}}$, with the population rectification difficulty $A$ replacing the classical variance $\Var(Y)$. This recovers the textbook formula for fixed-precision sample-size choice but with the appropriate hybrid human--LLM variance index.

\subsection{Alternative loss functions: the cube-root rule under MAE}
\label{rmk:alternative_loss}

The square-root allocation rule depends on the choice of loss. We derive the analogous rule under a weighted \emph{mean-absolute-error} (MAE) objective and discuss when the resulting cube-root rule may be preferable.

\paragraph{Derivation.}
Asymptotic normality of the \textsc{PPI++}\ estimator (Proposition~\ref{thm:end_to_end}), specialized to the SDR limit $m_q/n_q \to \infty$, gives $\sqrt{n_q}\bigl(\widehat{\theta}_q(\widehat{\lambda}_q) - \theta_q^\star\bigr) \Rightarrow \mathcal{N}(0, A_q)$, so
\[
\bE\bigl[|\widehat{\theta}_q(\widehat{\lambda}_q) - \theta_q^\star|\bigr] \;=\; \sqrt{\tfrac{2 A_q}{\pi n_q}} \cdot (1 + o(1))
\;\propto\; \sqrt{\tfrac{A_q}{n_q}}.
\]
The weighted MAE design problem is then
\[
\min_{\{n_q\}_{q \in \mathcal{T}}} \quad \sum_{q \in \mathcal{T}} w_q \sqrt{\tfrac{A_q}{n_q}}
\qquad \text{s.t.} \qquad \sum_{q \in \mathcal{T}} c_q n_q \le B.
\]
Lagrangian stationarity gives $-\tfrac{1}{2} w_q \sqrt{A_q} \, n_q^{-3/2} + \mu c_q = 0$, i.e.\
\begin{equation}
n_q^\star \;=\; \Bigl(\tfrac{w_q \sqrt{A_q}}{2 \mu c_q}\Bigr)^{2/3} \;\propto\; \Bigl(\tfrac{w_q}{c_q}\Bigr)^{2/3} A_q^{1/3}.
\end{equation}
Substituting into the binding budget constraint $\sum_q c_q n_q^\star = B$ pins down $\mu$ and yields the closed-form normalization analogous to Theorem~\ref{thm:optimal_allocation}.

\paragraph{Practical implication.}
The cube-root dependence on $A_q$ is less aggressive than the square-root rule under MSE: the MAE-optimal allocation puts more weight on easier questions and less on the hardest ones, relative to the square-root rule. This is desirable when difficulty predictions $\tilde{A}_q$ are noisy, because the cube-root rule is more robust to over-amplification of large $\tilde{A}_q$ outliers. Other losses admit the same treatment via the asymptotic normality of $\widehat{\theta}_q$, with the allocation exponent determined by how the expected loss scales with $n_q$: quantile (check) losses scale as $n_q^{-1/2}$ like MAE and yield the same cube-root rule, whereas a Huber loss with a fixed threshold is locally quadratic, scales as $1/n_q$, and recovers the square-root rule.

\subsection{Power analysis for hypothesis testing}
\label{rmk:power_analysis}

The variance scaling $\Var(\widehat{\theta}_q) = A_q / n_q$ supports \textsc{PPI++}-based power analysis for testing question-specific null hypotheses such as $H_{0,q}: \theta_q^\star = \theta_{q,0}$.

\paragraph{Single-test calculation.}
To detect a deviation $|\theta_q^\star - \theta_{q,0}| \ge \delta_q$ at significance level $\alpha$ with power $1 - \beta$, a standard Wald-test calculation under asymptotic normality yields the required sample size
\begin{equation}\label{eq:power_sample_size}
n_q \;=\; \frac{A_q \, (z_{1-\alpha/2} + z_{1-\beta})^2}{\delta_q^2},
\end{equation}
where $z_p$ is the standard-normal $p$-quantile. The classical human-only power calculation replaces $A_q$ with $\Var(Y_q)$. Because \textsc{PPI++} guarantees $A_q \le \Var(Y_q)$, the \textsc{PPI++}-based design always requires weakly fewer human labels for the same power.

\paragraph{Multiple-testing correction.}
When all $|\mathcal{T}|$ tests are run simultaneously, family-wise error-rate control requires per-test levels $\{\alpha_q\}_{q \in \mathcal{T}}$ with $\sum_q \alpha_q \le \alpha$ (Bonferroni) or $1 - \prod_q (1 - \alpha_q) \le \alpha$ (\v{S}id\'ak under independence). Substituting $\alpha_q$ into Equation~\eqref{eq:power_sample_size} gives test-specific requirements $n_q(\alpha_q)$; combining with the budget constraint $\sum_q c_q n_q \le B$ then becomes a constrained joint design problem. Unlike the MSE problem, this generally has no closed-form solution and is solved numerically, but the per-question structure is unchanged.

\subsection{Wave-level (block-level) allocation for matrix-sampling designs}
\label{rmk:matrix_sampling}

Our main framework allocates per question, which is appropriate when questions can be fielded to disjoint respondent groups. In many practical surveys, however, respondents answer multiple questions in a single session (\emph{matrix sampling}), so within-respondent responses are correlated. We extend the framework to the natural \emph{wave-level} (or \emph{block-level}) granularity in which each \emph{wave} $w$ is a block of questions $\mathcal{W}_w \subseteq \mathcal{T}$ fielded jointly to a shared pool of $n_w$ respondents.

\paragraph{Stacked sandwich covariance.}
Let $\bm{\theta}_w = (\theta_q^\star)_{q \in \mathcal{W}_w}$ be the wave-level parameter. Stacking the per-question \textsc{PPI++}\ estimating equations for $q \in \mathcal{W}_w$ into a joint system, the $M$-estimation machinery in Web Appendix~\ref{appsec:M-estimation} gives the wave-level sandwich covariance
\[
\bm{\Sigma}_w(\lambda) \;=\; \bm{H}_w^{-1} \, \bm{V}_w^{\Delta(\lambda)} \, \bm{H}_w^{-\top},
\qquad
\bm{V}_w^{\Delta(\lambda)} \;=\; \Var\bigl((\bm{\Delta}_{q,w})_{q \in \mathcal{W}_w}\bigr),
\]
where $\bm{\Delta}_{q,w}$ is the rectified score residual from Web Appendix~\ref{appsec:M_estimation:ppi}. Crucially, $\bm{V}_w^{\Delta(\lambda)}$ contains the within-respondent cross-question covariances; ignoring these would understate the per-respondent variance under matrix sampling.

\paragraph{Scalarization and allocation.}
Scalarizing $\bm{\Sigma}_w(\lambda)$ via the trace (A-optimality, $A^{(A)}_w = \mathrm{tr}(\bm{\Omega}_w \bm{\Sigma}_w(\lambda))$) or determinant (D-optimality, $A^{(D)}_w = \det(\bm{\Sigma}_w(\lambda))^{1/d_w}$) following Web Appendix~\ref{appsec:M_estimation:scalarize} yields a wave-level difficulty index $A_w$. The wave-level allocation problem,
\[
\min_{\{n_w\}_w} \sum_w \frac{w_w A_w}{n_w}
\qquad \text{s.t.} \qquad \sum_w c_w n_w \le B,
\]
is structurally identical to Theorem~\ref{thm:optimal_allocation}, so the optimal allocation retains the same square-root form
\[
n_w^\star \;\propto\; \sqrt{\frac{w_w A_w}{c_w}}.
\]

\paragraph{Empirical validation.}
Section~\ref{subsec:robustness_eval} validates this extension empirically on Twin-2K-500: each task in the dataset is a block of questions sharing a respondent pool (a wave in the present sense), and the allocation gains carry over essentially unchanged when the unit of allocation is the task rather than the individual question.

\section{Extension to General $M$-Estimation}
\label{appsec:M-estimation}

To streamline notation, throughout this appendix, we fix an arbitrary question $q$ and drop the subscript $q$.
Concretely, we write $\cP$ for $\cP_q$, $\bm{\theta}^\star$ for $\bm{\theta}_q^\star$, $d$ for $d_q$, $n$ for $n_q$, $m$ for $m_q$, and similarly drop $q$ on all derived quantities (e.g., $\bm{H}$ for $\bm{H}_q$, $\bm{V}^{\Delta(\lambda)}$ for $\bm{V}_q^{\Delta(\lambda)}$, and $\bm{\Sigma}(\lambda)$ for $\bm{\Sigma}_q(\lambda)$).
All expectations and variances are taken with respect to $(\bm{X},Y)\sim \cP$ unless stated otherwise.

\subsection{Population $M$-estimation and estimating equations}
\label{appsec:M_estimation:setup}

In the main text, we focused on scalar population means to build intuition. Here we describe a general $M$-estimation framework that covers many estimands used in marketing applications.

Let $\ell(\bm{X},Y;\bm{\theta})$ be a loss function that is convex and twice differentiable
in $\bm{\theta}\in \mathbb{R}^d$, and suppose the population risk has a unique minimizer:
\begin{align}
\label{eqn:pop_risk_def}
\bm{\theta}^\star
\;=\;
\argmin_{\bm{\theta} \in \mathbb{R}^d}\;
\bE\!\left[\ell(\bm{X},Y;\bm{\theta})\right].
\end{align}
Define the \emph{score} (estimating) function
\[
\bm{\psi}(\bm{X},Y;\bm{\theta})
\;:=\;
\nabla_{\bm{\theta}}\ell(\bm{X},Y;\bm{\theta})
\;\in\; \mathbb{R}^d.
\]
Under standard regularity conditions that permit differentiating under the expectation,
the first-order optimality condition for~\eqref{eqn:pop_risk_def} is the population moment equation
\begin{align}
\label{eqn:pop_score_moment}
\bE\!\left[\bm{\psi}(\bm{X},Y;\bm{\theta}^\star)\right] \;=\; \bm{0}.
\end{align}
We also define the population Jacobian/Hessian matrix
\begin{align*}
\bm{H}
\;:=\;
\bE\!\left[\nabla_{\bm{\theta}}\bm{\psi}(\bm{X},Y;\bm{\theta}^\star)\right]
\;=\;
\bE\!\left[\nabla_{\bm{\theta}}^2 \ell(\bm{X},Y;\bm{\theta}^\star)\right]
\;\in\; \mathbb{R}^{d\times d}.
\end{align*}
For identification and asymptotic normality we assume $\bm{H}$ is nonsingular (and, under convexity, typically
positive definite).

\subsection{Examples in marketing applications}
\label{appsec:M_estimation:examples}

We illustrate the notation through four common examples relevant to survey design and demand estimation.

\begin{example}[Mean estimation]
\label{ex:Mest_mean}
Mean estimation is the special case where $\bm{\theta}=\theta\in \mathbb{R}$, the loss is
$\ell(Y;\theta)=\frac{1}{2}(Y-\theta)^2$, and the score is $\psi(Y;\theta)=\theta-Y$.
The first-order condition $\bE[\psi(Y;\theta^\star)]=0$ yields $\theta^\star=\bE[Y]$, recovering the mean estimand
studied in the main text.
\end{example}

\begin{example}[Categorical response estimation]
\label{ex:Mest_categorical}
Estimating the distribution of categorical responses is fundamental in survey-based market research, for example
when the outcome is a $K$-way multiple-choice question (market share, brand preference, top choice, etc.).
Let $Y\in\{1,\dots,K\}$ denote the chosen category and consider the probability vector
$\bm{\theta}=(\Pr(Y=1),\dots,\Pr(Y=K))^\top \in \mathbb{R}^K$.

One convenient $M$-estimation formulation uses a squared-error loss on the indicator vector:
\[
\ell(Y;\bm{\theta})
\;=\;
\sum_{k=1}^K \frac{1}{2}\Big(\bI\{Y=k\}-\theta_k\Big)^2.
\]
The score is
\[
\bm{\psi}(Y;\bm{\theta})
\;=\;
\nabla_{\bm{\theta}}\ell(Y;\bm{\theta})
\;=\;
\bm{\theta}
-
\big(\bI\{Y=1\},\dots,\bI\{Y=K\}\big)^\top.
\]
The population moment condition $\bE[\bm{\psi}(Y;\bm{\theta}^\star)]=\bm{0}$ implies $\theta_k^\star=\Pr(Y=k)$ for each
$k$, i.e., it recovers the categorical response probabilities.
(If desired, one can additionally enforce the simplex constraints $\theta_k\ge 0$ and $\sum_k \theta_k=1$; this does
not change the basic PPI rectification idea below.)
\end{example}

\begin{example}[Linear regression]
\label{ex:Mest_ols}
A classic use case is estimating price--demand relationships, where $Y$ is demand and $\bm{X}\in \mathbb{R}^d$ encodes
price and other product attributes. Linear regression corresponds to the loss
\[
\ell(\bm{X},Y;\bm{\theta})
\;=\;
\frac{1}{2}\big(Y-\bm{X}^\top\bm{\theta}\big)^2,
\qquad
\bm{\psi}(\bm{X},Y;\bm{\theta})
\;=\;
\bm{X}\big(\bm{X}^\top\bm{\theta}-Y\big).
\]
The moment condition $\bE[\bm{X}(\bm{X}^\top\bm{\theta}^\star-Y)]=\bm{0}$ yields the population normal equations
and, when $\bE[\bm{X}\bm{X}^\top]$ is invertible,
\[
\bm{\theta}^\star \;=\; \bE[\bm{X}\bm{X}^\top]^{-1}\bE[\bm{X}Y],
\]
the population ordinary least squares solution.
\end{example}

\begin{example}[Multinomial logit (MNL) choice model]
\label{ex:Mest_mnl}
In conjoint analysis and discrete choice modeling, a respondent chooses among $K$ product profiles with feature
vectors $\bm{X}_1,\dots,\bm{X}_K$ (e.g., attribute encodings). Let $Y\in\{1,\dots,K\}$ be the chosen option and
$\bm{\theta}\in \mathbb{R}^d$ the partworth vector. The MNL negative log-likelihood is
\begin{align*}
\ell(\bm{X},Y;\bm{\theta})
&=
\log\!\Big(\sum_{k=1}^K \exp(\bm{X}_k^\top \bm{\theta})\Big)
-
\sum_{k=1}^K (\bm{X}_k^\top \bm{\theta})\,\bI\{Y=k\}.
\end{align*}
Define the model-implied choice probabilities
\[
P_k(\bm{\theta})
\;:=\;
\frac{\exp(\bm{X}_k^\top \bm{\theta})}{\sum_{\kappa=1}^K \exp(\bm{X}_\kappa^\top \bm{\theta})},
\qquad k=1,\dots,K.
\]
Then the score takes the familiar form
\[
\bm{\psi}(\bm{X},Y;\bm{\theta})
\;=\;
\sum_{k=1}^K \big(P_k(\bm{\theta})-\bI\{Y=k\}\big)\bm{X}_k.
\]
The population condition $\bE[\bm{\psi}(\bm{X},Y;\bm{\theta}^\star)]=\bm{0}$ defines $\bm{\theta}^\star$; in practice it
is typically solved numerically (e.g., via Newton or quasi-Newton methods).
\end{example}

\subsection{Prediction-powered $M$-estimator and asymptotic variance}
\label{appsec:M_estimation:ppi}

We now describe the prediction-powered construction for general $M$-estimation, starting with \textsc{PPI++} (our default) and viewing PPI as a special case.
Suppose we have a surrogate predictor (e.g., an LLM) that, given $\bm{X}$, produces a pseudo-outcome
$Y^{\mathrm{LLM}}=f(\bm{X})$. A natural way to build a surrogate score is
\[
\bm{\psi}^{\mathrm{LLM}}(\bm{X};\bm{\theta})
\;:=\;
\bm{\psi}(\bm{X}, f(\bm{X}); \bm{\theta}),
\]
i.e., we plug the pseudo-outcome into the original score.
(We write $\bm{\psi}^{\mathrm{LLM}}(\bm{X};\bm{\theta})$ to emphasize it is computable without observing $Y$.)

\noindent\textbf{PPI++ rectified estimating equation.}
Let $\{(\bm{X}_i,Y_i)\}_{i=1}^n$ be the labeled sample used for rectification, and let the synthetic pool be sufficiently large that we can treat $\bE[\bm{\psi}^{\mathrm{LLM}}(\bm{X};\bm{\theta})]$ as known (SDR asymptotics).
For tuning $\lambda\in[0,1]$, the \textsc{PPI++} $M$-estimator $\widehat{\bm{\theta}}(\lambda)$ is defined as a solution to
\begin{equation}
\label{eqn:Mestimator_score_pp}
\bm{0}
\;=\;
\frac{1}{n}\sum_{i=1}^{n}\bm{\psi}(\bm{X}_i, Y_i; \bm{\theta})
\;+\;
\lambda\Bigg(
\bE\!\left[\bm{\psi}^{\mathrm{LLM}}(\bm{X}; \bm{\theta})\right]
-
\frac{1}{n}\sum_{i=1}^{n}\bm{\psi}^{\mathrm{LLM}}(\bm{X}_i; \bm{\theta})
\Bigg).
\end{equation}
Setting $\lambda=0$ yields the classical $M$-estimator; setting $\lambda=1$ yields PPI.

\noindent\textbf{Special cases.}
When $\lambda=1$, the \textsc{PPI++} estimating equation~\eqref{eqn:Mestimator_score_pp} can be rewritten as the PPI estimating equation
\begin{align*}
\bm{0}
\;=\;
\bE\!\left[\bm{\psi}^{\mathrm{LLM}}(\bm{X}; \bm{\theta})\right]
\;+\;
\frac{1}{n}\sum_{i=1}^{n}\Big(\bm{\psi}(\bm{X}_i, Y_i; \bm{\theta})-\bm{\psi}^{\mathrm{LLM}}(\bm{X}_i; \bm{\theta})\Big).
\end{align*}
Setting $\lambda=0$ recovers the classical $M$-estimator.

\noindent\textbf{Consistency.}
Evaluating the left-hand side of Equation~\eqref{eqn:Mestimator_score_pp} at $\bm{\theta}^\star$ and taking expectations gives
\[
\bE\!\left[\bm{\psi}(\bm{X},Y;\bm{\theta}^\star)\right]
+
\lambda\Big(\bE\!\left[\bm{\psi}^{\mathrm{LLM}}(\bm{X}; \bm{\theta}^\star)\right]-\bE\!\left[\bm{\psi}^{\mathrm{LLM}}(\bm{X}; \bm{\theta}^\star)\right]\Big)
=
\bE\!\left[\bm{\psi}(\bm{X},Y;\bm{\theta}^\star)\right]
=
\bm{0},
\]
where the final equality is exactly the population first-order condition~\eqref{eqn:pop_score_moment}.
Thus, for any fixed $\lambda\in[0,1]$, the rectified estimating equation targets the same population parameter $\bm{\theta}^\star$ even if the surrogate predictor is misspecified. Under standard $M$-estimation regularity conditions, this implies $\widehat{\bm{\theta}}(\lambda)\xrightarrow{p}\bm{\theta}^\star$ as $n\to\infty$ (e.g., Theorem~2.7 of \citealt{newey1994large}).

\noindent\textbf{Asymptotic normality and sandwich covariance (SDR).}
Linearizing Equation~\eqref{eqn:Mestimator_score_pp} around $\bm{\theta}^\star$ yields the usual $M$-estimation influence-function representation. In the SDR regime (treating $\bE[\bm{\psi}^{\mathrm{LLM}}(\bm{X};\bm{\theta})]$ as a population expectation), the only sampling variability comes from the rectified score residual
\[
\bm{\Delta}(\lambda;\bm{X},Y;\bm{\theta}^\star)
\;:=\;
\bm{\psi}(\bm{X},Y;\bm{\theta}^\star)-\lambda\,\bm{\psi}^{\mathrm{LLM}}(\bm{X};\bm{\theta}^\star).
\]
Define its covariance matrix
\begin{align*}
\bm{V}^{\Delta(\lambda)}
\;:=\;
\Var\!\Big(\bm{\Delta}(\lambda;\bm{X},Y;\bm{\theta}^\star)\Big)
\;=\;
\Var\!\Big(\bm{\psi}(\bm{X},Y;\bm{\theta}^\star)-\lambda\,\bm{\psi}^{\mathrm{LLM}}(\bm{X};\bm{\theta}^\star)\Big).
\end{align*}
Then, under standard regularity conditions,
\[
\sqrt{n}\big(\widehat{\bm{\theta}}(\lambda)-\bm{\theta}^\star\big)
\;\xrightarrow{d}\;
\mathcal{N}\!\left(\bm{0},\,\bm{\Sigma}(\lambda)\right),
\qquad
\bm{\Sigma}(\lambda)
=
\bm{H}^{-1}\bm{V}^{\Delta(\lambda)}\bm{H}^{-\top},
\]
and therefore
\begin{align}
\label{eqn:sandwich_variance}
\Var\!\big(\widehat{\bm{\theta}}(\lambda)\big)
\;\approx\;
\frac{1}{n}\,\bm{\Sigma}(\lambda).
\end{align}
The form~\eqref{eqn:sandwich_variance} is a direct analogue of the scalar case in the main text: the labeled sample size $n$ appears only through the $1/n$ factor, and all task/question dependence is summarized by the matrix $\bm{\Sigma}(\lambda)$. When $\lambda=1$, this reduces to the PPI covariance (Theorem~1 of \citealt{angelopoulos2023ppi}).

\subsection{Scalarizing the covariance matrix for design}
\label{appsec:M_estimation:scalarize}

For multi-dimensional $M$-estimation, the asymptotic uncertainty is a covariance matrix
$\Var(\widehat{\bm{\theta}}(\lambda))\approx \bm{\Sigma}(\lambda)/n$ as in Equation~\eqref{eqn:sandwich_variance}.
To apply the allocation framework in the main text (which requires a scalar ``difficulty index'' per question), we must map this matrix into a scalar design objective. If the estimand is a differentiable scalar functional $g:\mathbb{R}^d\to \mathbb{R}$, one can directly apply the Delta method. 

When the goal is accurate estimation of the entire parameter vector (e.g., all partworths), one typically scalarizes the covariance matrix using a classical optimal design criterion. Below, we provide two common strategies that preserve the same $1/n$ scaling and are therefore directly compatible with our survey sample allocation design. It includes minimizing the trace and determinant of the covariance matrix, which correspond to the A-optimal and D-optimal objective, respectively.
There are more objectives in the optimal experimental design literature; see \citet{pukelsheim2006optimal, atkinson2007optimum, fedorov2013theory, silvey2013optimal}.

\subsubsection{Trace.}

A natural vector loss is the squared error, possibly after applying a linear transformation
$\bm{\eta}=\bm{L}\bm{\theta}$ (e.g., standardizing coefficients or focusing on a subset).
Since $\Var(\bm{L}\widehat{\bm{\theta}})=\bm{L}\Var(\widehat{\bm{\theta}})\bm{L}^\top$ and the squared bias is $o(1/n)$, we have
\[
\bE\!\left[\|\bm{L}\widehat{\bm{\theta}}(\lambda)-\bm{L}\bm{\theta}^\star\|_2^2\right]
\approx
\mathrm{tr}\!\Big(\Var(\bm{L}\widehat{\bm{\theta}}(\lambda))\Big)
\approx
\frac{1}{n}\,\mathrm{tr}\!\Big(\bm{L}\bm{\Sigma}(\lambda)\bm{L}^\top\Big).
\]
Using cyclicity of the trace, $\mathrm{tr}(\bm{L}\bm{\Sigma}(\lambda)\bm{L}^\top)
=\mathrm{tr}(\bm{\Omega}\bm{\Sigma}(\lambda))$ with $\bm{\Omega}:=\bm{L}^\top\bm{L}\succeq \bm{0}$.
Thus a weighted trace objective can be written as
\[
\mathrm{tr}\!\Big(\bm{\Omega}\Var(\widehat{\bm{\theta}}(\lambda))\Big)
\approx
\frac{1}{n}\,\mathrm{tr}\!\Big(\bm{\Omega}\bm{\Sigma}(\lambda)\Big).
\]
It penalizes the \emph{sum} of (appropriately weighted) variances
and is often preferred when one cares about average accuracy across coefficients, or when coefficients have different
scales and a meaningful weighting $\bm{\Omega}$ is available. We define the corresponding difficulty index as
\begin{equation}
\label{eqn:difficulty_A}
A^{(A)}(\lambda)
\;:=\;
\mathrm{tr}\!\Big(\bm{\Omega}\bm{\Sigma}(\lambda)\Big)
=
\mathrm{tr}\!\Big(\bm{\Omega}\bm{H}^{-1}\bm{V}^{\Delta(\lambda)}\bm{H}^{-\top}\Big),
\end{equation}
so that the A-optimal objective scales as $A^{(A)}(\lambda)/n$.

\subsubsection{Determinant.}
A complementary objective is based on the volume of the asymptotic confidence ellipsoid for $\bm{\theta}^\star$, which is proportional to $\det(\Var(\widehat{\bm{\theta}}(\lambda)))^{1/2}$.
Directly minimizing $\det(\Var(\widehat{\bm{\theta}}(\lambda)))$ is a common strategy.

To retain the linear $1/n$ scaling used in our allocation framework, we use the $d$-th root of the determinant:
\[
\det\!\Big(\Var(\widehat{\bm{\theta}}(\lambda))\Big)^{1/d}
\approx
\det\!\Big(\tfrac{1}{n}\bm{\Sigma}(\lambda)\Big)^{1/d}
=
\frac{1}{n}\,\det\!\Big(\bm{\Sigma}(\lambda)\Big)^{1/d}.
\]
This yields the scalar difficulty index
\begin{equation}
\label{eqn:difficulty_D}
A^{(D)}(\lambda)
\;:=\;
\det\!\Big(\bm{\Sigma}(\lambda)\Big)^{1/d}
=
\left[\det\!\Big(\bm{H}^{-1}\bm{V}^{\Delta(\lambda)}\bm{H}^{-\top}\Big)\right]^{1/d}.
\end{equation}
Compared to the trace criterion, the determinant depends on the product of eigenvalues (a volume measure) rather than their sum, and ranks allocations consistently under invertible linear reparameterizations of $\bm{\theta}$ (e.g., rescaling units), since such transformations change the criterion only by a multiplicative constant; this can be desirable when coefficient scales are arbitrary.

Both Equations~\eqref{eqn:difficulty_A} and~\eqref{eqn:difficulty_D} lead to objectives proportional to $A(\lambda)/n$, matching the structure used in the main text. Consequently, the same allocation analysis applies after scalarizing $\bm{\Sigma}(\lambda)$ and (optionally) tuning $\lambda$ to minimize the chosen scalar objective.

\section{Omitted Proofs}
\label{app:optimization_proofs}

\subsection{Proof of Lemma~\ref{lem:mse_var}}
\label{app:proof_mse_var}

Taking expectations on both sides of Equation~\eqref{eq:ppi++} yields
\[
\bE[\widehat{\theta}(\lambda)] = \bE[Y] = \theta^\star,
\]
so the MSE equals the variance for fixed $\lambda$. Here, $\bE[Y^{\mathrm{LLM}}] = \bE[\tilde{Y}^{\mathrm{LLM}}]$ because both equal $\bE[f(\bm{X})]$ with $\bm{X}\sim \cP(\bm{X})$.

Subtracting $\theta^\star$ from both sides,
\begin{align*}
\widehat{\theta}(\lambda)-\theta^\star
&=
\frac{1}{n}\sum_{i=1}^{n}\Big(Y_i-\lambda Y_i^{\mathrm{LLM}}\Big)
+
\frac{\lambda}{m}\sum_{i=1}^{m}\tilde{Y}_i^{\mathrm{LLM}}
-\theta^\star \\
&=
\frac{1}{n}\sum_{i=1}^{n}\Big(Y_i-\lambda Y_i^{\mathrm{LLM}}\Big)
+
\frac{\lambda}{m}\sum_{i=1}^{m}\tilde{Y}_i^{\mathrm{LLM}}
-\bE[Y-\lambda Y^{\mathrm{LLM}}]
-\lambda \bE[\tilde{Y}^{\mathrm{LLM}}] \\
&=
\frac{1}{n}\sum_{i=1}^{n}\Big(Y_i-\lambda Y_i^{\mathrm{LLM}}-\bE[Y-\lambda Y^{\mathrm{LLM}}]\Big)
+
\frac{\lambda}{m}\sum_{i=1}^{m}\Big(\tilde{Y}_i^{\mathrm{LLM}}-\bE[\tilde{Y}^{\mathrm{LLM}}]\Big).
\end{align*}

Define
\[
U
:=
\frac{1}{n}\sum_{i=1}^{n}\Big(Y_i-\lambda Y_i^{\mathrm{LLM}}-\bE[Y-\lambda Y^{\mathrm{LLM}}]\Big),
\qquad
V
:=
\frac{\lambda}{m}\sum_{i=1}^{m}\Big(\tilde{Y}_i^{\mathrm{LLM}}-\bE[\tilde{Y}^{\mathrm{LLM}}]\Big),
\]
so that
\[
\widehat{\theta}(\lambda)-\theta^\star = U+V.
\]
Each term is centered, so $\bE[U]=\bE[V]=0$. Moreover, $U$ depends only on the labeled sample, since $Y_i^{\mathrm{LLM}}=f(\bm{X}_i)$ is a function of the labeled prompt $\bm{X}_i$ (and, if $f$ is stochastic, an independent randomization seed), while $V$ depends only on the synthetic sample. Because the two samples, including any LLM randomization, are drawn independently, $(U,V)$ are independent and $\Cov(U,V)=0$. Therefore,
\[
\Var\!\left(\widehat{\theta}(\lambda)\right)
=
\Var(U+V)
=
\Var(U)+\Var(V).
\]

Because the summands in $U$ are i.i.d. with variance $\Var(Y-\lambda Y^{\mathrm{LLM}})$, we have
\[
\Var(U)=\frac{1}{n}\Var(Y-\lambda Y^{\mathrm{LLM}}).
\]
Similarly,
\[
\Var(V)=\frac{\lambda^2}{m}\Var(\tilde{Y}^{\mathrm{LLM}}).
\]
Thus,
\[
\Var\!\left(\widehat{\theta}(\lambda)\right)
=
\frac{1}{n}\Var\!\left(Y-\lambda Y^{\mathrm{LLM}}\right)
+
\frac{\lambda^2}{m}\Var\!\left(\tilde{Y}^{\mathrm{LLM}}\right).
\]
Under SDR, this simplifies to
\[
\Var\!\left(\widehat{\theta}(\lambda)\right)
=
\frac{1}{n}\Var\!\left(Y-\lambda Y^{\mathrm{LLM}}\right).
\]
\hfill\Halmos

\subsection{Proof of Theorem~\ref{thm:optimal_allocation}}
\label{ssec:proof_thm_opt_all}

We explicitly derive the optimal allocation rule by analyzing the properties of the objective function and solving the associated Lagrangian.

First, we establish the strict convexity of the optimization problem, which guarantees the uniqueness of the solution. Let $f_q(n_q) = w_q A_q / n_q$. For any $q$ where $w_q > 0$, the derivatives with respect to $n_q > 0$ are
\[
    f_q'(n_q) = -\frac{w_q A_q}{n_q^2}
    \quad \text{and} \quad
    f_q''(n_q) = \frac{2 w_q A_q}{n_q^3} > 0.
\]
Since $f''_q(n_q) > 0$, each term is strictly convex in $n_q$ for $q$ with $w_q>0$. For questions with $w_q=0$, the objective does not depend on $n_q$, and because $c_q>0$ any positive allocation to such questions can only tighten the budget constraint and weakly worsen the achievable objective value. Hence, at the unique optimum we have $n_q^\star=0$ for all $q$ with $w_q=0$, and the remaining subproblem over $\{q:w_q>0\}$ is strictly convex. For any $q$ with $w_q > 0$, $f_q(n_q) = w_q A_q/n_q \to \infty$ as $n_q \to 0^+$, so the optimum must be interior. As the budget constraint $\sum_{q\in\mathcal{T}} c_q n_q \le B$ defines a convex feasible set, the Karush--Kuhn--Tucker (KKT) conditions are both necessary and sufficient for optimality, and the solution is unique.

To find this solution, we form the Lagrangian function with multiplier $\mu$:
\[
    \mathcal{L}(\bm{n},\mu) = \sum_{q\in\mathcal{T}} \frac{w_q A_q}{n_q} + \mu\left(\sum_{q\in\mathcal{T}} c_q n_q - B\right).
\]
The first-order stationarity condition $\frac{\partial \mathcal{L}}{\partial n_q} = 0$ yields
\[
    -\frac{w_q A_q}{n_q^2} + \mu c_q = 0 \implies n_q = \sqrt{\frac{w_q A_q}{\mu c_q}}.
\]
Because the objective function is strictly decreasing in $n_q$, the budget constraint must bind at optimality (i.e., $\sum c_q n_q = B$). Substituting the expression for $n_q$ into the active budget constraint allows us to solve for the Lagrange multiplier $\mu$:
\[
    \sum_{q\in\mathcal{T}} c_q \sqrt{\frac{w_q A_q}{\mu c_q}} = B
    \implies
    \frac{1}{\sqrt{\mu}} \sum_{q\in\mathcal{T}} \sqrt{w_q A_q c_q} = B.
\]
Finally, substituting the derived value of $1/\sqrt{\mu}$ back into the expression for $n_q$ yields the closed-form optimal allocation rule stated in Equation~\eqref{eq:closed_form_sol}. \hfill\Halmos

\subsection{Proof of Proposition~\ref{thm:end_to_end}}
\label{app:end_to_end_proof}

Fix an arbitrary question $q \in \mathcal{T}$ and suppress the subscript $q$ for notational brevity. Write $\mu_{\mathrm{LLM}} := \bE[Y^{\mathrm{LLM}}]$ and $\bar{\tilde Y}^{\mathrm{LLM}} := m^{-1}\sum_{j=1}^m \tilde Y_j^{\mathrm{LLM}}$ for the synthetic-pool average used in Step~6. The labeled sample $\{(Y_i, Y_i^{\mathrm{LLM}})\}_{i=1}^{\tilde n}$ and the synthetic pool $\{\tilde Y_j^{\mathrm{LLM}}\}_{j=1}^m$ are independent.

\noindent\textbf{Preliminaries.}
The \textsc{PPI++} estimator for the mean with tuning parameter $\lambda$ is
\[
\widehat{\theta}(\lambda) \;=\; \frac{1}{\tilde{n}}\sum_{i=1}^{\tilde{n}} (Y_i - \lambda\,Y_i^{\mathrm{LLM}}) + \lambda\,\bar{\tilde Y}^{\mathrm{LLM}}.
\]
The pairs $(Y_i, Y_i^{\mathrm{LLM}})$ are i.i.d.\ draws from the population distribution.

We proceed in two steps: first establish the CLT for a fixed $\lambda^\star$, then extend to the plug-in $\widehat{\lambda}$.

\emph{Step 1 (fixed $\lambda^\star$).}
Using $\theta^\star = \bE Y = \bE(Y - \lambda^\star Y^{\mathrm{LLM}}) + \lambda^\star \mu_{\mathrm{LLM}}$, decompose
\[
\widehat{\theta}(\lambda^\star) - \theta^\star
\;=\;
\underbrace{\frac{1}{\tilde{n}}\sum_{i=1}^{\tilde n}\Bigl[(Y_i - \lambda^\star Y_i^{\mathrm{LLM}}) - \bE(Y - \lambda^\star Y^{\mathrm{LLM}})\Bigr]}_{T_1}
\;+\;
\lambda^\star \underbrace{\bigl(\bar{\tilde Y}^{\mathrm{LLM}} - \mu_{\mathrm{LLM}}\bigr)}_{T_2}.
\]
$T_1$ is an average of i.i.d.\ mean-zero terms with variance $A := \Var(Y - \lambda^\star Y^{\mathrm{LLM}})$, and $T_2$ is an average of i.i.d.\ mean-zero terms with variance $\Var(Y^{\mathrm{LLM}})$; $T_1$ and $T_2$ are independent by the independence of the labeled sample and the synthetic pool. The CLT applied to each term, together with Slutsky's theorem, yields
\begin{equation}
\label{eq:clt_fixed_lambda}
\sigma^{-1}\bigl(\widehat{\theta}(\lambda^\star) - \theta^\star\bigr) \xrightarrow{d} \mathcal{N}(0, 1),
\qquad
\sigma^2 \;=\; \frac{A}{\tilde n} + \frac{(\lambda^\star)^2 \Var(Y^{\mathrm{LLM}})}{m}.
\end{equation}

\emph{Step 2 (plug-in $\widehat{\lambda}$).}
Write $\widehat{\theta}(\lambda) = \bar Y + \lambda\bigl(\bar{\tilde Y}^{\mathrm{LLM}} - \bar{Y}^{\mathrm{LLM}}\bigr)$, where $\bar Y = \tilde n^{-1}\sum_i Y_i$ and $\bar{Y}^{\mathrm{LLM}} = \tilde n^{-1}\sum_i Y_i^{\mathrm{LLM}}$. Then
\[
\widehat{\theta}(\widehat{\lambda}) - \widehat{\theta}(\lambda^\star)
\;=\;
(\widehat{\lambda} - \lambda^\star)\bigl(\bar{\tilde Y}^{\mathrm{LLM}} - \bar{Y}^{\mathrm{LLM}}\bigr).
\]
Under finite fourth-moment conditions, $\widehat{\lambda} - \lambda^\star = O_p(\tilde{n}^{-1/2})$ (standard plug-in consistency; see \citealt{angelopoulos2023ppi}, Theorem~1), and $\bar{\tilde Y}^{\mathrm{LLM}} - \bar{Y}^{\mathrm{LLM}} = O_p(\tilde{n}^{-1/2} + m^{-1/2})$. Hence, the plug-in correction is $O_p\bigl(\tilde n^{-1/2} (\tilde n^{-1/2} + m^{-1/2})\bigr) = o_p(\sigma)$. By Slutsky's theorem, $\sigma^{-1}\bigl(\widehat{\theta}(\widehat{\lambda}) - \theta^\star\bigr) \xrightarrow{d} \mathcal{N}(0,1)$.

\emph{Step 3 (boundary and degenerate cases).}
Step 2 implicitly treats $\lambda^\star$ as interior to $[0,1]$ and $\Var(Y^{\mathrm{LLM}}) > 0$. We verify the three remaining cases.

(i) \emph{Boundary $\lambda^\star = 0$ with $\Var(Y^{\mathrm{LLM}}) > 0$} (so $\Cov(Y, Y^{\mathrm{LLM}}) \le 0$ in population). The clipped plug-in $\widehat{\lambda}$ converges to $0$ in probability, and $\widehat{\theta}(\widehat{\lambda}) - \widehat{\theta}(0) = \widehat{\lambda}(\bar{\tilde Y}^{\mathrm{LLM}} - \bar Y^{\mathrm{LLM}}) = o_p(\tilde n^{-1/2})$, so the estimator coincides with the human-only sample mean to first order and the classical CLT for $\bar Y$ gives the stated limit with $\sigma^2 = \Var(Y)/\tilde n$.

(ii) \emph{Boundary $\lambda^\star = 1$} (so $\Cov(Y, Y^{\mathrm{LLM}}) \ge \Var(Y^{\mathrm{LLM}})$ in population). The same argument applies with $\widehat{\lambda} \to 1$ and the estimator converges to standard PPI, whose CLT follows from Theorem~1 of \citet{angelopoulos2023prediction} with $\sigma^2 = \Var(Y - Y^{\mathrm{LLM}})/\tilde n + \Var(Y^{\mathrm{LLM}})/m$.

(iii) \emph{Degenerate surrogate $\Var(Y^{\mathrm{LLM}}) = 0$.} Then $Y^{\mathrm{LLM}}$ is almost surely equal to the constant $\mu_{\mathrm{LLM}} := \bE[Y^{\mathrm{LLM}}]$, so $\bar Y^{\mathrm{LLM}} = \bar{\tilde Y}^{\mathrm{LLM}} = \mu_{\mathrm{LLM}}$ almost surely for every $\tilde n, m$. The clipping convention gives $\lambda^\star = 0$ and $\widehat{\lambda} = 0$, so
\[
\widehat{\theta}(\widehat{\lambda}) \;=\; \bar Y + \widehat{\lambda}(\bar{\tilde Y}^{\mathrm{LLM}} - \bar Y^{\mathrm{LLM}}) \;=\; \bar Y \quad \text{a.s.}
\]
The classical CLT for the i.i.d.\ mean yields $\sqrt{\tilde n}(\bar Y - \theta^\star) \Rightarrow \mathcal{N}(0, \Var(Y))$, which matches the general formula evaluated at $\lambda^\star = 0$, $A = \Var(Y)$, and $\Var(Y^{\mathrm{LLM}}) = 0$ (so the second term in $\sigma^2$ vanishes).
\hfill\Halmos

\subsection{Robustness to Misspecified Variance Coefficients}
\label{subsec:robustness_A}

The allocation rule in Theorem~\ref{thm:optimal_allocation} assumes that the rectification difficulty coefficients $\{A_q\}_{q\in\mathcal{T}}$ are known. In practice, when designing a survey that contains new questions without labeled data, the researcher must use an estimate or prediction $\tilde{A}_q$ for each question. Because these predictions are made for questions that may differ from the historical training corpus, it is important to understand how errors in $\tilde{A}_q$ translate into efficiency losses in the resulting sample allocation.

Recall the design objective (a weighted sum of approximate mean squared errors in the SDR regime):
\[
J(\bm{A},\bm{n}) \;:=\; \sum_{q\in\mathcal{T}} w_q \frac{A_q}{n_q},
\]
where $\bm{A}=(A_1,\dots,A_Q)$ are the true rectification difficulties and $\bm{n}=(n_1,\dots,n_Q)$ is a feasible allocation of human labels. Let the feasible set be
\[
\mathcal{N} = \Bigl\{\bm{n}:\; n_q \geq 0\ \text{for all } q,\ \sum_{q\in\mathcal{T}} c_q n_q \le B\Bigr\}.
\]

Define the optimum and its value as
\[
\bm{n}^\star(\bm{A}) \in \arg\min_{\bm{n}\in\mathcal{N}} J(\bm{A},\bm{n}), \qquad
J^\star(\bm{A}) := J(\bm{A},\bm{n}^\star(\bm{A})).
\]
Theorem~\ref{thm:optimal_allocation} implies the closed-form optimum value
\begin{equation}
\label{eq:J_star_closed_form}
J^\star(\bm{A}) = \frac{1}{B}\left(\sum_{q\in\mathcal{T}} \sqrt{w_q A_q c_q}\right)^2.
\end{equation}

Now let $\widetilde{\bm{A}}=(\tilde{A}_1,\dots,\tilde{A}_Q)$ denote the estimated rectification difficulties used at the design stage (e.g., from meta-learning). Let
\[
\widetilde{\bm{n}}^\star \;:=\; \bm{n}^\star(\widetilde{\bm{A}})
\]
denote the allocation that is optimal for the plug-in model $\widetilde{\bm{A}}$, and the achieved design performance under the true coefficients is
\[
J(\bm{A},\widetilde{\bm{n}}^\star) = \sum_{q\in\mathcal{T}} w_q \frac{A_q}{\tilde{n}_q^\star}.
\]

\begin{proof}[Proof of Proposition~\ref{prop:robust_allocation}]
For brevity, write $J(\bm{A},\bm{n}) = \sum_q w_q A_q/n_q$ and let $\bm{n}^\star(\bm{A})$ denote the continuous minimizer given in Theorem~\ref{thm:optimal_allocation}. The first claim, $J(\bm{A},\widetilde{\bm{n}}^\star) \ge J^\star(\bm{A})$, follows directly from the definition of $J^\star(\bm{A})$ as the minimum of $J(\bm{A},\bm{n})$ over $\bm{n}\in\mathcal{N}$; $\widetilde{\bm{n}}^\star$ is a feasible allocation but not necessarily optimal under $\bm{A}$.

To establish the upper bound, define
\[
s_q = \sqrt{w_q A_q c_q},
\qquad
S = \sum_{q\in\mathcal{T}} s_q,
\]
and note from Equation~\eqref{eq:J_star_closed_form} that
\[
J^\star(\bm{A}) = \frac{S^2}{B}.
\]
Similarly, for the surrogate coefficients $\widetilde{\bm{A}}$ define
\[
\tilde{s}_q = \sqrt{w_q \tilde{A}_q c_q},
\qquad
\tilde{S} = \sum_{q\in\mathcal{T}} \tilde{s}_q.
\]
By Theorem~\ref{thm:optimal_allocation}, the optimal continuous allocation for $\widetilde{\bm{A}}$ is
\[
\tilde{n}_q^\star
=
\frac{B}{\tilde{S}}
\sqrt{\frac{w_q \tilde{A}_q}{c_q}},
\qquad q\in\mathcal{T}.
\]
Evaluating $J(\bm{A},\widetilde{\bm{n}}^\star)$ under the {true} coefficients $\bm{A}$ gives
\begin{align*}
J(\bm{A},\widetilde{\bm{n}}^\star)
&=
\sum_{q\in\mathcal{T}} w_q \frac{A_q}{\tilde{n}_q^\star}
=
\frac{\tilde{S}}{B} \sum_{q\in\mathcal{T}}
w_q A_q \sqrt{\frac{c_q}{w_q \tilde{A}_q}} \\
&=
\frac{\tilde{S}}{B} \sum_{q\in\mathcal{T}}
A_q \sqrt{\frac{w_q c_q}{\tilde{A}_q}}
=
\frac{\tilde{S}}{B} \sum_{q\in\mathcal{T}}
\sqrt{w_q A_q c_q}\,\sqrt{\frac{A_q}{\tilde{A}_q}} \\
&=
\frac{\tilde{S}}{B} \sum_{q\in\mathcal{T}}
s_q \sqrt{\frac{A_q}{\tilde{A}_q}}.
\end{align*}
Introduce the ratios $r_q = \tilde{A}_q / A_q$ and note that $A_q/\tilde{A}_q = 1/r_q$. Under the log-error condition in Proposition~\ref{prop:robust_allocation}, we have
\[
e^{-\eps} \le r_q \le e^\eps,
\qquad\text{so}\qquad
e^{-\eps/2} \le \sqrt{r_q} \le e^{\eps/2}.
\]
Moreover,
\[
\tilde{S} = \sum_{q\in\mathcal{T}} \sqrt{w_q \tilde{A}_q c_q}
= \sum_{q\in\mathcal{T}} s_q \sqrt{r_q}.
\]
Thus the ratio of the achieved objective to the optimal objective can be written as
\begin{align*}
\frac{J(\bm{A},\widetilde{\bm{n}}^\star)}{J^\star(\bm{A})}
&=
\frac{\tilde{S}}{S^2}
\sum_{q\in\mathcal{T}} s_q \sqrt{\frac{A_q}{\tilde{A}_q}}
=
\frac{1}{S^2}
\left(\sum_{q\in\mathcal{T}} s_q \sqrt{r_q}\right)
\left(\sum_{q\in\mathcal{T}} s_q \frac{1}{\sqrt{r_q}}\right).
\end{align*}
Define the normalized weights $\pi_q = s_q / S$, which satisfy $\pi_q \ge 0$ and $\sum_q \pi_q = 1$. Then
\[
\frac{J(\bm{A},\widetilde{\bm{n}}^\star)}{J^\star(\bm{A})}
=
\left(\sum_{q\in\mathcal{T}} \pi_q \sqrt{r_q}\right)
\left(\sum_{q\in\mathcal{T}} \pi_q \frac{1}{\sqrt{r_q}}\right).
\]

We now bound this expression. Write $u_q=\sqrt{r_q}\in[m,M]$ with $m=e^{-\eps/2}$ and $M=e^{\eps/2}$; the log-error condition implies $u_q\in[m,M]$ for every $q$. Set $a=\sum_q \pi_q u_q$ and $b=\sum_q \pi_q u_q^{-1}$. For any $u_q\in[m,M]$, $(u_q-m)(M-u_q)\ge0$; expanding and dividing by $u_q>0$ gives
\[
u_q+\frac{mM}{u_q}\le m+M .
\]
Averaging with respect to the convex weights $\{\pi_q\}$ yields $a+mM\,b\le m+M$. Since $a,b,mM>0$, the elementary inequality $xy\le (x+y)^2/4$ gives
\[
\frac{J(\bm{A},\widetilde{\bm{n}}^\star)}{J^\star(\bm{A})}
= ab
= \frac{a\,(mM\,b)}{mM}
\le \frac{(a+mM\,b)^2}{4mM}
\le \frac{(m+M)^2}{4mM}.
\]
Since $mM=1$ and $m+M=2\cosh(\eps/2)$, the right-hand side equals $\cosh^2(\eps/2)=1+\tfrac{\eps^2}{4}+O(\eps^4)$, which establishes the upper bound in Proposition~\ref{prop:robust_allocation}(ii). Moreover, with $x=e^{\eps}\ge1$,
\[
e^{\eps}-\cosh^2(\eps/2)=x-\frac{x+2+x^{-1}}{4}=\frac{(x-1)(3x+1)}{4x}\ge0,
\]
so $\cosh^2(\eps/2)\le e^{\eps}$ and the earlier bound $e^{\eps}$ follows as a looser corollary. The bound is sharp as a universal bound: for $\eps>0$, equality holds iff, on the support of $\pi$, $u_q\in\{m,M\}$ with $\sum_{q:r_q=e^{\eps}}\pi_q=\sum_{q:r_q=e^{-\eps}}\pi_q=\tfrac12$; if the fixed weights $\{\pi_q\}$ do not admit such a half-mass split, the bound remains valid but is not exactly attained for those weights.

For the lower bound, Cauchy--Schwarz implies
\[
\left(\sum_q \pi_q \sqrt{r_q}\right)
\left(\sum_q \pi_q \frac{1}{\sqrt{r_q}}\right)
\ge
\left(\sum_q \pi_q\right)^2
= 1,
\]
with equality if and only if $\sqrt{r_q}$ is constant in $q$, i.e., $r_q$ is constant across questions. This yields the lower bound and the characterization of equality cases.
\end{proof}

\subsection{Average-Case Performance Guarantee}
\label{app:avg_case_robustness}

The worst-case bound $\cosh^2(\eps/2)$ in Proposition~\ref{prop:robust_allocation} depends only on the maximal log-error $\eps$. Here we derive a tighter average-case characterization governed by the \emph{variance} of the log-errors rather than their worst-case magnitude.

Recall from the proof of Proposition~\ref{prop:robust_allocation} that
\[
\frac{J(\bm{A},\widetilde{\bm{n}}^\star)}{J^\star(\bm{A})}
=
\left(\sum_{q} \pi_q \sqrt{r_q}\right)
\left(\sum_{q} \pi_q \frac{1}{\sqrt{r_q}}\right),
\]
where $\pi_q = s_q/S$ are normalized weights and $r_q = \tilde{A}_q/A_q$ are the multiplicative errors. Let $\eta_q = \log r_q$ denote the log-scale errors. Substituting $\sqrt{r_q} = e^{\eta_q/2}$ and $1/\sqrt{r_q} = e^{-\eta_q/2}$, the ratio becomes
\[
\frac{J(\bm{A},\widetilde{\bm{n}}^\star)}{J^\star(\bm{A})}
=
\left(\sum_q \pi_q e^{\eta_q/2}\right)
\left(\sum_q \pi_q e^{-\eta_q/2}\right).
\]

\begin{proposition}[Average-case efficiency loss]
\label{prop:avg_case}
Suppose the log-errors $\eta_q = \log(\tilde{A}_q/A_q)$ satisfy $\sum_q \pi_q \eta_q = 0$ (zero weighted mean on the log scale) and $\sum_q \pi_q \eta_q^2 = \sigma_\eta^2$ (weighted variance). Then:
\begin{enumerate}[(i)]
\item \textbf{Exact representation:}
\[
\frac{J(\bm{A},\widetilde{\bm{n}}^\star)}{J^\star(\bm{A})}
=
\left(\sum_q \pi_q e^{\eta_q/2}\right)
\left(\sum_q \pi_q e^{-\eta_q/2}\right).
\]
\item \textbf{Second-order approximation:} For small $\sigma_\eta$,
\[
\frac{J(\bm{A},\widetilde{\bm{n}}^\star)}{J^\star(\bm{A})}
\approx
1 + \frac{\sigma_\eta^2}{4}.
\]
\item \textbf{Gaussian large-$Q$ limit:} If $\eta_q$ are treated as i.i.d.\ draws from $\mathcal{N}(0, \sigma_\eta^2)$ with uniform weights $\pi_q = 1/Q$, and $Q$ is large, then
\[
\frac{J(\bm{A},\widetilde{\bm{n}}^\star)}{J^\star(\bm{A})}
\xrightarrow{p}
e^{\sigma_\eta^2/4} \approx 1 + \frac{\sigma_\eta^2}{4} + O(\sigma_\eta^4).
\]
\end{enumerate}
\end{proposition}
\begin{proof}
Part~(i) restates the exact formula. For part~(ii), expand each exponential to second order: $e^{\eta_q/2} \approx 1 + \eta_q/2 + \eta_q^2/8$. Under the centering condition $\sum_q \pi_q \eta_q = 0$,
\[
\sum_q \pi_q e^{\eta_q/2} \approx 1 + \frac{\sigma_\eta^2}{8}, \qquad
\sum_q \pi_q e^{-\eta_q/2} \approx 1 + \frac{\sigma_\eta^2}{8}.
\]
Multiplying:
\[
\left(1 + \frac{\sigma_\eta^2}{8}\right)^2 = 1 + \frac{\sigma_\eta^2}{4} + O(\sigma_\eta^4).
\]
For part~(iii), when $Q$ is large and $\eta_q \stackrel{\mathrm{i.i.d.}}{\sim} \mathcal{N}(0,\sigma_\eta^2)$ with uniform weights $\pi_q = 1/Q$, the law of large numbers gives $\frac{1}{Q}\sum_q e^{\eta_q/2} \to \bE[e^{\eta/2}] = e^{\sigma_\eta^2/8}$ (moment generating function of a Gaussian), and similarly $\frac{1}{Q}\sum_q e^{-\eta_q/2} \to e^{\sigma_\eta^2/8}$. The product converges to $e^{\sigma_\eta^2/4}$.
\end{proof}

The key implication is that the efficiency loss depends on the \emph{variance} of the log-errors, not on the worst-case error $\eps$. For our empirical application, the meta-learning residuals have $\sigma_\eta \approx 0.42$ (consistent with $R^2 \approx 0.56$ on the log scale), so $\sigma_\eta^2 \approx 0.176$ and the average-case efficiency ratio is approximately $1 + \sigma_\eta^2/4 \approx 1.044$, i.e., about $4.4\%$ excess MSE. This is well below the worst-case bound $\cosh^2(\eps/2) \approx \cosh^2(0.4) \approx 1.17$ (using the 95th-percentile absolute log-error $\eps\approx0.8$), and is consistent with the empirically observed performance of the meta-learned allocation.

\section{Additional Empirical Results}
\label{appsec:additional_empirical}

This appendix collects supplementary empirical results that complement the main analysis in Section~\ref{sec:empirics}.

\subsection{Supplementary diagnostics for the Twin-2K-500 data}
\label{appsec:data_summary}

Table~\ref{tab:Aq-by-task} in Section~\ref{subsec:ppi_diagnostics} reports task-level statistics. Here we provide two further descriptive views.

\paragraph{Per-task LLM accuracy.} Figure~\ref{fig:accuracy-by-task} shows the per-task LLM accuracy $\overline{\mathrm{acc}}_q$, which ranges from $85.9\%$ (Outcome Bias) to $50.6\%$ (Allais).

\begin{figure}[ht]
    \centering
    \includegraphics[width=0.85\textwidth]{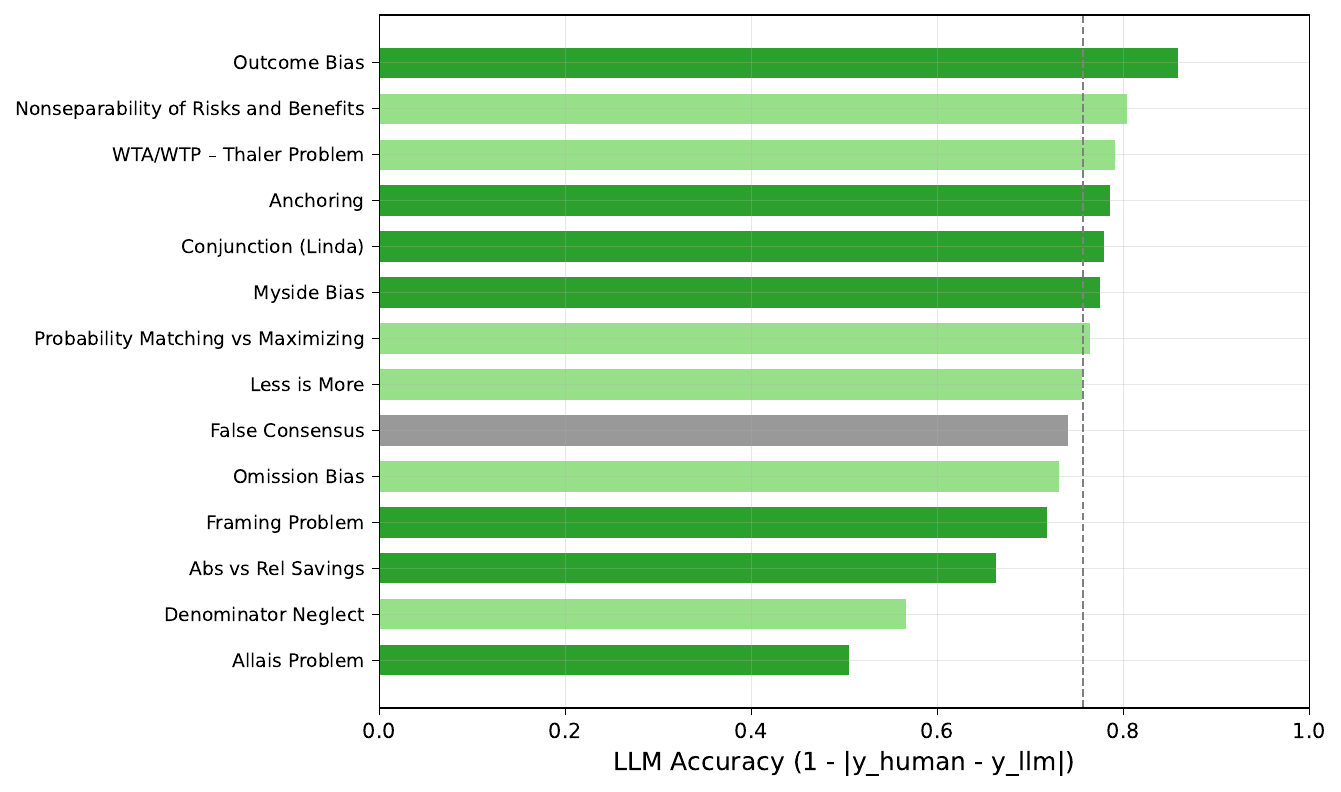}
    \caption{Task-level LLM accuracy $\overline{\mathrm{acc}}_q$, sorted from highest to lowest. The dashed vertical line marks the overall mean across the 68 questions.}
    \label{fig:accuracy-by-task}
\end{figure}

\paragraph{LLM accuracy vs.\ human test--retest reliability.}
Figure~\ref{fig:test-retest} compares human test--retest reliability with LLM accuracy across all 68 questions. Questions where humans are inconsistent tend also to be harder for the LLM (Spearman $\rho = 0.64$), suggesting a shared source of difficulty rooted in question ambiguity. The average human test--retest accuracy is $0.82$, compared with $0.76$ for LLM accuracy. As discussed in Section~\ref{subsec:ppi_diagnostics} (Probability Matching example), this aggregate gap can mask substantial individual-level differences.

\begin{figure}[ht]
    \centering
    \includegraphics[width=0.7\textwidth]{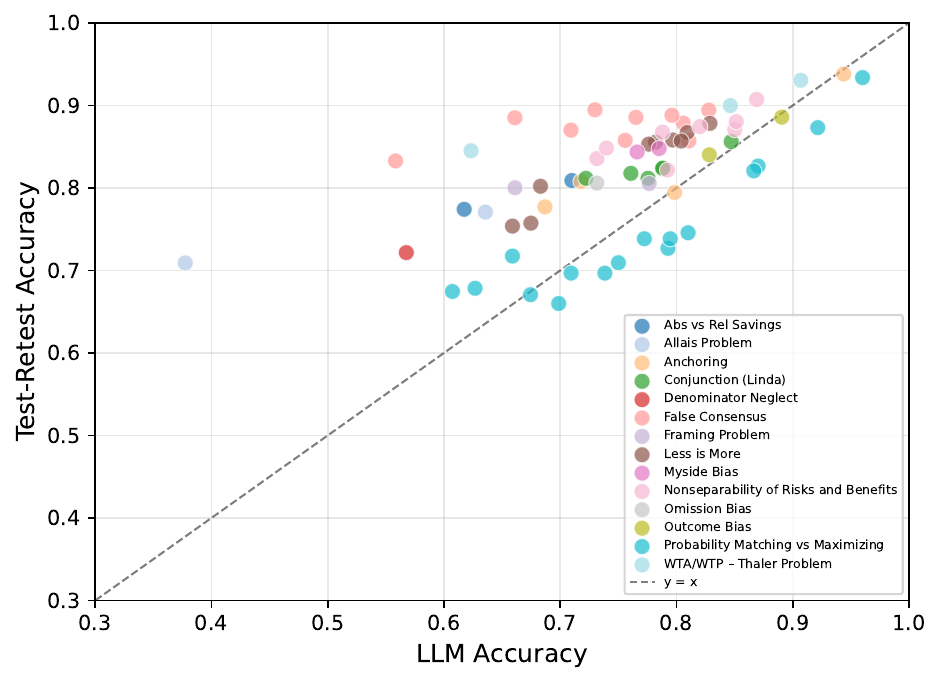}
    \caption{Human test--retest accuracy vs.\ LLM accuracy for 68 questions; dashed line is the $45^\circ$ reference.}
    \label{fig:test-retest}
\end{figure}

\subsection{Absolute MSE across budget levels}
\label{appsec:absolute_mse}

Table~\ref{tab:absolute-mse} reports absolute MSE, RMSE, and MAE at selected budget levels,
confirming the $1/B$ decay and stable percentage reductions across the full range.
At $B=200$ ($\approx 3$ responses per question), the MSE for \textsc{PPI + Opt.\ (Pred.)} is $3.33\times 10^{-2}$ vs.\ $3.68\times 10^{-2}$ for \textsc{SM + Uniform}, a meaningful improvement in a data-scarce regime.

\begin{table}[ht]
\centering
\footnotesize
\caption{Monte Carlo performance across selected budget levels.}
\label{tab:absolute-mse}
\begin{tabular}{rllll}
\toprule
$B$ & \textsc{SM + Uniform} & \textsc{PPI + Uniform} & \textsc{PPI + Opt. (Pred.)} & \textsc{PPI + Opt. (Oracle)} \\
\midrule
\multicolumn{5}{c}{Panel A: MSE ($\times 10^{-2}$)} \\
200 & 3.68 \scriptsize{[3.58, 3.75]} & 3.53 \scriptsize{[3.46, 3.60]} & 3.33 \scriptsize{[3.19, 3.52]} & 3.31 \scriptsize{[3.24, 3.37]} \\
400 & 1.82 \scriptsize{[1.79, 1.86]} & 1.77 \scriptsize{[1.74, 1.79]} & 1.74 \scriptsize{[1.71, 1.78]} & 1.63 \scriptsize{[1.60, 1.67]} \\
600 & 1.22 \scriptsize{[1.20, 1.25]} & 1.18 \scriptsize{[1.16, 1.20]} & 1.13 \scriptsize{[1.11, 1.16]} & 1.08 \scriptsize{[1.06, 1.10]} \\
800 & 0.92 \scriptsize{[0.90, 0.94]} & 0.87 \scriptsize{[0.85, 0.90]} & 0.84 \scriptsize{[0.83, 0.86]} & 0.82 \scriptsize{[0.81, 0.83]} \\
1000 & 0.73 \scriptsize{[0.71, 0.75]} & 0.71 \scriptsize{[0.69, 0.73]} & 0.68 \scriptsize{[0.66, 0.69]} & 0.65 \scriptsize{[0.64, 0.67]} \\
1500 & 0.50 \scriptsize{[0.49, 0.51]} & 0.48 \scriptsize{[0.47, 0.49]} & 0.46 \scriptsize{[0.45, 0.47]} & 0.44 \scriptsize{[0.43, 0.45]} \\
2000 & 0.38 \scriptsize{[0.37, 0.39]} & 0.37 \scriptsize{[0.35, 0.38]} & 0.34 \scriptsize{[0.33, 0.35]} & 0.33 \scriptsize{[0.32, 0.34]} \\
\midrule
\multicolumn{5}{c}{Panel B: RMSE ($\times 10^{-2}$)} \\
200 & 19.17 \scriptsize{[18.92, 19.36]} & 18.78 \scriptsize{[18.60, 18.98]} & 18.25 \scriptsize{[17.87, 18.76]} & 18.18 \scriptsize{[18.00, 18.37]} \\
400 & 13.50 \scriptsize{[13.37, 13.62]} & 13.30 \scriptsize{[13.19, 13.39]} & 13.19 \scriptsize{[13.06, 13.35]} & 12.78 \scriptsize{[12.66, 12.91]} \\
600 & 11.05 \scriptsize{[10.96, 11.18]} & 10.84 \scriptsize{[10.75, 10.94]} & 10.65 \scriptsize{[10.56, 10.79]} & 10.40 \scriptsize{[10.29, 10.50]} \\
800 & 9.58 \scriptsize{[9.50, 9.70]} & 9.35 \scriptsize{[9.21, 9.47]} & 9.18 \scriptsize{[9.10, 9.28]} & 9.05 \scriptsize{[8.99, 9.13]} \\
1000 & 8.53 \scriptsize{[8.45, 8.63]} & 8.41 \scriptsize{[8.28, 8.57]} & 8.22 \scriptsize{[8.11, 8.32]} & 8.08 \scriptsize{[7.99, 8.20]} \\
1500 & 7.06 \scriptsize{[6.99, 7.12]} & 6.93 \scriptsize{[6.85, 7.01]} & 6.75 \scriptsize{[6.69, 6.86]} & 6.62 \scriptsize{[6.54, 6.72]} \\
2000 & 6.17 \scriptsize{[6.09, 6.24]} & 6.04 \scriptsize{[5.96, 6.13]} & 5.84 \scriptsize{[5.78, 5.90]} & 5.73 \scriptsize{[5.68, 5.80]} \\
\midrule
\multicolumn{5}{c}{Panel C: MAE ($\times 10^{-2}$)} \\
200 & 14.83 \scriptsize{[14.66, 14.98]} & 14.44 \scriptsize{[14.31, 14.60]} & 14.33 \scriptsize{[14.03, 14.70]} & 14.29 \scriptsize{[14.14, 14.46]} \\
400 & 10.41 \scriptsize{[10.33, 10.54]} & 10.21 \scriptsize{[10.10, 10.31]} & 10.33 \scriptsize{[10.23, 10.47]} & 10.12 \scriptsize{[10.02, 10.24]} \\
600 & 8.51 \scriptsize{[8.44, 8.59]} & 8.32 \scriptsize{[8.25, 8.39]} & 8.32 \scriptsize{[8.23, 8.41]} & 8.23 \scriptsize{[8.15, 8.29]} \\
800 & 7.37 \scriptsize{[7.32, 7.45]} & 7.18 \scriptsize{[7.07, 7.27]} & 7.16 \scriptsize{[7.11, 7.24]} & 7.15 \scriptsize{[7.11, 7.21]} \\
1000 & 6.54 \scriptsize{[6.48, 6.59]} & 6.40 \scriptsize{[6.31, 6.53]} & 6.41 \scriptsize{[6.33, 6.49]} & 6.38 \scriptsize{[6.31, 6.48]} \\
1500 & 5.40 \scriptsize{[5.36, 5.44]} & 5.29 \scriptsize{[5.22, 5.34]} & 5.26 \scriptsize{[5.21, 5.33]} & 5.23 \scriptsize{[5.16, 5.29]} \\
2000 & 4.71 \scriptsize{[4.63, 4.77]} & 4.61 \scriptsize{[4.55, 4.66]} & 4.55 \scriptsize{[4.51, 4.59]} & 4.52 \scriptsize{[4.48, 4.59]} \\
\bottomrule
\end{tabular}
\end{table}

\subsection{Decomposition: allocation effect vs.\ estimator effect}
\label{appsec:sm_opt_decomposition}

We decompose the 14.5\% MSE reduction of \textsc{PPI + Opt.\ (Oracle)} into an \emph{allocation effect} and an \emph{estimator effect}.
The decomposition is path-dependent, so we report both paths.
To isolate the allocation channel, we introduce ``\textsc{SM + Opt.\ ($\Var(Y)$)}'': the ordinary sample mean (no \textsc{PPI++} correction) with Neyman allocation on $\Var(Y_q)$.
Table~\ref{tab:sm_opt_decomp} reports Monte Carlo MSE reductions (20 replications $\times$ 200 draws, averaged across 13 budget levels).

Table~\ref{tab:sm_opt_decomp} presents the $2 \times 2$ estimator--allocation grid.
Moving across a row (uniform $\to$ optimal allocation) adds 10.5~pp; moving down a column (SM $\to$ \textsc{PPI++}) adds 3.6~pp; the interaction is only 0.4~pp, so the two channels are approximately additive.
Allocation is the dominant channel, contributing roughly three-quarters of the total 14.5\% gain.
Both \textsc{SM + Opt.} and \textsc{PPI + Opt.\ (Oracle)} use oracle population quantities; in practice, both $\Var(Y_q)$ and $A_q$ must be estimated from data.
Notably, the fully feasible \textsc{PPI + Opt.\ (Pred.)} (11.4\%, Table~\ref{tab:compact-reduction}) also exceeds \textsc{SM + Opt.\ ($\Var(Y)$)}, even though the comparison is tilted against it: the classical design is handed the true $\Var(Y_q)$, which for a new survey would itself have to be estimated from pilot human data, whereas the meta-learned allocation uses no pilot data at all.

\begin{table}[ht]
\centering
\footnotesize
\caption{Decomposition of MSE reduction into allocation and estimator effects.}
\label{tab:sm_opt_decomp}
\begin{tabular}{lcc}
\toprule
& Uniform Allocation & Optimal Allocation \\
\midrule
Sample Mean (SM) & 0.0\% (baseline) & 10.5\% \scriptsize{[10.0, 11.1]} \\
\textsc{PPI++} & 3.6\% \scriptsize{[3.0, 4.1]} & 14.5\% \scriptsize{[14.0, 15.0]} \\
\midrule
\multicolumn{3}{l}{\footnotesize Estimator effect (SM $\to$ PPI++): $+3.6$ pp (column)} \\
\multicolumn{3}{l}{\footnotesize Allocation effect (Uniform $\to$ Optimal): $+10.5$ pp (row)} \\
\multicolumn{3}{l}{\footnotesize Interaction: $14.5 - 10.5 - 3.6 = +0.4$ pp} \\
\bottomrule
\end{tabular}
\end{table}

\subsection{Cross-dataset transferability of meta-learning}
\label{appsec:cross_dataset_transfer}

We test whether meta-learned rectification difficulty transfers across datasets by training on CCES (133 questions, Gemini~2.5~Flash) and predicting $A_q$ on Twin-2K-500 (68 questions, GPT-4o), so that both the survey domain and the response-generating LLM change between training and deployment.
Using PCA(5) and scale features on \texttt{text-embedding-3-large} embeddings, the cross-dataset prediction achieves Spearman $\rho = 0.525$ ($p < 0.001$) against the full-sample difficulties $\log \widehat{A}^{\,\mathrm{full}}_q$.
Table~\ref{tab:cross_dataset} reports MC MSE reductions: allocating Twin-2K-500 respondents with CCES-predicted $A_q$ and the \textsc{PPI++} estimator yields a 9.6\% MSE reduction, 68\% of the \textsc{PPI + Opt.\ (Oracle)} oracle (14.1\%).
This is weaker than within-dataset meta-learning (61--79\% of oracle), as expected given the different survey domains, but statistically significant, confirming that text embeddings capture question-difficulty structure that generalizes across survey contexts.

\begin{table}[ht]
\centering
\footnotesize
\caption{Cross-dataset transfer: CCES $\to$ Twin-2K-500.}
\label{tab:cross_dataset}
\begin{tabular}{lc}
\toprule
Method & MC MSE \% Reduction \\
\midrule
\textsc{SM + Uniform} & 0.00 (baseline) \\
\textsc{PPI + Uniform} & 3.6\% \\
Cross-dataset $A_q$ + \textsc{PPI++} & 9.6\% \\
\textsc{PPI + Opt.\ (Oracle)} & 14.1\% \\
\bottomrule
\end{tabular}
\end{table}

\subsection{Robustness to heterogeneity: detailed results}
\label{appsec:robustness_table}

Figure~\ref{fig:robustness} and Table~\ref{tab:robustness} report the MC MSE percentage reductions referenced in Section~\ref{subsec:robustness_eval}, broken down by heterogeneity parameter $\alpha$ and method at $B = 2000$.

\begin{figure}[ht]
    \centering
    \includegraphics[width=0.6\textwidth]{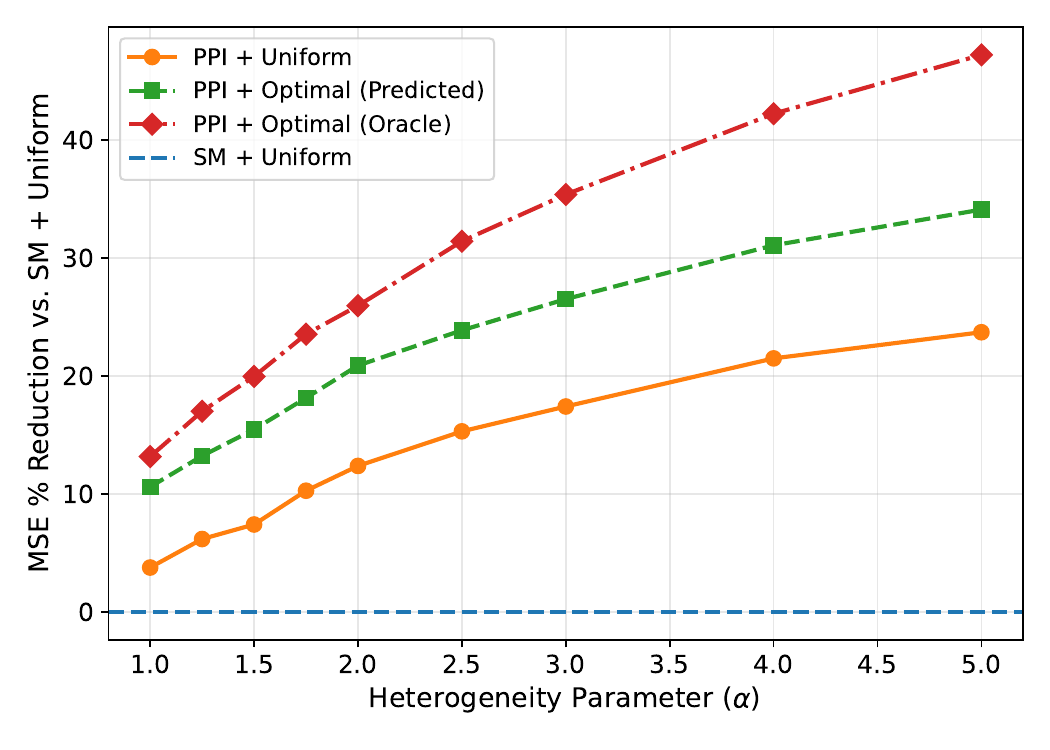}
    \caption{MSE \% reduction vs.\ difficulty dispersion $\alpha$ at $B = 2000$.}
    \label{fig:robustness}
\end{figure}

\begin{table}[ht]
\centering
\caption{MC MSE \% reduction under amplified $\widehat{A}^{\,\mathrm{full}}_q$ heterogeneity at $B = 2000$.}
\label{tab:robustness}
\footnotesize
\begin{tabular}{clll}
\toprule
$\alpha$ & \textsc{PPI + Uniform} & \textsc{PPI + Opt.\ (Pred.)} & \textsc{PPI + Opt.\ (Oracle)} \\
\midrule
\multicolumn{4}{c}{Panel A: MSE \% Reduction vs.\ \textsc{SM + Uniform}} \\
1.00 & \phantom{0}3.8\% \scriptsize{[\phantom{0}1.0, \phantom{0}7.1]} & 10.6\% \scriptsize{[\phantom{0}8.4, 13.2]} & 13.2\% \scriptsize{[\phantom{0}9.7, 16.5]} \\
1.25 & \phantom{0}6.2\% \scriptsize{[\phantom{0}3.0, \phantom{0}9.4]} & 13.2\% \scriptsize{[10.0, 15.6]} & 17.0\% \scriptsize{[13.5, 19.0]} \\
1.50 & \phantom{0}7.4\% \scriptsize{[\phantom{0}4.1, 10.3]} & 15.5\% \scriptsize{[12.3, 18.1]} & 19.9\% \scriptsize{[17.6, 22.4]} \\
1.75 & 10.3\% \scriptsize{[\phantom{0}7.8, 13.1]} & 18.1\% \scriptsize{[16.1, 21.5]} & 23.5\% \scriptsize{[21.7, 26.6]} \\
2.00 & 12.4\% \scriptsize{[\phantom{0}9.9, 15.1]} & 20.9\% \scriptsize{[17.9, 24.0]} & 25.9\% \scriptsize{[23.1, 28.6]} \\
2.50 & 15.3\% \scriptsize{[11.4, 17.7]} & 23.9\% \scriptsize{[22.1, 26.1]} & 31.4\% \scriptsize{[30.2, 32.4]} \\
3.00 & 17.4\% \scriptsize{[14.5, 19.6]} & 26.5\% \scriptsize{[22.6, 29.0]} & 35.3\% \scriptsize{[32.6, 37.4]} \\
4.00 & 21.5\% \scriptsize{[19.2, 24.2]} & 31.0\% \scriptsize{[28.6, 33.1]} & 42.2\% \scriptsize{[40.2, 44.0]} \\
5.00 & 23.7\% \scriptsize{[19.9, 26.9]} & 34.1\% \scriptsize{[32.7, 35.8]} & 47.2\% \scriptsize{[45.2, 49.4]} \\
\bottomrule
\end{tabular}

\medskip
\begin{minipage}{\textwidth}
\scriptsize\noindent\textit{Notes.} Meta-learning predictions are held fixed at their $\alpha = 1$ values; only the true $\widehat{A}^{\,\mathrm{full}}_q$ distribution is transformed.
\end{minipage}
\end{table}

\putbib[mybib]

\end{appendices}
\end{bibunit}


\begin{thebibliography}{44}
\providecommand{\natexlab}[1]{#1}
\providecommand{\url}[1]{\texttt{#1}}
\expandafter\ifx\csname urlstyle\endcsname\relax
  \providecommand{\doi}[1]{doi: #1}\else
  \providecommand{\doi}{doi: \begingroup \urlstyle{rm}\Url}\fi

\bibitem[{Accenture}(2025)]{accenture2025aaru}
{Accenture}.
\newblock Accenture invests in and collaborates with ai-powered agentic
  prediction engine aaru.
\newblock
  \href{https://newsroom.accenture.com/news/2025/accenture-invests-in-and-collaborates-with-ai-powered-agentic-prediction-engine-aaru}{Link},
  March 2025.
\newblock Accessed: 2026-06-18.

\bibitem[Adig{\"u}zel and Wedel(2008)]{adiguzel2008split}
Feray Adig{\"u}zel and Michel Wedel.
\newblock Split questionnaire design for massive surveys.
\newblock \emph{Journal of Marketing Research}, 45\penalty0 (5):\penalty0
  608--617, 2008.

\bibitem[Angelopoulos et~al.(2023{\natexlab{a}})Angelopoulos, Bates, Fannjiang,
  Jordan, and Zrnic]{angelopoulos2023prediction}
Anastasios~N Angelopoulos, Stephen Bates, Clara Fannjiang, Michael~I Jordan,
  and Tijana Zrnic.
\newblock Prediction-powered inference.
\newblock \emph{Science}, 382\penalty0 (6671):\penalty0 669--674,
  2023{\natexlab{a}}.

\bibitem[Angelopoulos et~al.(2023{\natexlab{b}})Angelopoulos, Duchi, and
  Zrnic]{angelopoulos2023ppi}
Anastasios~N Angelopoulos, John~C Duchi, and Tijana Zrnic.
\newblock {PPI}++: Efficient prediction-powered inference.
\newblock \emph{arXiv preprint arXiv:2311.01453}, 2023{\natexlab{b}}.

\bibitem[Angelopoulos et~al.(2025)Angelopoulos, Eisenstein, Berant, Agarwal,
  and Fisch]{angelopoulos2025costoptimal}
Anastasios~N. Angelopoulos, Jacob Eisenstein, Jonathan Berant, Alekh Agarwal,
  and Adam Fisch.
\newblock Cost-optimal active ai model evaluation.
\newblock \emph{arXiv preprint arXiv:2506.07949}, 2025.

\bibitem[Anthis et~al.(2025)Anthis, Liu, Richardson, Kozlowski, Koch,
  Brynjolfsson, Evans, and Bernstein]{anthis2025position}
Jacy~Reese Anthis, Ryan Liu, Sean~M Richardson, Austin~C Kozlowski, Bernard
  Koch, Erik Brynjolfsson, James Evans, and Michael~S Bernstein.
\newblock Position: Llm social simulations are a promising research method.
\newblock In \emph{Forty-second International Conference on Machine Learning
  Position Paper Track}, 2025.

\bibitem[Argyle et~al.(2023)Argyle, Busby, Fulda, Gubler, Rytting, and
  Wingate]{argyle2023out}
Lisa~P Argyle, Ethan~C Busby, Nancy Fulda, Joshua~R Gubler, Christopher
  Rytting, and David Wingate.
\newblock Out of one, many: Using language models to simulate human samples.
\newblock \emph{Political Analysis}, 31\penalty0 (3):\penalty0 337--351, 2023.

\bibitem[Bakker et~al.(2022)Bakker, Chadwick, Sheahan, Tessler,
  Campbell-Gillingham, Balaguer, McAleese, Glaese, Aslanides, Botvinick,
  et~al.]{bakker2022fine}
Michiel Bakker, Martin Chadwick, Hannah Sheahan, Michael Tessler, Lucy
  Campbell-Gillingham, Jan Balaguer, Nat McAleese, Amelia Glaese, John
  Aslanides, Matt Botvinick, et~al.
\newblock Fine-tuning language models to find agreement among humans with
  diverse preferences.
\newblock \emph{Advances in neural information processing systems},
  35:\penalty0 38176--38189, 2022.

\bibitem[Bisbee et~al.(2024)Bisbee, Clinton, Dorff, Kenkel, and
  Larson]{bisbee2024synthetic}
James Bisbee, Joshua~D Clinton, Cassy Dorff, Brenton Kenkel, and Jennifer~M
  Larson.
\newblock Synthetic replacements for human survey data? the perils of large
  language models.
\newblock \emph{Political Analysis}, 32\penalty0 (4):\penalty0 401--416, 2024.

\bibitem[Brand et~al.(2023)Brand, Israeli, and Ngwe]{brand2023using}
James Brand, Ayelet Israeli, and Donald Ngwe.
\newblock Using {LLMs} for market research.
\newblock \emph{Harvard business school marketing unit working paper},
  \penalty0 (23-062), 2023.

\bibitem[Broska et~al.(2025)Broska, Howes, and van Loon]{broska2025mixed}
David Broska, Michael Howes, and Austin van Loon.
\newblock The mixed subjects design: Treating large language models as
  potentially informative observations.
\newblock \emph{Sociological Methods \& Research}, 54\penalty0 (3):\penalty0
  1074--1109, 2025.
\newblock \doi{10.1177/00491241251326865}.

\bibitem[Brucks and Toubia(2025)]{brucks2025prompt}
Melanie Brucks and Olivier Toubia.
\newblock Prompt architecture induces methodological artifacts in large
  language models.
\newblock \emph{PloS one}, 20\penalty0 (4):\penalty0 e0319159, 2025.

\bibitem[Cochran(1977)]{cochran1977sampling}
William~G. Cochran.
\newblock \emph{Sampling Techniques}.
\newblock John Wiley \& Sons, New York, 3rd edition, 1977.

\bibitem[Deville and S{\"a}rndal(1992)]{deville1992calibration}
Jean-Claude Deville and Carl-Erik S{\"a}rndal.
\newblock Calibration estimators in survey sampling.
\newblock \emph{Journal of the American Statistical Association}, 87\penalty0
  (418):\penalty0 376--382, 1992.

\bibitem[Dzyabura and Hauser(2011)]{dzyabura2011active}
Daria Dzyabura and John~R Hauser.
\newblock Active machine learning for consideration heuristics.
\newblock \emph{Marketing Science}, 30\penalty0 (5):\penalty0 801--819, 2011.

\bibitem[Fisch et~al.(2024)Fisch, Maynez, Hofer, Dhingra, Globerson, and
  Cohen]{fisch2024stratified}
Adam Fisch, Joshua Maynez, R.~Alex Hofer, Bhuwan Dhingra, Amir Globerson, and
  William~W. Cohen.
\newblock Stratified prediction-powered inference for hybrid language model
  evaluation.
\newblock \emph{arXiv preprint arXiv:2406.04291}, 2024.

\bibitem[Huang et~al.(2025)Huang, Wu, and Wang]{huang2025many}
Chengpiao Huang, Yuhang Wu, and Kaizheng Wang.
\newblock How many human survey respondents is a large language model worth? an
  uncertainty quantification perspective.
\newblock \emph{arXiv preprint arXiv:2502.17773}, 2025.

\bibitem[Huber and Zwerina(1996)]{huber1996importance}
Joel Huber and Klaus Zwerina.
\newblock The importance of utility balance in efficient choice designs.
\newblock \emph{Journal of Marketing Research}, 33\penalty0 (3):\penalty0
  307--317, 1996.

\bibitem[{Index Ventures}(2026)]{indexventures2026simile}
{Index Ventures}.
\newblock Life, the universe, and simile: Leading simile's series a.
\newblock
  \href{https://www.indexventures.com/perspectives/life-the-universe-and-simile-leading-similes-series-a/}{Link},
  February 2026.
\newblock Accessed: 2026-06-18.

\bibitem[{Ipsos}(2025{\natexlab{a}})]{ipsos2025stanfordsynthetic}
{Ipsos}.
\newblock Ipsos partners with stanford university to pioneer the future of
  market research with synthetic data.
\newblock
  \href{https://www.ipsos.com/en-us/ipsos-partners-stanford-university-pioneer-future-market-research-synthetic-data}{Link},
  2025{\natexlab{a}}.
\newblock Accessed: 2026-06-18.

\bibitem[{Ipsos}(2025{\natexlab{b}})]{ipsos2025syntheticrespondents}
{Ipsos}.
\newblock Meet your new respondents: Research with synthetic data.
\newblock
  \href{https://www.ipsos.com/en-us/meet-your-new-respondents-research-synthetic-data}{Link},
  July 2025{\natexlab{b}}.
\newblock Accessed: 2026-06-18.

\bibitem[{Ipsos Digital}(2026)]{ipsosdigital2026innotest}
{Ipsos Digital}.
\newblock Innotest with synthetic respondents.
\newblock \href{https://www.ipsos.digital/inno-synthetic}{Link}, 2026.
\newblock Accessed: 2026-06-18.

\bibitem[Isaki and Fuller(1982)]{isaki1982survey}
Cary~T. Isaki and Wayne~A. Fuller.
\newblock Survey design under the regression superpopulation model.
\newblock \emph{Journal of the American Statistical Association}, 77\penalty0
  (377):\penalty0 89--96, 1982.

\bibitem[Ji et~al.(2025)Ji, Lei, and Zrnic]{ji2025predictions}
Wenlong Ji, Lihua Lei, and Tijana Zrnic.
\newblock Predictions as surrogates: Revisiting surrogate outcomes in the age
  of {AI}.
\newblock \emph{arXiv preprint arXiv:2501.09731}, 2025.

\bibitem[Maier et~al.(2025)Maier, Aslak, Fiaschi, Rismal, Fletcher, Luhmann,
  Dow, Pappas, and Wiecki]{maier2025purchase}
Benjamin~F. Maier, Ulf Aslak, Luca Fiaschi, Nina Rismal, Kemble Fletcher,
  Christian~C. Luhmann, Robbie Dow, Kli Pappas, and Thomas~V. Wiecki.
\newblock {LLMs} reproduce human purchase intent via semantic similarity
  elicitation of {Likert} ratings.
\newblock \emph{arXiv preprint arXiv:2510.08338}, 2025.

\bibitem[McCoy et~al.(2022)McCoy, Ciulli, and Bradlow]{mccoy2022twoforone}
John~P. McCoy, Rachele Ciulli, and Eric~T. Bradlow.
\newblock Two-for-one conjoint: {B}ayesian cross-category learning for
  shared-attribute categories.
\newblock \emph{SSRN Working Paper}, 2022.

\bibitem[Motoki et~al.(2024)Motoki, Pinho~Neto, and Rodrigues]{motoki2024more}
Fabio Motoki, Valdemar Pinho~Neto, and Victor Rodrigues.
\newblock More human than human: measuring {ChatGPT} political bias.
\newblock \emph{Public Choice}, 198\penalty0 (1):\penalty0 3--23, 2024.

\bibitem[Neyman(1934)]{neyman1934two}
Jerzy Neyman.
\newblock On the two different aspects of the representative method: The method
  of stratified sampling and the method of purposive selection.
\newblock \emph{Journal of the Royal Statistical Society}, 97\penalty0
  (4):\penalty0 558--625, 1934.

\bibitem[Ouyang et~al.(2022)Ouyang, Wu, Jiang, Almeida, Wainwright, Mishkin,
  Zhang, Agarwal, Slama, Ray, et~al.]{ouyang2022instruct}
Long Ouyang, Jeffrey Wu, Xu~Jiang, Diogo Almeida, Carroll Wainwright, Pamela
  Mishkin, Chong Zhang, Sandhini Agarwal, Katarina Slama, Alex Ray, et~al.
\newblock Training language models to follow instructions with human feedback.
\newblock \emph{Advances in neural information processing systems},
  35:\penalty0 27730--27744, 2022.

\bibitem[Peng et~al.(2025)Peng, Gui, Merlau, Fan, Sliman, Brucks, Johnson,
  Morwitz, Althenayyan, Bellezza, et~al.]{peng2025mega}
Tianyi Peng, George Gui, Daniel~J Merlau, Grace~Jiarui Fan, Malek~Ben Sliman,
  Melanie Brucks, Eric~J Johnson, Vicki Morwitz, Abdullah Althenayyan, Silvia
  Bellezza, et~al.
\newblock Digital twins as funhouse mirrors: Five key distortions.
\newblock \emph{arXiv preprint arXiv:2509.19088}, 2025.

\bibitem[Raghunathan and Grizzle(1995)]{raghunathan1995split}
Trivellore~E Raghunathan and James~E Grizzle.
\newblock A split questionnaire survey design.
\newblock \emph{Journal of the American Statistical Association}, 90\penalty0
  (429):\penalty0 54--63, 1995.

\bibitem[S{\'a}ndor and Wedel(2001)]{sandor2001designing}
Zsolt S{\'a}ndor and Michel Wedel.
\newblock Designing conjoint choice experiments using managers' prior beliefs.
\newblock \emph{Journal of Marketing Research}, 38\penalty0 (4):\penalty0
  430--444, 2001.

\bibitem[S{\"a}rndal et~al.(1992)S{\"a}rndal, Swensson, and
  Wretman]{sarndal1992model}
Carl-Erik S{\"a}rndal, Bengt Swensson, and Jan Wretman.
\newblock \emph{Model Assisted Survey Sampling}.
\newblock Springer-Verlag, New York, 1992.

\bibitem[Schaffner et~al.(2025)Schaffner, Shih, Ansolabehere, and
  Pope]{cces2024}
Brian Schaffner, Marissa Shih, Stephen Ansolabehere, and Jeremy Pope.
\newblock Cooperative election study common content, 2024.
\newblock \href{https://doi.org/10.7910/DVN/X11EP6}{[link]}, 2025.
\newblock {Harvard} Dataverse, V9.

\bibitem[Toubia and Hauser(2007)]{toubia2007research}
Olivier Toubia and John~R Hauser.
\newblock Research note—on managerially efficient experimental designs.
\newblock \emph{Marketing Science}, 26\penalty0 (6):\penalty0 851--858, 2007.

\bibitem[Toubia et~al.(2003)Toubia, Simester, Hauser, and
  Dahan]{toubia2003fast}
Olivier Toubia, Duncan~I Simester, John~R Hauser, and Ely Dahan.
\newblock Fast polyhedral adaptive conjoint estimation.
\newblock \emph{Marketing Science}, 22\penalty0 (3):\penalty0 273--303, 2003.

\bibitem[Toubia et~al.(2004)Toubia, Hauser, and Simester]{toubia2004polyhedral}
Olivier Toubia, John~R Hauser, and Duncan~I Simester.
\newblock Polyhedral methods for adaptive choice-based conjoint analysis.
\newblock \emph{Journal of Marketing Research}, 41\penalty0 (1):\penalty0
  116--131, 2004.

\bibitem[Toubia et~al.(2025)Toubia, Gui, Peng, Merlau, Li, and
  Chen]{toubia2025database}
Olivier Toubia, George~Z Gui, Tianyi Peng, Daniel~J Merlau, Ang Li, and Haozhe
  Chen.
\newblock Database report: Twin-2k-500: A data set for building digital twins
  of over 2,000 people based on their answers to over 500 questions.
\newblock \emph{Marketing Science}, 44\penalty0 (6):\penalty0 1446--1455, 2025.

\bibitem[Vafa et~al.(2025)Vafa, Athey, and Blei]{vafa2025estimating}
Keyon Vafa, Susan Athey, and David~M Blei.
\newblock Estimating wage disparities using foundation models.
\newblock \emph{Proceedings of the National Academy of Sciences}, 122\penalty0
  (22):\penalty0 e2427298122, 2025.

\bibitem[Wang et~al.(2025)Wang, Ye, and Zhao]{wang2025finetune}
Lei Wang, Zikun Ye, and Jinglong Zhao.
\newblock Efficient inference using large language models with limited human
  data: Fine-tuning then rectification.
\newblock \emph{arXiv preprint arXiv:2511.19486}, 2025.

\bibitem[Wang et~al.(2024)Wang, Zhang, and Zhang]{wang2024market}
Mengxin Wang, Dennis~J Zhang, and Heng Zhang.
\newblock Large language models for market research: A data-augmentation
  approach.
\newblock \emph{arXiv preprint arXiv:2412.19363}, 2024.

\bibitem[Ye et~al.(2025)Ye, Yoganarasimhan, and Zheng]{ye2025lola}
Zikun Ye, Hema Yoganarasimhan, and Yufeng Zheng.
\newblock {LOLA}: {LLM}-assisted online learning algorithm for content
  experiments.
\newblock \emph{Marketing Science}, 44\penalty0 (5):\penalty0 995--1016, 2025.

\bibitem[Ye et~al.(2026)Ye, Lyu, and Tao]{ye2026allocating}
Zikun Ye, Jiameng Lyu, and Rui Tao.
\newblock Allocating human oversight in ai-enabled analytics.
\newblock Working paper, available at SSRN:
  \href{https://papers.ssrn.com/sol3/papers.cfm?abstract_id=6561778}{[link]},
  2026.

\bibitem[Yin and Xin(2026)]{yin2026synthetic}
Qichuan~Ethan Yin and Linwei Xin.
\newblock Synthetic but not infinite: How much llm-generated data to use in
  market research.
\newblock \emph{Available at SSRN 6078686}, 2026.

\end{thebibliography}


\begin{thebibliography}{7}
\providecommand{\natexlab}[1]{#1}
\providecommand{\url}[1]{\texttt{#1}}
\expandafter\ifx\csname urlstyle\endcsname\relax
  \providecommand{\doi}[1]{doi: #1}\else
  \providecommand{\doi}{doi: \begingroup \urlstyle{rm}\Url}\fi

\bibitem[Angelopoulos et~al.(2023{\natexlab{a}})Angelopoulos, Bates, Fannjiang,
  Jordan, and Zrnic]{angelopoulos2023prediction}
Anastasios~N Angelopoulos, Stephen Bates, Clara Fannjiang, Michael~I Jordan,
  and Tijana Zrnic.
\newblock Prediction-powered inference.
\newblock \emph{Science}, 382\penalty0 (6671):\penalty0 669--674,
  2023{\natexlab{a}}.

\bibitem[Angelopoulos et~al.(2023{\natexlab{b}})Angelopoulos, Duchi, and
  Zrnic]{angelopoulos2023ppi}
Anastasios~N Angelopoulos, John~C Duchi, and Tijana Zrnic.
\newblock {PPI}++: Efficient prediction-powered inference.
\newblock \emph{arXiv preprint arXiv:2311.01453}, 2023{\natexlab{b}}.

\bibitem[Atkinson et~al.(2007)Atkinson, Donev, and Tobias]{atkinson2007optimum}
Anthony Atkinson, Alexander Donev, and Randall Tobias.
\newblock \emph{Optimum experimental designs, with SAS}, volume~34.
\newblock OUP Oxford, 2007.

\bibitem[Fedorov(2013)]{fedorov2013theory}
Valerii~Vadimovich Fedorov.
\newblock \emph{Theory of optimal experiments}.
\newblock Elsevier, 2013.

\bibitem[Newey and McFadden(1994)]{newey1994large}
Whitney~K Newey and Daniel McFadden.
\newblock Large sample estimation and hypothesis testing.
\newblock \emph{Handbook of econometrics}, 4:\penalty0 2111--2245, 1994.

\bibitem[Pukelsheim(2006)]{pukelsheim2006optimal}
Friedrich Pukelsheim.
\newblock \emph{Optimal design of experiments}.
\newblock SIAM, 2006.

\bibitem[Silvey(2013)]{silvey2013optimal}
Samuel Silvey.
\newblock \emph{Optimal design: an introduction to the theory for parameter
  estimation}, volume~1.
\newblock Springer Science \& Business Media, 2013.

\end{thebibliography}
\end{document}